\documentclass[11pt]{article}

\usepackage{jmlr2e}
\usepackage{multirow} 
\usepackage{FMTmicros}
\usepackage{tikz}
\usetikzlibrary{arrows,decorations.pathmorphing,backgrounds,positioning,fit,petri,graphs}
\usepackage[textfont=small,labelfont=bf]{caption}
\usepackage[textfont=small,labelfont=scriptsize]{subcaption}
\usepackage{booktabs}
\usepackage[linesnumbered,lined,boxed,commentsnumbered,ruled]{algorithm2e}
\usepackage{enumerate}
\usepackage{placeins}

\definecolor{citecol}{HTML}{6F130C}
\definecolor{tableofcontent}{HTML}{1F4A83}
\definecolor{urlcol}{HTML}{2470D8}

\usepackage{hyperref}
\hypersetup{
    colorlinks=true,       
    linkcolor=tableofcontent,          
    citecolor=citecol,        
    urlcolor=blue,           
}
\usepackage{graphics}
\usepackage{longtable}
\usepackage{threeparttable}

\newcommand{\C}{\mathbb{C}}
\newcommand{\R}{\mathbb{R}}
\newcommand{\Z}{\mathbb{Z}}
\newcommand{\N}{\mathbb{N}}
\newcommand{\dom}{{\mathcal{D}}}   

\newcommand{\gph}{\mathcal{G}} 
\newcommand{\VG}{V}  
\newcommand{\EG}{E}  
\newcommand{\wG}{\boldsymbol{w}}  
\newcommand{\mG}{\boldsymbol{K}}  
\newcommand{\adjG}{A}
\newcommand{\vG}{v}  
\newcommand{\uG}{p}  
\newcommand{\eG}{e}  
\newcommand{\dG}{\boldsymbol{d}}  
\newcommand{\distG}{\boldsymbol{\rho}}  
\newcommand{\vol}{\mathrm{vol}}  
\newcommand{\gL}{\mathcal{L}}  
\newcommand{\tri}{\mathcal{T}} 
\newcommand{\pV}[1][\uG]{[#1]}
\newcommand{\wV}[1][j,\pV]{\omega_{#1}}

\newcommand{\vf}{\boldsymbol{f}}
\newcommand{\vg}{\boldsymbol{g}}
\newcommand{\vc}{\boldsymbol{c}}
\newcommand{\ve}{\boldsymbol{1}}

\newcommand{\scala}{ 
    \alpha
}
\newcommand{\scalb}{ 
    \beta
}
\newcommand{\mask}{ 
    h
}
\newcommand{\maska}{ 
    a
}
\newcommand{\maskb}[1][n]{ 
    b^{(#1)}
}

\newcommand{\cfra}[1][j,\uG]{ 
    \boldsymbol{\varphi}_{#1}
    }
\newcommand{\cfrb}[2][j,\uG]{ 
    \boldsymbol{\psi}_{#1}^{#2}
}
\newcommand{\ufrsys}[1][J_0]{ 
    \mathsf{UFS}_{#1}
}
\newcommand{\fra}[1][j,\pV]{ 
    \fike{\varphi}{#1}
}
\newcommand{\frb}[2][j,\pV]{ 
    \fike{\psi}{#1}^{#2}
}
\newcommand{\dfrsys}{ 
    \mathsf{DFS}
}
\newcommand{\AS}{ 
    \mathsf{AS}
}

\newcommand{\ipG}[1]{                
\left\langle #1 \right\rangle
}
\newcommand{\nrmG}[1]{               
\left\|#1\right\|
}
\newcommand{\FT}[1]{ 
    \widehat{#1}
}
\newcommand{\fft}[1][j]{\mathbf{F}_{#1}}
\newcommand{\analOp}{\mathbf{W}}
\newcommand{\synOp}{\mathbf{V}}
\newcommand{\Id}{\mathbf{I}}

\newcommand{\supp}{\mathrm{supp}}

\newcommand{\conj}[1]{ 
 \overline{#1}
}
\newcommand{\setsep}{:}
\newcommand{\Dirac}{\boldsymbol{\delta}}

\newcommand{\veps}{\varepsilon}
\newcommand{\bigO}{\mathcal{O}}
\newcommand{\spoc}{\mathrm{spoc}}
\newcommand{\flop}{\mathrm{flop}}

\usepackage{lastpage}

\ShortHeadings{Decimated Framelet System on Graphs and Fast $\gph$-Framelet Transforms}{Zheng, Zhou, Wang, Zhuang}
\begin{document}

\title{\vspace{-1cm}{Decimated Framelet System on Graphs and}\\{Fast $\gph$-Framelet Transforms}}

\author{\name Xuebin Zheng
\email \texttt{\rm xuebin.zheng@sydney.edu.au}\\
       \addr University of Sydney Business School\\
       University of Sydney
       \AND
\name  Bingxin Zhou\thanks{First corresponding author}
\email \texttt{\rm bzho3923@uni.sydney.edu.au}\\
       \addr University of Sydney Business School\\
       University of Sydney
       \AND
\name Yu Guang Wang\thanks{Second corresponding author} \email \texttt{\rm yuguang.wang@sjtu.edu.cn}\\
       \addr 
       Institute of Natural Sciences and
       School of Mathematical Sciences\\
       Shanghai Jiao Tong University\\
       \&
       Max Planck Institute for Mathematics in Sciences\\
       \& School of Mathematics and Statistics\\
       University of New South Wales
       \AND
       \name Xiaosheng Zhuang \email \texttt{\rm xzhuang7@cityu.edu.hk} \\
       \addr Department of Mathematics\\
       City University of Hong Kong}

\editor{Joan Bruna}

\maketitle

\begin{abstract}
Graph representation learning has many real-world applications, from self-driving LiDAR, 3D computer vision to drug repurposing, protein classification, social networks analysis. An adequate representation of graph data is vital to the learning performance of a statistical or machine learning model for graph-structured data. This paper proposes a novel multiscale representation system for graph data, called decimated framelets, which form a localized tight frame on the graph. The decimated framelet system allows storage of the graph data representation on a coarse-grained chain and processes the graph data at multi scales where at each scale, the data is stored on a subgraph.
Based on this, we establish decimated $\gph$-framelet transforms for the decomposition and reconstruction of the graph data at multi resolutions via a constructive data-driven filter bank.
The graph framelets are built on a chain-based orthonormal basis that supports fast graph Fourier transforms. From this, we give a fast algorithm for the decimated $\gph$-framelet transforms, or {\fgt}, that has linear computational complexity $\bigo{}{N}$ for a graph of size $N$. The effectiveness for constructing the decimated framelet system and the {\fgt} is demonstrated by a simulated example of random graphs and real-world applications, including multiresolution analysis for traffic network and representation learning of graph neural networks for graph classification tasks.
\end{abstract}

\vspace{0.5cm}
\begin{keywords}
Graphs, Decimated tight framelets, Tree, SPOC, Undecimated tight framelets, Filter bank, Fast $\gph$-framelet transforms, Fast Fourier transforms, Coarse-grained chain, Graph Laplacian, Haar basis, Graph convolution, Graph neural networks, Multiresolution analysis
\end{keywords}

\tableofcontents

\section{Introduction}\label{sec:intro}
Geometric structure and feature of data are in the center of many commonly seen systems, which plays a pivotal role in understanding and guiding modern data science and machine learning avenues.
In physics, particles interact with each other following physics laws to form matters at different scales, such as atoms, molecules, planets, stars, solar systems, galaxies, and the whole universe \cite{shlomi2020graph,perraudin2019deepsphere}. In biology, genes as sequences of DNA or RNA encode molecules' functions to form cells, tissues, organs, organ systems, plants, animals, and the entire ecosystem. In human societies, individuals linked by the social rules interact to shape the new communities, societies, governments, nations, and countries. In computer science, data specified by the Internet protocols provide end-to-end data communications from the lowest link layers, through packeting, addressing, transmitting, routing, and receiving, to the highest application layers between PCs, smartphones, digital sensors. All these systems, and many others in the fields of neuroscience, cognitive science, sociology, have an underlying data structure that can be represented by graphs, where vertices are distinct elements or actors, and edges indicate connections or interactions between the elements or actors.
Moreover, a multiscale structure appears in all these systems, which is determined by the intensity of `force' of the interactions which form clusters or subsystems in different scales. A large amount of information might be generated and attached to the complex system, where we can view a system and its information as a graph and graph data. 
This paper constructs a tight decimated framelet system to represent graph data and develops an efficient algorithm for framelet transforms. 

On Euclidean domains, harmonic analysis \citep{St1970, StWe1971, St1993} has been an active research branch of mathematics since the seminal work of Fourier \citep{Fourier1822}. In the past two centuries, it has become a well developed subject with application in areas as diverse as signal processing, representation theory, number theory, quantum mechanics, tidal analysis, and neuroscience. In the last four decades, one of its sub-branches -- wavelet analysis has been intensively studied by many pioneers \citep{Meyer1990, Chui1992, Daubechies1992, Mallat2009, Han2017}. In recent years, there has been a great interest in developing wavelet-like representation systems for data on non-Euclidean domains, including manifolds and graphs. One motivation comes from the interdisciplinary research demand of computer science and mathematics, when the data are not only big but also have intrinsic geometric structure, from such as social networks, biology, chemistry, physics, finance to image processing. The big data are usually regarded as samples from a smooth manifold, where the graph Laplacian approximates the underlying manifold Laplacian \citep{Singer2006}. The underlying manifold encodes the geometric information of the data, which approach has been widely used in machine learning and statistical models \citep{roweis2000nonlinear, tenenbaum2000global, BeMaNi2004, zhang2004principal, BeNi2003, BeNi2008, bruna2013spectral,DeBrVa2016,cheung2018graph,cheung2020graph}. Moreover, people have combined deep learning and graph signal processing \citep{sandryhaila2013discrete, sandryhaila2014discrete, shuman2013emerging,zhou2021graph}, which fosters emerging fields of \emph{geometric deep learning} \citep{bronstein2017geometric,bronstein2021geometric,bodnar2021mpsn,bodnar2021cwn} and \emph{graph neural networks} \citep{wu2020comprehensive,ma2020unified,zheng2020mathnet,zheng2021framelets,zhou2021spectral}. 

Motivated by the importance of data processing, the sparse representation of graph data, and the increasing interest in harmonic analysis for graph signal processing, in this paper, we investigate the characterization, construction, and computation for tight framelets on graphs. We lay out the framework of this paper, as follows.
We focus on an undirected connected graph $\gph=(\VG,\EG)$ with the vertex set $\VG$ and edge set $\EG$. Our goal is to provide a systematic method for multiscale decomposition and reconstruction of a graph signal $\vf: \VG\rightarrow \C$ defined on the graph $\gph$.

\begin{figure}[thb!]
\begin{minipage}{\textwidth}
\centering
\begin{minipage}{\textwidth}
\centering
\includegraphics[width=0.55\textwidth]{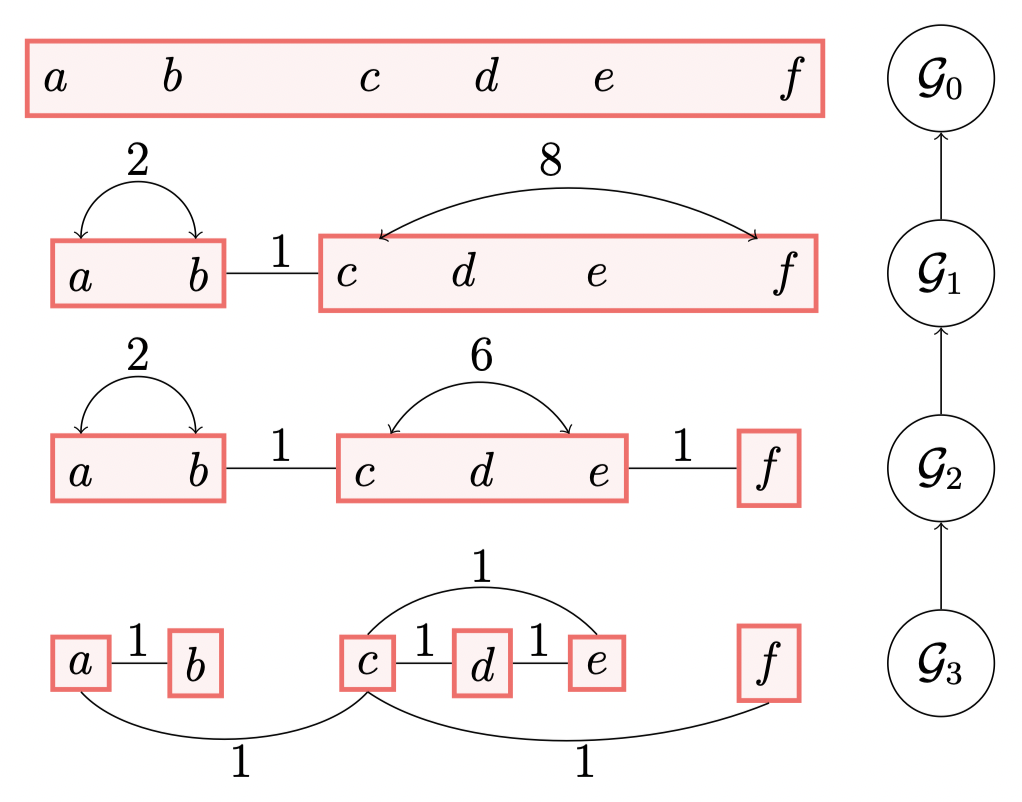}
\end{minipage}
\begin{minipage}{0.8\textwidth}
\caption{Coarse-grained chain $\gph_{3\rightarrow0}:=\gph_3\rightarrow \gph_2\rightarrow \gph_1\rightarrow \gph_0$ of $\gph\equiv \gph_3$. Here the bottom graph $\gph_3\equiv \gph$ is the underlying graph. Each box in $\gph_j=(\VG_j,\EG_j,\wG_j)$ is a vertex of $\gph_j$, the line between two vertices is a edge, the arc on a same vertex indicates a self-loop edge, and the number on a line/curve is the corresponding edge weight.}
\label{fig:coarse-grain-toy}
\end{minipage}
\end{minipage}
\end{figure}

\begin{itemize}
\item[(1)]
Graph clustering generates a coarse-grained chain $\gph_{J\rightarrow J_0}$ with $\gph_J \equiv \gph$ which can be built with a \emph{nested} property (finest to coarsest). That is, each vertex in $\gph_j=(\VG_j, \EG_j)$ is a cluster of the previous $\gph_{j-1}$ and hence a cluster of the original graph $\gph$. The \emph{bottom-up} construction provides a multiscale structure on $\gph$. It assigns a level of the chain to \emph{scale} values and defines \emph{translation} by cluster centroids. The resulting chain plays an important role in construction of our \emph{decimated tight framelets}. Figure~\ref{fig:coarse-grain-toy} illustrates clustering of a graph $\gph$ to a coarse-grained chain $\gph_{3\rightarrow 0}$.

\item[(2)] With the constructed coarse-grained chain, we can define a sequence of orthonormal bases $\{\eigfm^{\gph_j} \}_{\ell=1}^{|\VG_j|}$, $j=J_0,\ldots, J$ in a \emph{top-down} manner. Eventually, we would obtain a \emph{global} orthonormal basis $\{\eigfm: \VG\rightarrow \C\}_{\ell=1}^{|\VG|}$ for the underlying graph $\gph$, which can be used to construct tight framelets on $\gph$. The strategy for framelet construction is independent of the orthogonal basis: one can apply either spectral methods (such as eigendecomposition of the graph Laplacian) or non-spectral methods (such as interval decomposition).  The top-down approach utilizes a chain, which provides an orthogonal basis mimicking the classical Fourier bases, and defines a notion of \emph{frequencies} 
that can be used to distinguish \emph{low-pass} and \emph{high-pass} information for a graph signal. Any signal $\vf: V\rightarrow \C$ can be precisely represented in the graph framelet domain by the framelet transforms.

\item[(3)] Motivated by the localization property of the filtered kernel 
\begin{equation*}
    K_{j,\alpha}(p,v):=\sum_\ell \FT{\alpha}\left(\frac{\eigvm}{2^j}\right) \overline{\eigfm(p)}\eigfm(\vG),
     \mbox{~for~} p,v\in\VG.
\end{equation*} 
and the framelet filter bank on $\R$, we define \emph{tight framelets} $\frb{}=K_{j,\alpha}(p,\cdot)$ on a graph. Framelet $\frb{}=K_{j,\alpha}(p,\cdot)$ at scale $j$ is localized around the centroid $\pV$, and the set of framelets is a tight frame with multiscale property. Here, $\widehat{\alpha}$ is the filter which localizes and smooths the kernel, and $\lambda_{\ell}$ is the eigenvalue of the underlying graph Laplacian for degree $\ell$. The $2^j$ is the scaling factor for level $j$ in the multiscale system. In this work, we investigate two types of tight frames on the graph: \emph{undecimated} and \emph{decimated tight framelets}. Undecimated framelets take the form of $\cfrb{n} : V\rightarrow C$ with node $\uG\in \VG$ running through vertices on $\VG$ (at scale $j$), while decimated framelets are of the form $\frb{n} : \VG\rightarrow C$ with node $\pV\in \VG_{j+1}$ through all vertices on $\VG_{j+1}$. Both framelets are defined on the underlying graph $\gph$, while for the decimated version, the `centers' $\uG$ and $\pV$ locate at different layers of the chain. Tightness of the framelet system provides an exact representation of the graph data in the framelet domain. It also generates the perfect reconstruction for graph signals under framelet transforms.
By the characterization theorem, designing a tight framelet system is 
equivalent to finding a proper filter bank. For decimated framelets, the filter bank depends on the chain structure.

\item[(4)] The chain-based orthogonal basis gives rise to fast algorithmic transforms for the framelet system. With the chain-based orthogonal basis $\{\eigfm\}$, we have the \emph{fast graph Fourier transforms} to compute the graph Fourier coefficients $\FT{\vf_\ell}$ from $\vf$. With the above tight framelet filter banks for the decimated framelets on $\gph$, we can implement fast $\gph$-framelet transforms for signals defined on $\gph$. Such fast framelet transforms provide a way for sparse representation, efficient computation, and effective processing of signals on a graph. See Figure~\ref{fig:algo:fb} for an illustration for the one-level discrete $\gph$-framelet transforms.
\end{itemize}

\begin{figure}[th]
\begin{minipage}{\textwidth}
\centering
\includegraphics[width=0.8\textwidth]{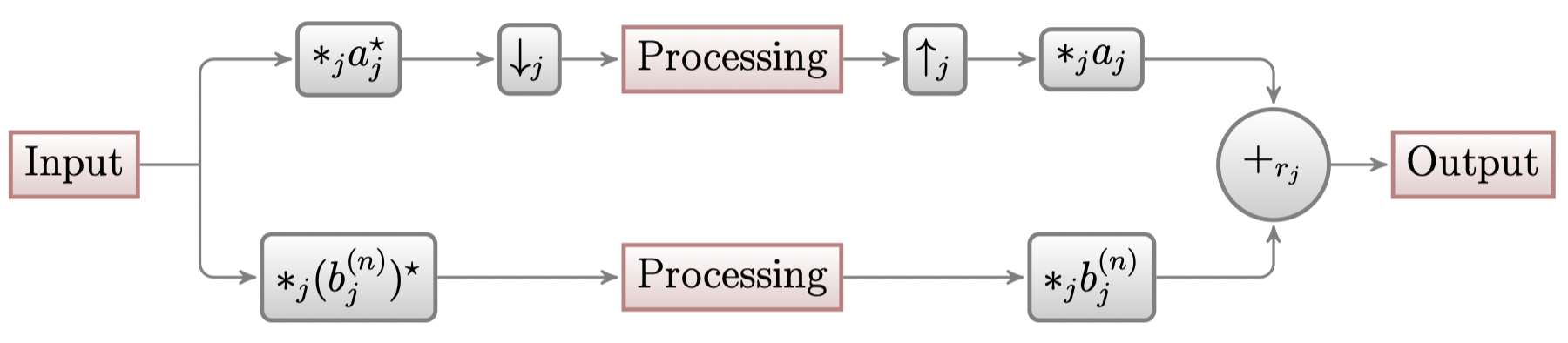}
\begin{minipage}{0.8\textwidth}
\caption{One-level $\gph$-framelet transforms including framelet decomposition and reconstruction based on the filter bank $\{\maska_j;\maskb[1]_j,\ldots,\maskb[r_j]_j \}$. Input graph signal is decomposed using convolutions ($*$) with filters and possible downsampling ($\downarrow$) to output framelet coefficient sequences, which are processed before reconstruction that uses the same filter bank. Here the node $\dconv\maskb_j$ of the convolution operation ranges over $n=1,\ldots,r_j$ and the $+_{r_j}$ is the summation over the low-pass filtered coefficient sequence and all $r_j$ high-pass filtered coefficient sequences.}
\label{fig:algo:fb}
\end{minipage}
\end{minipage}
\end{figure}


The contributions of the paper lie in the following aspects.
\begin{enumerate}
  \item[(i)] We introduce the decimated framelet system on a graph, as a special affine system. \cite{WaZh2018} constructed continuous framelets and semi-discrete framelets on a continuous (or smooth) manifold. \cite{Dong2017} constructed undecimated framelets on graphs, where the framelets $\psi_{j,\uG}$ are defined for all $\uG\in\VG$ and all $j$. These correspond to a continuous wavelet tight frame on a manifold. However, the undecimated framelet system may result in a high redundancy rate when the number of levels is big, as there is no \emph{coarsening} in the framelet system (see also \cite{HaVaGr2011}). One needs to use Chebyshev approximation for fast framelet transforms in this case. In contrast, by using the graph's clustering feature, we embed the edge relation and clustering features into a chain-based orthonormal system to establish a \emph{decimated framelet system}. The resulting tight framelet system that corresponds to a discrete framelet system on a manifold of \cite{WaZh2018} would achieve a low redundancy rate. 

  \item[(ii)] We present a framework for the characterization and construction of decimated and undecimated tight framelets on graphs. In particular, we offer a complete characterization for the tightness of decimated framelet systems on graphs. The characterization implies an equivalent condition of the framelet tightness that significantly simplify the design of tight graph framelets, due to the excellent property of the constructed chain-based orthonormal system. 

  \item[(iii)] We provide a fast algorithm in $\mathcal{O}(|\VG|)$ for the discrete Fourier transforms under the established chain-based orthonormal system. The fast algorithm overcomes the general difficulty of applying FFT for graph-structured signals on irregular nodes. Building on a chain-based orthogonal basis and a sequence of tight framelet filter banks, the discrete decimated framelets allow fast computation of decomposition and reconstruction for a graph signal in the multiscale framelet representation --- \emph{fast $\gph$-framelet transforms} or F$\gph$Ts.

  \item[(iv)] The chain-based orthonormal basis and tight framelets provide a spectral method in defining graph convolution and pooling, which are key ingredients of graph neural networks. See preliminary empirical studies in \cite{wang2019haar,LI2020188,zheng2020mathnet} and Section~\ref{sec:fgconv} provide a multiscale analysis tool with potential applications in areas such as superresolution analysis \citep{Deng_2019_ICCV,candes2014towards,Schiebinger2017superresolution} and geometric matrix/tensor completion \citep{jiang2020framelet,berg2017graph,monti2017geometric}. We provide framelet multiresolution analysis on the Minnesota traffic network. We would develop an {\fgt}-based spectral graph convolution for graph neural networks on various graph classification tasks. 
\end{enumerate}


We organize the rest of the paper as follows. In Section~\ref{sec:pre}, we introduce basic notions used in the paper, including graph, chain, orthonormal basis, frame, framelet and tightness.
In Section~\ref{sec:char}, we define both undecimated and decimated framelet systems on graphs and provide a complete characterization for the tightness of both types of framelet systems. In Section~\ref{sec:constr}, we construct a chain-based orthonormal basis, including examples of graph Laplacian-based basis and Haar-like basis. 
We then define decimated tight framelets on the graph using the above chain-based orthogonal basis. 
In Section~\ref{sec:decomp.reconstr}, we develop a fast computational strategy for the transforms in the orthogonal basis representation and framelet representation.
We develop algorithms for discrete Fourier transforms under the chain-based orthonormal basis, and give a sufficient condition that guarantees linear computational complexity of the algorithm.
Decimated $\gph$-framelet transforms can then be implemented based on the underlying filter bank, and the computational complexity is in the same order as the corresponding DFT.
In Section~\ref{sec:toy}, we show a simple example to illustrate the full construction of a chain-based orthonormal basis and decimate tight framelets. In Section~\ref{sec:numer}, we show examples of filter banks for decimated framelets, and provide a test of computational complexity which shows the consistent result to the algorithmic analysis. We finally apply decimated framelets for MRA of Minnesota traffic network, and graph convolution of GNNs on graph classification tasks. 
Discussion and remarks are given in Section~\ref{sec:discussion}.

\section{Related Work}\label{subsec:relatedwork}
\paragraph{Clustering} Our construction of decimated tight framelets on the graph utilizes graph clustering. There are many existing clustering algorithms \citep{gphClus1,gphClus2, gphClus3, gphClus4, gphClus5, gphClus6}. While we do not focus on clustering methods, which themselves are an active research area, we mention a few examples. The coarse grained chain $\gph_{J\rightarrow J_0}=\gph_J\rightarrow \gph_{J-1}\rightarrow \cdots \rightarrow \gph_{J_0}$ has a tree structure in nature (see Figure~\ref{fig:coarse-grain-toy} for an example. One can simply take $\VG$ at the root node as a single cluster if $\gph_{J_0}$ is not with a single cluster). To obtain such a chain, one can use a partition approach, which recursively partitions a graph into subgraphs, or, reversely, a merging approach, which groups various vertices to generate clusters. The partition approach starts from the root of a tree (the full graph) and continues to its branches (clusters) up to a certain level, for example, the graph cut method \citep{gphCut1, shi2000normalized}. In contrast, the merging approach starts from the leaves (vertices of a graph) of a tree and climbs up to a certain level (clusters), for example, the $k$-means methods \citep{kmean1,kmean2,karypis1998fast}. One can also categorize graph clustering algorithms based on \emph{spatial-frequency} criteria, that is, by whether the clustering method is spectral-based. The spectral-based graph clustering algorithms usually use graph Laplacian, eigendecomposition or wavelet decomposition \citep{MeBeOs1992,BeFl2012,Garcia-Cardona_etal2014}, while the non-spectral clustering employs information from an appropriately defined weight function $\wG$ \citep{Coifman_etal2005,Dongen2000,GaCo2012,GaNaCo2010,ChDa2010,LaLe2006,BoVeZa2001,chen2014unsupervised}.

\paragraph{Wavelets and framelets} Once a chain of nested graphs is obtained, we can build the multiscale structure on a graph $\gph$ by framelet filter banks, which bridges with the classical wavelet/framelet systems on $\R$. We use non-homogeneous affine systems to describe our framelet systems on $\gph$, where a non-homogeneous affine system on a domain $\dom$ typically takes the form
\[
\{\varphi_{J_0,k} \setsep k\in \Lambda_{J_0}\}\cup \{\psi_{j,k}^{(n)}\setsep k\in \Lambda_{j}, n=1,\ldots, r\}_{j=J_0}^J
\]
for $J_0,J\in\Z$, $r\in\N$. Here $\Lambda_j$ are the index sets of the \emph{translation points} and \emph{centers} for the framelets on the domain $\dom$.
The function $\varphi_{j,k}$ is from a scaling function $\varphi$ which captures the \emph{low frequency} information of the graph signal and $\psi_{j,k}^{(n)}$ is associated with framelet functions $\psi^{(n)}$ which captures \emph{high frequency} information at a specific scale. The functions of $\Psi:=\{\varphi; \psi^{(1)},\ldots,\psi^{(r)}\}$ are associated with a filter bank $\filtbk:=\{\maska; \maskb[1],\ldots,\maskb[r]\}$ by the \emph{Unitary Extension Principle} (UEP).
On $\R^d$, affine systems and the construction of wavelet/framelet systems have been well studied \citep{RoSh1997,RoSh1997b,chui1998affine,ChHeSt2002,DaHaRoSh2003}. Non-homogeneous affine systems on $\R^d$ for a more general setting are studied in \cite{Han2010, Han2012}. The construction of directional affine systems was further exploited in \cite{guo2004wavelets,  HanZhuang2015, Zhuang2016, CheZhuang2018, ATREAS20191}. In particular, \cite{WaZh2018} constructed a tight framelet system on a compact, smooth manifold $\mfd$. A framelet system on a more abstract locally compact abelian group is proposed by \cite{ChGo2019}.

A variety of wavelets/framelets on manifolds have been proposed. \cite{CoMa2006} introduced orthogonal diffusion wavelets on a smooth manifold by diffusion operators.  Mhaskar et al. \citep{fasttour,MaMh2008, Mhaskar2010} extended the construction from diffusion wavelets to diffusion polynomial frames on a manifold. \cite{Dong2017} constructed continuous framelets on a compact Riemannian manifold, where they made the first connection of filter bank with continuous framelets on manifolds. \cite{wang2018analysis} developed semi-discrete and fully discrete tight framelets on a simplex of any dimension by using a quadrature rule on the simplex for integral discretization. \cite{wang2017fully,le2017needlet} developed discrete spherical needlets (a highly localized tight frame on the unit sphere) by filtered spherical harmonic expansion and quadrature rule, and applied to the decomposition of spherical random fields.
\cite{li2019tensor} constructed tight tensor framelets for vector fields on the sphere and developed a fast algorithm for the tensor framelet transforms.
The impressive advance of data science and deep learning, such as graph neural networks, has driven the rapid development of the field of wavelets and framelets on graphs \citep{shuman2015spectrum}. 

Wavelets on graphs were first proposed in \cite{CrKo2003}. The diffusion polynomial approach was then introduced to construct diffusion wavelets on graphs \citep{CoMa2006, MaMh2008}.
With the spectral graph theory, \cite{HaVaGr2011} established the wavelet analysis, and \cite{Dong2017} constructed undecimated tight wavelet frames on graphs. Haar transform-based approach on graphs were studied in \cite{GaNaCo2010, GaCo2012,LI2020188}. \cite{ChFiMh2015} constructed an orthogonal polynomial system on a weighted tree, which gave a chain-based Haar orthonormal basis. On another direction, \cite{bruna2013invariant,mallat2012group,mallat2016understanding} developed wavelet scattering that is a wavelet system combined with neural activation that mimics the architecture of a deep convolutional neural network. \cite{gao2019geometric, gama2019stability, zou2020graph} developed the wavelet scattering methods for graphs. We propose to use the graph framelet transforms to provide a multiscale representation for graph data and feature extraction and thus to define a new graph convolution as a core computational block for the graph neural network rather than using wavelet scattering to replace the whole GNN architecture.

\paragraph{Fast transform and graph convolution} 
The fast transform algorithms are a fundamental tool for efficient computation with wavelets or framelets. One applies the forward transform to decompose a signal representation from the spatial domain to a set of wavelet basis. The adjoint transform, in contrast, reconstructs the signal from its wavelet representation back to the spatial domain. A fast implementation benefits, in particular, for large-scale graphs. 
The fast algorithm of wavelet transforms was first proposed by \cite{mallat1989theory} with multi-resolution analysis, which later inspires a number of applications in signal processing. The wavelet coefficients through the fast algorithm are evaluated level by level using discrete convolution and decimated operators rather than estimating inner products for wavelet coefficients. The domain of fast wavelet transforms has been well developed for discrete signals, see \cite{candes2006fast,han2013properties,han2016directional}, to mention but a few. 

Graph convolution was first introduced in \cite{bruna2013spectral} based on the convolution theorem and spectral graph theory \citep{Chung1997}. Since then, a good number of graph convolution and graph neural network models have been proposed \citep{DeBrVa2016,kipf2016semi,veli2018graph,Fawzi_etal2018,cheung2018graph,LI2020188,ma2019graph,wang2019haar,ma2020path,zhou2021graph}. The graph convolution is conducted in the spectral domain as diagonal multiplications. However, due to the non-Euclidean nature, discrete Fourier transforms,  suffer from a high cost for computing the orthogonal basis from the graph Laplacian. To overcome the difficulty,  \cite{DeBrVa2016} approximates smooth filters with Chebyshev polynomials, which then circumvents eigendecomposition for the graph Laplacian. \cite{kipf2016semi,gilmer2017neural,veli2018graph,xu2018how} further simplified the convolutional layer with first-order approximation and proposed spatial-based graph convolution. Such methods rely on node feature aggregation among neighbour vertices. Meanwhile, fast Fourier or wavelet transform algorithms were developed which can be exploited to accelerate spectral-based graph convolution \citep{Dong2017, LI2020188, zheng2020mathnet}.

\paragraph{Most related works} We discuss important papers that are close to and motivate our work. \cite{ChFiMh2015} (see also its extension to digraph in \cite{ChMhZh2018}) established an orthogonal system with localization properties on the tree and its approximation theory based on the notion of \emph{filtration} for a tree. Such an orthogonal system is indeed a Haar-type wavelet system on a tree. The idea of \cite{Haar1910} has been `invented', `reinvented' and `rejuvenated' in different domains.  By defining `$h$-hop' neighborhoods on a graph, \cite{CrKo2003} reduced the construction of wavelets on a graph to the classical Haar wavelets on the interval $[0,1)$. Using the concept of `folders' and `subfolders', Haar-like bases are constructed in \cite{GaNaCo2010} for high-dimensional data.
Directional Haar tight framelets exist in an arbitrary dimension $\R^d$, and their projection to lower dimensions induce the directional box splines \citep{HaLiZh2019, LiZh2019}. The application of Haar transforms on graphs for deep learning tasks can be seen in \cite{LI2020188,wang2019haar,zheng2020mathnet}.

In \cite{HaVaGr2011}, spectral graph theory is used to define scaling operator $T_g^t=g(t\mathcal{L})$ for functions on a graph, where $g$ is a  function, $t$ is a scaling parameter, and $\mathcal{L}$ is the graph Laplacian. A wavelet $\psi_{t,\vG}$ at scale $t$ and `centered' at a vertex $\vG$ can then be defined as such operator applied to a delta impulse signal, that is, $\psi_{t,\vG} = T_g^t\delta_{\vG}$. Spectral wavelet graph transforms (SGWT) are then implemented based on the Chebyshev polynomial approximation. \cite{Dong2017} constructed tight wavelet frames on both manifolds and graphs, where $\psi_{j,\vG}$ is defined from a scaling function $\phi$ and hence associated the construction of wavelet frames on a graph with a filter bank. Fast wavelet frame transforms on the graph (WFTG) can then be implemented based on the Chebyshev polynomial approximation of filters.

In \cite{WaZh2018}, tight framelets, including both continuous and semi-discrete tight framelets on a Riemannian manifold, are constructed based on orthogonal eigenpairs on the manifold, where localized filtered kernel functions were used to define \emph{scaling} and \emph{translation} on the manifold. Framelets on the manifold could then be constructed from the filtered kernels. More importantly, the framelets are associated with a filter bank that enables a simple characterization of tight framelets. In particular, quadrature rules (weighted sampling point sets) with polynomial exactness property naturally induced the construction of semi-discrete tight framelets on the manifold, which are significantly different from the construction of \cite{HaVaGr2011, Dong2017}. The quadrature-based framelets have fast tight framelet filter bank transforms (F$\mathcal{M}$T), which is useful for practical applications in manifold-data multiresolution analysis. This paper follows the line of \cite{WaZh2018} but is mainly focused on the multiscale data analysis on graphs and its fast algorithms.

\section{Preliminaries}
\label{sec:pre}
In this section, we introduce some basic notation and properties on graphs, frames and filter banks, which are based on the works of \citet{Chung1997,Dong2017,HaVaGr2011,WaZh2018,wang2019tight}.

\subsection{Graph and Chains}
\label{subsec:graph}

An undirected and weighted graph $\gph$ is an ordered triple $\gph=(\VG,\EG,\wG)$ with a non-empty finite set $\VG$ of vertices, a set  $\EG\subseteq \VG\times \VG$ of edges between vertices in $\VG$, and a non-negative weight function $\wG: \EG\rightarrow\R$.  For an undirected graph $\gph$, we denote $|\VG|$ and $|\EG|$ the numbers of vertices and edges. An edge $e\in\EG$ with vertices $\uG,\vG\in\VG$ is an unordered pair denoted by $(\uG,\vG)$ or $(\vG,\uG)$. In this paper, we assume the self-loops $(\vG,\vG)$, $\vG\in\VG$ in the edge set $\EG$, and we extend $\wG$ from $\EG$ to $\VG\times\VG$ by letting $\wG(\uG,\vG):=0$ for $(\uG,\vG)\notin\EG$. The extended weight function $\wG$ is also called the \emph{adjacency (representation) matrix} of $\gph$. Note that for an undirected graph, the weight $\wG$ is \emph{symmetric} in the sense that $\wG(\uG,\vG)=\wG(\vG,\uG)$ for all $\uG,\vG\in\VG$.
We denote the \emph{degree} of a vertex $\vG\in\vertex$ by
$
\dG(\vG) := \sum_{\uG\in\VG}\wG(\vG,\uG).
$
The \emph{volume} of the graph $\vol(\gph):=\vol(\VG)=\sum_{\vG\in\VG}\dG(\vG)$, which is the sum of degrees of all vertices of $\gph$. Given a subset $\VG_0$ of $\VG$, the \emph{volume} of $\VG_0$ is the sum of degrees of all nodes in $\VG_0$, that is, $\vol(\VG_0):=\sum_{\vG\in\VG_0} \dG(\vG)$.

Let $(\eG_1,\ldots,\eG_n)$ be a sequence of edges in $\gph$. If there exist distinct vertices $\vG_0,\ldots,\vG_n$ in $\VG$ such that any pair of consecutive nodes is connected by an edge of $\gph$, that is, $\eG_j = (\vG_{j-1},\vG_j)$ for $j=1,\ldots,n$, then the sequence $(\eG_1,\ldots,\eG_n)$ is called a \emph{path} of $\gph$ between $\vG_0$ and $\vG_n$, and the \emph{length} of the path is defined to be $\sum_{j=1}^n \wG(\vG_{j-1},\vG_j)$. Two vertices $\uG$ and $\vG$ are \emph{connected} if there exists a path between $\uG$ and $\vG$. If $\uG$ and $\vG$ are connected, the \emph{distance} $\distG(\uG,\vG)$ between two vertices $\uG$ and $\vG$ is defined as the length of the shortest possible path between them. We define $\distG(\uG,\vG)=\infty$ if there is no path between $\uG$ and $\vG$. A graph itself is called \emph{connected} if any two distinct vertices of $\gph$ are connected. In the paper, we only consider connected graphs.

Let $\gph=(\VG,\EG,\wG)$ and $\gph_c=(\VG_c,\EG_c,\wG_c)$ be two graphs. We say that $\gph_c$ is a \emph{coarse-grained graph} of $\gph$ if $\VG_c$ is a \emph{partition} of $V$; that is, there exist $k$ ($k>1$) subsets $\VG_1,\ldots,\VG_k$ of $\VG$ such that
\[
\VG_c =\{\VG_1,\VG_2,\ldots,\VG_k\},\quad
\VG_1\cup\cdots\cup\VG_k = \VG, \quad
\VG_i\cap \VG_j = \emptyset,\; 1\le i< j\le k.
\]
In this case, each vertex $V_j$ of $\gph_c$ is called a \emph{cluster} for $\gph$. The edges of $\gph_c$ are the links between clusters of $\gph$. The nodes in $\gph$ that are in the same cluster (node) of $\gph_c$ define an equivalence relation on $\gph$: two vertices $\uG$ and $\vG$ are equivalent, denoted by $\uG\sim\vG$, if $\uG$ and $\vG$ are in the same cluster. An equivalent class (cluster), as a vertex in $\gph_c$, associated with a vertex $\vG\in\VG$ can then be denoted as $[\vG]_{\gph_c}:=\{\uG\in\gph \setsep \uG\sim\vG\}$, and we let the set of clusters $\VG_c=\VG_{\sim} =\{[\vG]_{\gph_c}\setsep \vG\in\VG\}$. If no confusion arises, we will drop the subscript $\gph_c$ and simply use $[\vG]$ to denote a cluster in $\gph$ with respect to the coarse-grained graph $\gph_c$. We denote by $|[\vG]|$ or $\#[\vG]$ the number of vertices in the cluster $[\vG]$.

Let $J,J_0$, $J\ge J_0$ be two integers.
A \emph{coarse-grained chain} $\gph_{J\rightarrow J_0}:=(\gph_J, \gph_{J-1},\ldots,\gph_{J_0})$ of $\gph$ is a sequence of graphs with $\gph_J\equiv \gph$ such that each $\gph_j=(\VG_j,\EG_j,\wG_j)$ is a coarse-grained graph of $\gph$ for all $J_0\le j\le J$, and $[v]_{\gph_j}\subseteq [v]_{\gph_{j-1}}$ for all $j=J_0+1,\ldots, J$ and all $v\in\gph$. Here $[v]_{\gph_j}$ is the set of nodes of $\gph$ at level $j$ that contains node $v$, and can also be viewed as an (equivalence class) cluster of any finer level.
Note that in the coarse-grained chain $\gph_{J\to J_0}$, for any $p,v\in \gph$, $[p]_{\gph_j}$ is either belonging to $[v]_{\gph_j}$ or has no intersection with $[v]_{\gph_j}$. So the children at a finer level of any cluster $[v]_{\gph_j}$ form a partition of the node set $[v]_{\gph_j}$. 

We call $\gph_j$ \emph{the level-$j$ graph} of the chain $\gph_{J\to J_0}$. The $\gph_{j-1}$ can be viewed as a \emph{coarse-grained graph} of $\gph_j$ for $j=J_0+1,\ldots,J$. For convenience of discussion, we treat each vertex $\vG$ of the finest level graph $\gph_J\equiv\gph$ as a cluster of singleton, that is, $[\vG]_{\gph_J} =\{\vG\}$. We say the vertices of the graph on each level of the chain are \emph{nodes} of the chain. When $\gph_{J_0},\dots,\gph_J$ are not all equal, we call $\gph_{J\to J_0}$ a \emph{decimated chain}.
See Figure~\ref{fig:coarse-grain-toy} for an illustration of a coarse-grained chain.


When $\gph_j\equiv\gph$ for all $j=J_0,\ldots,J$, we call $\gph_{J\rightarrow J_0}$ an \emph{undecimated chain} of $\gph$. When there is only one vertex in the coarsest graph $\gph_{J_0}$, that is, $\VG_{J_0}=\{[\vG]_{\gph_{J_0}}\}=\{V\}$, we call $\gph_{J\rightarrow J_0}$ a \emph{tree} and denote it by $\tri$. The $\{V\}$ is the root of $\tri$ at the top level $J_0$, and all vertices of $\gph$ are the \emph{leaves} at the bottom level $J$. The vertices of $\gph_j$ (clusters of $\gph$) are the \emph{nodes} of the tree at level $j$. The node $[\vG]$ is called the parent of $\vG$.

\subsection{Orthonormal Bases on Graphs and Chains}\label{subsec:ONB}
Let $l_2(\gph):=l_2(\gph,\ipG{\cdot,\cdot}_\gph)$ be the Hilbert space of \emph{vectors} $\vf:\VG\rightarrow \C$ on the graph $\gph$ equipped with the inner product
\begin{equation*}
\ipG{\vf,\vg}_\gph:=\sum_{\vG\in\VG} \vf(\vG)\conj{\vg(\vG)}, \quad \vf,\vg\in l_2(\gph),
\end{equation*}
where $\conj{\vg}$ is the complex conjugate to $\vg$.
The induced norm $\nrmG{\cdot}_\gph$ is then given by $\nrmG{\vf}_\gph := \sqrt{\ipG{\vf,\vf}_\gph}$ for $\vf\in l_2(\gph)$.  For simplicity, we shall drop the subscript $\gph$, and use $\ipG{\cdot,\cdot}$ and $\nrmG{\cdot}$ if no confusion arises.

Let $\delta_{\ell,\ell'}$ be the \emph{Kronecker delta} satisfying $\delta_{\ell,\ell'}=1$ if $\ell=\ell'$ and $\delta_{\ell,\ell'}=0$ if $\ell\neq\ell'$, and $\NV:=|V|$ the total number of vertices of the graph $\gph$. A finite subset $\{\eigfm\}_{\ell=1}^{\NV}$ of $l_2(\gph)$ is said an \emph{orthonormal basis} for $l_2(\gph)$ if
\begin{equation*}
    \ipG{\eigfm,\eigfm[\ell']} = \delta_{\ell,\ell'}, \quad 1\le \ell,\ell'\le \NV.
\end{equation*}

Let $\{\eigfm\}_{\ell=1}^\NV$ be an orthonormal basis for $l_2(\gph)$. For $\ell=1,\dots,N$, let
\begin{equation*}
	\Fcoem{\vf}:= \ipG{\vf,\eigfm}
\end{equation*}
be the (generalized) \emph{Fourier coefficient} of degree $\ell$ for $\vf\in l_2(\gph)$ with respect to $\eigfm$. Let $\Fcoem{\vf}:=\bigl(\Fcoe[1]{\vf},\dots,\Fcoe[N]{\vf}\bigr)\in\C^N$ be the sequence of the Fourier coefficients for $\vf$.
Then,
$\vf = \sum_{\ell=1}^\NV \FT{\vf}_\ell \,\eigfm$ for all $f\in l_2(\gph)$,
and Parseval's identity holds: $\nrmG{\vf}^2=\sum_{\ell=1}^\NV|\FT{\vf}_\ell|^2$ for all $\vf\in l_2(\gph)$.
We say $\{(\eigfm,\eigvm)\}_{\ell=1}^{\NV}$ is an \emph{orthonormal eigen-pair} for $l_2(\gph)$ if $\{\eigfm\}_{\ell=1}^\NV$ is an orthonormal basis for $l_2(\gph)$ with $\eigfm[1]\equiv1/\sqrt{\NV}$ and $\{\eigvm\}_{\ell=1}^{\NV}\subseteq\R$ is a nondecreasing sequence of nonnegative numbers satisfying $0=\eigvm[1]\le \cdots \le \eigvm[\NV]$. A typical example is the \emph{eigen-pairs}, that is, the set of all pairs of the eigenvectors and eigenvalues of the \emph{graph Laplacian} on $\gph$. The (combinatorial or unnormalized) \emph{graph Laplacian} operator $\gL: l_2(\gph)\rightarrow l_2(\gph)$ is
\begin{equation}\label{defn:gL}
	[\gL \vf](\uG):=\dG(\uG) \vf(\uG)-\sum_{\vG\in \VG}\wG(\uG,\vG)\vf(\vG), \quad \uG\in \VG,\; \vf\in l_2(\gph).
\end{equation}
One can verify that
$\ipG{\vf,\gL \vf} =\frac12\sum_{\uG,\vG} \wG(\uG,\vG)|\vf(\uG)-\vf(\vG)|^2\ge0$. The eigenvalues $\widetilde{\eigvm}$ of $\gL$ are then nonnegative, and associated with eigenvectors $\eigfm: \gL \,\eigfm = \widetilde{\eigvm}\,\eigfm, \;\ell=1,\ldots, \NV$. For simplicity, we need to take the square root of eigenvalue to define $\eigvm:=\sqrt{\widetilde{\eigvm}}$, which is to meet the construction of framelets (see below). To be precise, the set of eigenvectors $\{\eigfm\}_{\ell=1}^\NV$ then  forms  an orthonormal basis for $l_2(\gph)$ satisfying $0=\eigvm[1]\le\ldots\le \eigvm[\NV]$ and $\eigfm[1] \equiv 1/\sqrt{N}$.
An orthonormal eigen-pair can be deduced from other positive semi-definite operators on $l_2(\gph)$, for example, diffusion operators \citep{MaMh2008}, or by interval decomposition method, such as Haar orthonormal basis \citep{ChFiMh2015}.

Let $\gph_{J\to J_0}:=(G_{J},\dots,G_{J_0})$ be a chain of the graph $\gph$ with $N$ vertices. Let $l(\gph_{J\to J_0})$ be the set of all vectors $\vf$ defined on the union of vertices at all levels $V_J\cup \cdots \cup V_{J_0}$.
\begin{definition}\label{defn:orth.basis.chain}
	A set of pairs of vectors and complex numbers $\{(\eigfm,\eigvm)\}_{\ell=1}^{N}$ in $l(\gph_{J\to J_0})$ is called an orthonormal basis for the chain $\gph_{J\to J_0}$ if the restriction $\{(\eigfm|_{\gph_{j}},\eigvm)\}_{\ell=1}^{N}$ on the $j$th level graph $\gph_{j}$ is an orthonormal basis for $l_2(\gph_j)$ at each level $j=J_0,\dots,J$.
\end{definition}

In Sections~\ref{sec:orthonormal.basis.Laplacian} and \ref{sec:haar.orth.basis}, we will show the construction of an orthonormal basis for chain $\gph_{J\to J_0}$ by either graph Laplacian or interval decomposition.

\subsection{Tight Frames and Filter Banks}
\label{subsec:frameletR}
Orthonormal bases as we mentioned above are non-redundant systems for $l_2(\gph)$. In this paper, we would mainly focus on \emph{redundant systems} for $l_2(\gph)$ that are frames. 
Let $\{\vg_\ell\}_{\ell=1}^M$ be a set of elements from $l_2(\gph)$. We say that $\{\vg_\ell\}_{\ell=1}^M$ is a \emph{frame} for $l_2(\gph)$ if there exist constants $A$ and $B$, $0<A\le B< \infty$ such that
\begin{equation}\label{defn:frame}
	A\nrmG{\vf}^2\le \sum_{\ell=1}^M |\ipG{\vf,\vg_\ell}|^2 \le B\nrmG{\vf}^2\quad \forall \vf \in l_2(\gph).
\end{equation}
Here, we call $A,B$ the \emph{frame bounds}.
When the frame bounds $A=B=1$, we say $\{\vg_\ell\}_{\ell=1}^M$ a \emph{tight frame} for $l_2(\gph)$, and by the polarization identity, \eqref{defn:frame} is then equivalent to
\begin{equation}\label{defn:frame2}
	\vf = \sum_{\ell=1}^M \ipG{\vf,\vg_\ell}\vg_\ell.
\end{equation}
When $\{\vg_\ell\}_{\ell=1}^M$ is a tight frame and $\nrmG{\vg_\ell}=1$ for $\ell=1,\dots,M$, we must have $M=\NV$ and $\{\vg_\ell\}_{\ell=1}^\NV$ an orthonormal basis for $l_2(\gph)$. Tight frames are of significance as we can use coefficients $\ipG{\vf,\vg_\ell}$ to represent the vector $\vf$. See \cite{Daubechies1992}.

Let $\Psi:=\{\scala; \scalb^{(1)},\ldots,\scalb^{(r)}\}$ be a set of functions in $L_1(\R)$, which is the space of absolutely integrable functions on $\R$ with respect to the Lebesgure measure. The \emph{Fourier transform} $\FT{\gamma}$  of a function $\gamma\in L_1(\R)$ is defined by $\FT{\gamma}(\xi):=\int_{\R}\gamma(t)e^{-2\pi i t\xi} \IntD{t}$, $\xi\in\R$. The Fourier transform on $L_1(\R)$ can be naturally extended to the space $L_2(\R)$ of square integrable functions on $\R$. See, for example, \cite{stein2011fourier}.

A \emph{filter (or mask)} $\mask:=\{\mask_k\}_{k\in\Z}\subseteq \C$ is a complex-valued sequence in $l_1(\Z):=\{h=\{h_k\}_{k\in\Z}\subseteq\C \setsep \sum_{k\in\Z} |h_k|<\infty \}$. A \emph{filter bank} is a set of filters.
The \emph{Fourier series} of a sequence $\{\mask_k\}_{k\in\Z}$ is the $1$-periodic function $\FT{\mask}(\xi):=\sum_{k\in\Z}\mask_k e^{-2\pi i k\xi}$, $\xi\in\R$.
Let $\Psi_j=\{\scala; \scalb^{(1)},\ldots,\scalb^{(r)}\}$ be a set of \emph{framelet generators} associated with a filter bank $\filtbk:=\{\maska; \maskb[1],\ldots,\maskb[r]\}$, where the Fourier transforms of the functions in $\Psi$ and the Fourier series of the filters in $\filtbk$ satisfy
\begin{equation}
\label{eq:refinement}
    \FT{\scala}(2\xi) = \FS{\maska}(\xi)\FT{\scala}(\xi),\quad
    \FT{\scalb^{(n)}}(2\xi) = \FS{\maskb}(\xi)\FT{\scala}(\xi),\quad n=1,\ldots,r, \; \xi\in\Rone.
\end{equation}
The first equation in \eqref{eq:refinement} is called the \emph{refinement equation}. The $\maska$ is called \emph{low-pass filter} or \emph{refinement mask} and $\maskb$, $n=1,\ldots,r$ are called \emph{high-pass filters} or \emph{framelet masks}.
One can then consider (stationary non-homogeneous) \emph{affine system} of the form
\begin{equation}\label{defn:AS:R}
\AS_J(\Psi)=\{\scala(2^J\cdot-k)\setsep k\in\Z\}\cup\{\scalb^{(n)}(2^j\cdot-k) \setsep k\in\Z, n=1,\ldots,r, j\ge J\}.
\end{equation}
Under certain extension principles such as the unitary extension principle (UEP) \citep{RoSh1997,DaHaRoSh2003}, one can build the affine system $\AS_J(\Psi)$ that is a tight frame for $L_2(\R)$. We call the elements $\AS_J(\Psi)$ \emph{tight framelets} for $L_2(\R)$ in this case.
More generally, one can consider the (non-stationary, non-homogeneous) affine system
\begin{equation}\label{defn:nAS:R}
\AS_J(\{\Psi_j\}_{j=J}^\infty)=\{\scala_{j}(2^J\cdot-k)\setsep k\in\Z\}\cup\{\scalb_{j}^{(n)}(2^j\cdot-k) \setsep k\in\Z, n=1,\ldots,r_j, j\ge J\},
\end{equation}
where $\Psi_j:=\{\scala_{j};\scalb^{(1)}_j,\ldots,\scalb^{(r_j)}_j\}\subset L_2(\R)$ are framelet generators at scale $j$. Examples of such (non-stationary, non-homogeneous) affine systems include framelets on $\Rd$ \citep{Han2012,Zhuang2016}, and framelets on compact Riemannian manifolds \citep{WaZh2018}.

In the paper, we use the same ``\,$\widehat{\,\cdot\,}\,$'' notation for Fourier coefficients, Fourier transforms and Fourier series for simplicity.
For a set $\Omega$, let $l_2(\Omega)$ represent the set of all $l_2$-summable sequences on $\Omega$ and $l(\Omega)$ be the set of all complex-valued sequences supported on $\Omega$. For a finite set $\Omega$, a complex-valued sequence supported on $\Omega$ is a $l_2$-summable sequence. Thus, $l(\Omega)\equiv l_2(\Omega)$. Then, we only use the notation $l_2(\Omega)$ when we discuss a space of sequences on a finite set $\Omega$.

\section{Characterization of Tight Framelets on $\gph$}\label{sec:char}
In this section, we construct undecimated and decimated framelet systems on a graph and give equivalence characterizations for the tightness of them.

Given a set of orthonormal eigen-pairs for a coarse-grained chain $\gph_{J\to J_0}$ of the graph $\gph$, we can define a (stationary) \emph{undecimated framelet system} on $\gph$ by a set of framelet generators $\Psi:=\{\scala;\scalb^{(1)},\ldots,\scalb^{(r)}\}$ and a filter bank $\filtbk:=\{\maska;\maskb[1],\ldots,\maskb[r]\}$. Alternatively, we can construct a \emph{decimated framelet system} with a sequence of framelet generator sets $\Psi_J,\dots,\Psi_{J_0}$ and a sequence of filter banks $\filtbk_J,\dots,\filtbk_{J_0}$. Here, we use ``undecimated'' when the framelet system is invariant on all levels of the coarse-grained chain, and ``decimated'' when a sequence of framelets has different scales at different levels of a chain. Each framelet generator and the corresponding filter bank are associated with the graph of a particular level.

\subsection{Undecimated Tight Framelets on $\gph$}
\label{sec:UFS}

The undecimated framelets on a graph can be seen as framelets on the discretised manifold, which are important counterpart to compact Riemannian manifolds \citep{WaZh2018}. Here, we give detailed construction of undecimated tight framelets on a graph.

The construction of framelets uses the graph spectrum and framelet generators. Let $\{(\eigfm,\eigvm)\}_{\ell=1}^\NV$ be orthonormal eigen-pairs for $l_2(\gph)$
and $\Psi=\{\scala;\scalb^{(1)},\ldots,\scalb^{(r)}\}$ be a set of functions in $L_1(\R)$ associated with a filter bank $\filtbk=\{\maska;\maskb[1],\ldots,\maskb[r]\}$ satisfying \eqref{eq:refinement}.
For $j\in\Z$ and $\uG\in\VG$, the \emph{undecimated framelets} $\cfra(\vG)$ and $\cfrb{n}(\vG)$, $\vG\in\VG$ at scale $j$ are \emph{filtered Bessel kernels} (or summability kernels)
\begin{equation}\label{defn:ufra:ufrb}
\begin{aligned}
    \cfra(\vG) :=&  \sum_{\ell=1}^{\NV} \FT{\scala}\left(\frac{\eigvm}{2^{j}}\right)\conj{\eigfm(\uG)}\eigfm(\vG),\\
    \cfrb{n}(\vG) :=&  \sum_{\ell=1}^{\NV} \FT{\scalb^{(n)}}\left(\frac{\eigvm}{2^{j}}\right)\conj{\eigfm(\uG)}\eigfm(\vG), \quad n = 1,\ldots,r.
\end{aligned}
\end{equation}
See for example, \citet{Brauchart_etal2015,MaMh2008}.
Here, $j$ and $\uG$ in  $\cfra(\vG)$ and $\cfrb{n}(\vG)$ indicate the ``dilation'' at scale $j$ and the ``translation'' at a vertex $\uG\in\VG$. They are analogues of those of wavelets in $\Rd$. The functions $\alpha,\beta^{n}$ of $\Psi$ are called \emph{framelet generators} or \emph{scaling functions} for the undecimated framelet system.

Let $J, J_1$, $J > J_1$ be two integers. An \emph{undecimated framelet system} $\ufrsys[]\left(\Psi,\filtbk;\gph\right)$ (starting from a scale $J_1$) is a (non-homogeneous, stationary) affine system:
\begin{equation}\label{defn:UFS}
\begin{aligned}
   \ufrsys[J_1]^{J}(\Psi,\filtbk)
   &:=\ufrsys[J_1]^J(\Psi,\filtbk;\gph) \\
   &:=\{\cfra[J_1,\uG] \setsep \uG\in\VG\} \cup\{\cfrb{n} \setsep \uG\in\VG, n = 1,\ldots,r, j= J_1,\ldots, J\}.
\end{aligned}
\end{equation}
The system $\ufrsys[J_1]^{J}(\Psi,\filtbk)$ is then called an \emph{undecimated tight frame} for $l_2(\gph)$ and the elements in $\ufrsys[J_1]^{J}(\Psi,\filtbk)$ are called \emph{undecimated tight framelets} on $\gph$.

The sequence $\{\ufrsys[J_1]^{J}(\Psi,\filtbk)\setsep J_1=J_0,J_0+1,\ldots,J\}$ of undecimated framelet systems provides a tool for the multiresolution analysis on $l_2(\gph)$, which passes messages between finer and coarser scales of a chain on $\gph$. The precision of \emph{transforms} is ensured by the tightness of the framelet system. The following theorem gives important equivalence conditions for the tightness of a sequence of undecimated framelet systems.

\begin{theorem}\label{thm:UFS}
Let $\gph=(\VG,\EG,\wG)$ be a graph and $\{(\eigfm,\eigvm)\}_{\ell=1}^\NV$ a set of orthonormal eigen-pairs for $l_2(\gph)$. Let $\Psi=\{\scala;\scalb^{(1)},\ldots,\scalb^{(r)}\}$ be a set of functions in $L_1(\R)$ associated with a filter bank $\filtbk=\{\maska;\maskb[1],\ldots,\maskb[r]\}$ satisfying \eqref{eq:refinement}. Let $J_0, J$, $J_0\le J$, be two integers, which indicate the coarsest scale and the finest scale of a chain. Let $\ufrsys[J_1]^{J}(\Psi,\filtbk;\gph), J_1=J_0,\ldots, J$, be an undecimated framelet system given in \eqref{defn:UFS} with framelets $\cfra$ and $\cfrb{n}$ given in \eqref{defn:ufra:ufrb}. Then, the following statements are equivalent.
\begin{enumerate}[{\rm(i)}]
\item The undecimated framelet system $\ufrsys[J_1]^{J}(\Psi,\filtbk;\gph)$ is a tight frame for $l_2(\gph)$ for each $J_1=J_0,\ldots,J$. That is,
    \begin{equation}\label{eq:f:UFS0}
    \nrmG{\vf}^2 =\sum_{\uG\in\VG}\Big|\ipG{\vf,\cfra[J_1,\uG]}\Big|^2
    +\sum_{j=J_1}^{J}\sum_{n=1}^r\sum_{\uG\in\VG}\Big|\ipG{\vf,\cfrb{n}}\Big|^2\quad\forall \vf\in l_2(\gph), J_1=J_0,\ldots,J.
    \end{equation}

\item  For all $\vf\in l_2(\gph)$ and for $j=J_0,\ldots,J-1$, the following identities hold:
\begin{align}
&   \vf = \sum_{\uG\in\VG} \ipG{\vf,\cfra[J,\uG]}\cfra[J,\uG]
+\sum_{n=1}^r\sum_{\uG\in\VG}\ipG{\vf,\cfrb[J,\uG]{n}}\cfrb[J,\uG]{n},
\label{thmeq:normalization1}\\
& \sum_{\uG\in\VG} \ipG{\vf,\cfra[j+1,\uG]}\cfra[j+1,\uG]
= \sum_{\uG\in\VG} \ipG{\vf,\cfra}\cfra+
\sum_{n=1}^{r}\sum_{\uG\in\VG} \ipG{\vf,\cfrb{n}}\cfrb{n}. \label{thmeq:2scale1}
\end{align}

\item For  all $\vf\in l_2(\gph)$ and for $j=J_0,\ldots, J-1$, the following identities hold:
\begin{align}
           & \nrmG{\vf}^2 = \sum_{\uG\in\VG} \bigl|\ipG{\vf,\cfra[J,\uG]}\bigr|^{2}
           +\sum_{n=1}^r\sum_{\uG\in\VG}\bigl|\ipG{\vf,\cfrb[J,\uG]{n}}\bigr|^{2}, \quad\label{thmeq:normalization2}\\
            & \sum_{\uG\in\VG} \bigl|\ipG{\vf,\cfra[j+1,\uG]}\bigr|^{2}
            = \sum_{\uG\in\VG} \bigl|\ipG{\vf,\cfra}\bigr|^{2} + \sum_{n=1}^{r}\sum_{\uG\in\VG} \bigl|\ipG{\vf,\cfrb{n}}\bigr|^{2}.&\label{thmeq:2scale2}
\end{align}

\item The functions in $\Psi$ satisfy
\begin{align}
   1 = \left|\FT{\scala}\left(\frac{\eigvm}{2^{J}}\right)\right|^{2} + \sum_{n=1}^{r}\left|\FT{\scalb^{(n)}}\left(\frac{\eigvm}{2^{J}}\right)\right|^{2} &\quad \forall
  \ell=1,\ldots,\NV, \label{thmeq:nrm:alpha:beta}\\
     \left|\FT{\scala}\left(\frac{\eigvm}{2^{j+1}}\right)\right|^{2}
                = \left|\FT{\scala}\left(\frac{\eigvm}{2^{j}}\right)\right|^{2} + \sum_{n=1}^{r}\left|\FT{\scalb^{(n)}}\left(\frac{\eigvm}{2^{j}}\right)\right|^{2} &\quad \forall
 \begin{array}{l}
 \ell=1,\ldots,\NV,\\
 j=J_0,\ldots,J-1.
 \end{array}\label{thmeq:2scale:alpha:beta}
 \end{align}

\item The identities in \eqref{thmeq:nrm:alpha:beta} hold  and  the filters in the filter bank $\filtbk$ satisfy
\begin{align}
  \left|\FS{\maska}\left(\frac{\eigvm}{2^{j}}\right)\right|^{2} + \sum_{n=1}^{r} \left|\FS{\maskb}\left(\frac{\eigvm}{2^{j}}\right)\right|^{2} = 1 \quad \forall \ell\in\sigma_{\scala}^{(j)},\; j = J_0+1,\ldots,J,
\label{thmeq:2scale:masks}
 \end{align}
where
\[
\sigma_{\scala}^{(j)}:=\left\{\ell\in\{1,\ldots,\NV\} \setsep \FT{\scala}\left(\frac{\eigvm}{2^j}\right) \neq 0\right\}.
\]
\end{enumerate}
\end{theorem}

\begin{proof}
(i)$\Longleftrightarrow$(ii). Let $\mathcal{V}_j:=\mathrm{span}\{\cfra: \uG\in\VG\}$ and $\mathcal{W}_j^n:=\mathrm{span}\{\cfrb{n}: \uG\in\VG\}$. Define projections  $\cfrpra{j}, \cfrprb{j}$, $n=1,\dots,r$ as
\begin{equation}\label{eqs:proj.ufr}
    \cfrpra{j}(\vf) := \sum_{\uG\in\VG} \ipG{\vf,\cfra}\cfra,\quad
    \cfrprb{j}(\vf) := \sum_{\uG\in\VG} \ipG{\vf,\cfrb{n}}\cfrb{n},\quad \vf\in l_2(\gph).
\end{equation}
Since $\ufrsys[J_1]^{J}(\Psi,\filtbk)$ is a tight frame for $l_2(\gph)$ for  $J_1= J_0,\ldots,J$, by polarization identity,
\begin{equation}\label{eq:f:UFS}
\begin{aligned}
  \vf = \cfrpra{J_1}(\vf) + \sum_{j=J_1}^{J}\sum_{n=1}^{r} \cfrprb{j}(\vf)
    = \cfrpra{J_1+1}(\vf) + \sum_{j=J_1+1}^{J}\sum_{n=1}^{r} \cfrprb{j}(\vf)
\end{aligned}
\end{equation}
for all $\vf\in l_2(\gph)$ and for all $J_1= J_0,\ldots,J$.
Thus,  for $J_1=J_0,\ldots,J-1$, we have
\begin{equation}\label{eq:ufr.pr.J.J1}
 \cfrpra{J_1+1}(\vf) =  \cfrpra{J_1}(\vf) + \sum_{n=1}^{r} \cfrprb{J_1}(\vf),
\end{equation}
which is \eqref{thmeq:2scale1}. Moreover, when $J_1=J$, \eqref{eq:f:UFS} gives \eqref{thmeq:normalization1}.  Consequently, (i)$\Longrightarrow$(ii).
Conversely,  recursively using \eqref{eq:ufr.pr.J.J1} gives
\begin{equation}\label{eq:cfrpra.m1}
  \cfrpra{m+1}(\vf) = \cfrpra{J_1}(\vf) + \sum_{j=J_1}^{m}\sum_{n=1}^{r} \cfrprb{j}(\vf)
\end{equation}
for all $J_1\le m\le J-1$. Taking $m=J-1$ together with \eqref{thmeq:normalization1}, we deduce \eqref{eq:f:UFS},
which is equivalent to \eqref{eq:f:UFS0}. Thus, (ii)$\Longrightarrow$(i).

(ii)$\Longleftrightarrow$(iii). The equivalence between (ii) and (iii) simply follows from the polarization identity.

(ii)$\Longleftrightarrow$(iv). By the orthonormality of $\eigfm$,
\begin{equation*}\label{eq:cfr.coeff}
    \ipG{\vf,\cfra} = \sum_{\ell=1}^{\NV} \conj{\FT{\scala}\left(\frac{\eigvm}{2^{j}}\right)}\Fcoem{\vf}\:\eigfm(\uG), \quad
    \ipG{\vf,\cfrb{n}} = \sum_{\ell=1}^{\NV} \conj{\FT{\scalb^{(n)}}\left(\frac{\eigvm}{2^{j}}\right)}\Fcoem{\vf}\:\eigfm(\uG),
\end{equation*}
where $\FT{\vf}_\ell=\ipG{\vf,\eigfm}$ is the Fourier coefficient of $\vf$ with respect to $\eigfm$.
This together with \eqref{eqs:proj.ufr} and \eqref{defn:ufra:ufrb} gives, for $j\ge J_0$ and $n=1,\dots,r$, the Fourier coefficients for the projections $\cfrpra{j}(\vf)$ and $\cfrprb{j}(\vf)$:
\begin{equation}\label{eq:Fcoe.cfrpr}
    \Fcoem{\left(\cfrpra{j}(\vf)\right)}
    = \left|\FT{\scala}\left(\frac{\eigvm}{2^{j}}\right)\right|^{2} \Fcoem{\vf},\quad
    \Fcoem{\left(\cfrprb{j}(\vf)\right)}
    = \left|\FT{\scalb^{(n)}}\left(\frac{\eigvm}{2^{j}}\right)\right|^{2} \Fcoem{\vf},\quad  \ell=1,\ldots,\NV,
\end{equation}
which implies that \eqref{thmeq:normalization1} and \eqref{thmeq:2scale1} are equivalent to \eqref{thmeq:nrm:alpha:beta} and \eqref{thmeq:2scale:alpha:beta} respectively.
Thus, (ii)$\Longleftrightarrow$(iv).

(iv)$\Longleftrightarrow$(v).  By the relation in \eqref{eq:refinement}, it can be deduced that for $\ell\ge0$ and $j\ge J_0$,
\begin{align*}
     \left|\FT{\scala}\left(\frac{\eigvm}{2^{j}}\right)\right|^{2} + \sum_{n=1}^{r}\left|\FT{\scalb^{(n)}}\left(\frac{\eigvm}{2^{j}}\right)\right|^{2}
    = \left(\left|\FT{\maska}\left(\frac{\eigvm}{2^{j+1}}\right)\right|^{2} + \sum_{n=1}^{r}\left|\FT{\maskb}\left(\frac{\eigvm}{2^{j+1}}\right)\right|^{2}\right)\left|\FT{\scala}\left(\frac{\eigvm}{2^{j+1}}\right)\right|^{2}.
\end{align*}
This shows that \eqref{thmeq:2scale:alpha:beta} is equivalent to \eqref{thmeq:2scale:masks}. Therefore, (iv)$\Longleftrightarrow$(v).
\end{proof}

When $\FT{\scala}(\xi)=1$ in a neighborhood of the origin, the system in \eqref{defn:UFS} is simplified by the following property. Let $\Dirac_\uG\in l_2(\gph)$ be the \emph{Dirac sequence} defined by that $\Dirac_\uG(\uG)=1$ and $\Dirac_\uG(\vG)=0$ if $\uG\neq\vG$
\begin{proposition}
Suppose there exists a constant $C>0$ such that
\begin{equation}\label{eq:alpha1}
\FT{\scala}(\xi) \equiv 1\quad \forall |\xi|\le C.
\end{equation}
Then, for sufficiently large $j$, we have
$\cfra[j,\uG] = \Dirac_\uG$.
\end{proposition}
\begin{proof}
In view of \eqref{eq:alpha1}, we have for sufficient large $j$, $\FT{\scala}(\eigvm/2^j)\equiv1$ for all $\ell=1,\ldots,\NV$. Then, by the orthonormality of $\{\eigfm\}_{\ell=1}^\NV$, we have
\begin{equation*}
	\cfra[j,\uG](\vG) = \sum_{\ell=1}^{\NV}\conj{\eigfm(\uG)}\eigfm(\vG) = \Dirac_\uG(\vG), \quad\vG\in\VG.
\end{equation*}
\end{proof}
Hence, when \eqref{eq:alpha1} holds and $J_1$ is sufficiently large, the undecimated framelet system  $\ufrsys[J_1]^{J}(\Psi,\filtbk)$ defined in \eqref{defn:UFS} becomes
\[
\ufrsys[J_1]^{J}(\Psi,\filtbk) \equiv \{\Dirac_\uG\setsep \uG\in\VG\},
\]
which is a trivial orthonormal basis for $l_2(\gph)$. We observe from this case that the framelet systems become more and more localized as the level $J_1$ increases.

\subsection{Decimated Tight Framelets on $\gph$}\label{sec:dfs}
In this section, we introduce a \emph{decimated tight framelet system} on a chain of a graph $\gph$. The scaling of a framelet is associated with the level of the associated graph in the chain. The framelets at level $j$ ($j=J_0,\dots,J$) take the vertices of the $j$th-level graph of the chain as their transition points. By doing this, the number of framelets equals to the nodes of the chain. The redundancy of the representation system is then determined by the clustering that constructed the chain of the graph. The decimated framelet system, compared with the undecimated framelet system, compresses the graph size at a coarser scale and saves the storage while no information in the framelet representation is lost.

Let $\gph=(\VG,\EG,\wG)$ be a graph and $\gph_{J\rightarrow J_0}:=(\gph_J,\ldots,\gph_{J_0})$ be a coarse-grained chain of $\gph$.
For each vertex $\pV$ in $\gph_j=(\VG_j,\EG_j,\wG_j)$, we assign a real number $\wV\in\R$, which we call the \emph{(associated) weight}.
At the bottom level with $j=J$, we let $\wV[J,{[\uG]_{\gph_J}}]\equiv1$ for all $[\uG]_{\gph_J}=\{\uG\}$ in $\VG_J$.
Let $\QN[j]:=\{\wV: \pV\in\VG_j\}$ be the set of weights on $\gph_j$ and $\QN[J\rightarrow J_0]:=(\QN[J],\ldots,\QN[J_0])$ be the sequence of sets of weights for the coarse-grained chain $\gph_{J\rightarrow J_0}$. 


\begin{definition}[Scaling functions]
	Let $\Psi_j=\{\scala_{j};\scalb^{(1)}_j,\ldots,\scalb^{(r_j)}_j\}$ be a set of functions in $L_1(\R)$ at scale $j$ for $j=J_0,\ldots, J$. We link the framelet generators in $\Psi_{j}$ and $\Psi_{j-1}$ by a filter bank $\filtbk_j:=\{\maska_{j};\maskb[1]_j,\ldots,\maskb[r_{j-1}]_j\}$ in that, for $\xi\in\R$ and $0<\Lambda_{J_0}\le \Lambda_{J_0+1}\le \cdots\le \Lambda_{J}<\infty$,
\begin{equation}
\label{eq:refinement:nonstationary}
\begin{aligned}
    \FT{\scala_{j-1}}(\xi/\Lambda_{j-1}) &= \FS{\maska_{j}}(\xi/\Lambda_{j})\FT{\scala_{j}}(\xi/\Lambda_j),\\
    \FT{\scalb^{(n)}_{j-1}}(\xi/\Lambda_{j-1}) &= \FS{\maskb_{j}}(\xi/\Lambda_j)\FT{\scala_{j}}(\xi/\Lambda_j), \quad n=1,\ldots,r_{j-1},
\end{aligned}
\end{equation}
where $\Lambda_{J_0},\Lambda_{J_0+1},\dots,\Lambda_J$ are called \emph{scaling factors}.
We call $\scala_{j};\scalb^{(1)}_j,\ldots,\scalb^{(r_j)}_j$ are the framelet generators or scaling functions of level $j$.
\end{definition}
\begin{definition}[Decimated framelets]\label{defn:decimated.framlets}
	The decimated framelets $\fra(\vG)$ and $\frb{n}(\vG)$, $\uG,\vG\in\VG$, at scale $j=J_0,\ldots,J$ for the chain $\gph_{J\to J_0}$ of the graph $\gph$, framelet generators in \eqref{eq:refinement:nonstationary} and a weights sequence $\QN[J\rightarrow J_0]$ are defined by
\begin{equation}
\label{defn:fra.frb}
\begin{aligned}
    \fra(\vG) &:= \sqrt{\wV}\sum_{\ell=1}^{\NV} \FT{\scala_{j}}\left(\frac{\eigvm}{\Lambda_{j}}\right)\conj{\eigfm(\pV)}\eigfm(\vG), \quad [\uG]\in\VG_j,\\
    \frb{n}(\vG) &:= \sqrt{\wV[j+1,{[\uG]}]}\sum_{\ell=1}^{\NV} \FT{\scalb_{j}^{(n)}}\left(\frac{\eigvm}{\Lambda_{j}}\right)\conj{\eigfm(\pV)}\eigfm(\vG), \quad \pV\in\VG_{j+1}, \;n=1,\ldots, r_j,
\end{aligned}
\end{equation}
where for $j=J$, we let $V_{J+1}:=V_J$ and $\wV[J+1,{[\uG]}]:=\wV[J,{[\uG]}]$. We call $\fra$ and $\frb{n}$ low-pass and high-pass (decimated) framelets at scale $j$.
\end{definition}

The decimated framelets in Definition~\ref{defn:decimated.framlets} are constructed based on the orthonormal basis associated with the chain $\gph_{J\to J_0}$.
Here, the function $\eigfm(\pV_{\gph_j})$ can either be defined as $\eigfm(\pV_{\gph_j}):=\frac{1}{\#\pV_{\gph_j}}\sum_{\vG\in\pV_{\gph_j}}\eigfm(\vG)$, or $\eigfm(\pV_{\gph_j}):=\max_{\vG\in\pV_{\gph_j}}\eigfm(\vG)$. In \eqref{defn:fra.frb}, the vertices $[\uG]=[\uG]_{\gph_{j+1}}\in \VG_{j+1}$ for the high-pass framelet $\frb{n}$ of level $j$ are at the $(j+1)$th level while the vertices $[p]_{\gph_j}$ in $ V_{j}$ for $\fra$ of level $j$ are at the $j$th level. These can be interpreted from the view of multiresolution analysis that $\frb{n}$ in fact lies in the scale $j+1$ while $\fra$ is in the scale $j$.


\begin{definition}[Decimated framelet system]\label{def:DFS}
	The (decimated) framelet system\\ $\dfrsys(\{\Psi_j\}_{j=J_1}^J,\{\filtbk_j\}_{j=J_1+1}^{J})$ on $\gph$ (starting from a scale $J_1$) is a non-homogeneous, non-stationary affine system which is a set of low-pass and high-pass framelets given by
\begin{equation}\label{defn:DFS}
\begin{aligned}
\dfrsys(\{\Psi_j\}_{j=J_1}^J,\{\filtbk_j\}_{j=J_1+1}^{J}) 
&:= \dfrsys(\{\Psi_j\}_{j=J_1}^J,\{\filtbk_j\}_{j=J_1+1}^{J};\gph_{J\rightarrow J_1}, \QN[J\rightarrow J_1]) \\ 
&:=
\{\fra[{J_1},\pV] \setsep \pV\in\VG_{J_1}\}\cup \{\frb{n} \setsep \pV\in \VG_{j+1}, \: j = J_1,\ldots,J\}.
\end{aligned}
\end{equation}
\end{definition}
%
%
For $\QN[j]=\{\wV: \pV\in\VG_j\}$ on $\gph_j$, we define
\begin{equation}\label{eq:sum:wt:u}
 \QU{\ell'}(\QN[j]):=\sum_{\pV\in\VG_j}\wV\:\eigfm(\pV)\overline{\eigfm[\ell'](\pV)}.
\end{equation}
Note that $\QU{\ell'}(\QN[J]) = \delta_{\ell,\ell'}$ since $\wV[J,{[\uG]}]\equiv1$ and $[\uG]_{\gph_J}=\{\uG\}$ is a singleton.

Similar to the undecimated setting, the following then gives equivalence characterizations for the tightness of a decimated framelet system $\{\dfrsys(\{\Psi_j\}_{j=J_1}^J,\{\filtbk_j\}_{j=J_1+1}^{J})\setsep J_1=J_0,J_0+1,\ldots,J\}$.

\begin{theorem}
\label{thm:DFS}
Let $\Psi_j:=\{\scala_{j};\scalb^{(1)}_j,\ldots,\scalb^{(r_j)}_j\}$, $j=J_0,\ldots, J$ be a sequence of framelet generators sets  in $L_1(\R)$ associated with a sequence of filter banks $\filtbk_j=\{\maska_{j};\maskb[1]_j,\ldots,\maskb[r_{j-1}]_j\}$, $j=J_0+1,\ldots, J$, see \eqref{eq:refinement:nonstationary}.
Let $\gph_{J\rightarrow J_0}$ be a coarse-grained chain of a graph $\gph$ with a weight sequence $\QN[J\rightarrow J_0]$. Let $\dfrsys(\{\Psi_j\}_{j=J_1}^J,\{\filtbk_j\}_{j=J_1+1}^{J})$, $J_1=J_0,\ldots, J$ be a sequence of decimated framelet systems for the chain $\gph_{J\rightarrow J_0}$ in Definition~\ref{def:DFS} with framelets in Definition~\ref{defn:decimated.framlets} with framelet generators $\Psi_j$. Then, the following statements are equivalent.
 \begin{enumerate}[{\rm(i)}]
 \item The decimated framelet system $\dfrsys(\{\Psi_j\}_{j=J_1}^J,\{\filtbk_j\}_{j=J_1+1}^{J})$  is a tight frame for $l_2(\gph)$ for all $J_1=J_0,\ldots,J$, that is,
    \begin{equation}\label{eq:f:DFS0}
    \hspace{-6mm}\nrmG{\vf}^2
    =\sum_{[\uG]\in\VG_{J_1}}\Big|\ipG{\vf,\fra[J_1,\pV]}\Big|^2
    +\sum_{j=J_1}^{J}\sum_{n=1}^{r_j}\sum_{[\uG]\in\VG_{j+1}}\Big|\ipG{\vf,\frb{n}}\Big|^2\quad\forall \vf\in l_2(\gph), J_1=J_0,\ldots,J.
    \end{equation}

\item  For $\vf\in l_2(\gph)$ and $j=J_0,\ldots,J-1$,
\begin{align}
&   \vf = \sum_{\pV\in\VG_J} \ipG{\vf,\fra[J,\pV]}\fra[J,\pV]
+\sum_{n=1}^{r_J}\sum_{\pV\in\VG_J}\ipG{\vf,\frb[J,\pV]{n}}\frb[J,\pV]{n},
\label{thmeq:DFS:nrm1}\\
& \sum_{\pV\in\VG_{j+1}} \ipG{\vf,\fra[j+1,\pV]}\fra[j+1,\pV]
= \sum_{\pV\in\VG_{j}} \ipG{\vf,\fra}\fra+
\sum_{n=1}^{r_j}\sum_{\pV\in\VG_{j+1}} \ipG{\vf,\frb{n}}\frb{n}. \label{thmeq:DFS:2scale1}
\end{align}

\item For $\vf\in l_2(\gph)$ and $j=J_0,\ldots, J-1$,
\begin{align}
& \nrmG{\vf}^2 = \sum_{\pV\in\VG_J} \Bigl|\ipG{\vf,\fra[J,\pV]}\Bigr|^{2}+\sum_{n=1}^{r_J}\sum_{\pV\in\VG_J}\Bigl|\ipG{\vf,\frb[J,\pV]{n}}\Bigr|^{2}, \quad\label{thmeq:DFS:nrm2}\\
& \sum_{\pV\in\VG_{j+1}} \Bigl|\ipG{\vf,\fra[j+1,\pV]}\Bigr|^{2}
   = \sum_{\pV\in\VG_j} \Bigl|\ipG{\vf,\fra}\Bigr|^{2} + \sum_{n=1}^{r_j}\sum_{\pV\in\VG_{j+1}} \Bigl|\ipG{\vf,\frb{n}}\Bigr|^{2}.&\label{thmeq:DFS:2scale2}
\end{align}

  %
%

\item The framelet generators in $\Psi_j$ and the weights in $\QN[j]$ satisfy, for $1\leq \ell,\ell'\leq N$,
\begin{align}
  & 1 = \left|\FT{\scala_{j}}\left(\frac{\eigvm}{\Lambda_J}\right)\right|^{2} + \sum_{n=1}^{r_J}\left|\FT{\scalb_{j}^{(n)}}\left(\frac{\eigvm}{\Lambda_{J}}\right)\right|^{2},\label{thmeq:DFS:nrm:alpha:beta}\\[1mm]
            &
\overline{\FT{\scala_{j+1}}\left(\frac{\eigvm}{\Lambda_{j+1}}\right)}
{\FT{\scala_{j+1}}}\left(\frac{\eigvm[\ell']}{\Lambda_{j+1}}\right)\QU{\ell'}(\QN[{j+1}])
-\overline{\FT{\scala_{j}}\left(\frac{\eigvm}{\Lambda_{j}}\right)}
{\FT{\scala_{j}}}\left(\frac{\eigvm[\ell']}{\Lambda_{j}}\right)\QU{\ell'}(\QN[j])
\nonumber\\[1mm]
&\;=\sum_{n=1}^{r_j} \overline{\FT{\scalb^{(n)}_{j}}\left(\frac{\eigvm}{\Lambda_{j}}\right)}
{\FT{\scalb^{(n)}_{j}}}\left(\frac{\eigvm[\ell']}{\Lambda_{j}}\right)\QU{\ell'}(\QN[j+1])\quad \forall j=J_0,\ldots,J-1,\label{thmeq:DFS:2scale:alpha:beta}
 \end{align}
 where $\QU{\ell'}(\QN[j])$ is given by \eqref{eq:sum:wt:u}.

\item The identities in \eqref{thmeq:DFS:nrm:alpha:beta} hold, and for all $(\ell,\ell')\in \sigma_{\scala,\overline{\scala}}^{(j)}$ and $j=J_0+1,\ldots,J$,
\begin{align}
&\overline{\FT{\maska_{j}}\left(\frac{\eigvm}{\Lambda_{j}}\right)}{\FT{\maska_{j}}}\left(\frac{\eigvm[\ell']}{\Lambda_{j}}\right)
\QU{\ell'}(\QN[{j-1}])+
\sum_{n=1}^{r_{j-1}} \overline{\FT{\maskb_{j}}\left(\frac{\eigvm}{\Lambda_{j}}\right)}{\FT{\maskb_{j}}}\left(\frac{\eigvm[\ell']}{\Lambda_{j}}\right) \QU{\ell'}(\QN[{j}]) = \QU{\ell'}(\QN[{j}]),  \label{thmeq:DFS:mask}
 \end{align}
where
\begin{equation}\label{eq:sigma.scala}
 \sigma_{\scala,\overline{\scala}}^{(j)}:=\left\{(\ell,\ell')\in\N\times \N \setsep \overline{\FT{\scala}\left(\frac{\eigvm}{\Lambda_{j}}\right)}{\FT{\scala}}\left(\frac{\eigvm[\ell']}{\Lambda_{j}}\right)\neq0\right\}.
\end{equation}
\end{enumerate}
\end{theorem}
\begin{proof} The proofs of (i)$\Longleftrightarrow$(ii)$\Longleftrightarrow$(iii) are similar to those in Theorem~\ref{thm:UFS}. We next prove (iii)$\Longleftrightarrow$(iv)$\Longleftrightarrow$(v).

(iii) $\Longleftrightarrow$ (iv). For $\vf\in l_2(\gph)$, by the  the orthonormality of $\{\eigfm\}_{\ell=1}^\NV$, we have
\begin{equation}\label{eq:fr.coeff}
\begin{aligned}
    \ipG{\vf,\fra} &= \sqrt{\wV}\sum_{\ell=1}^{\NV} \conj{\FT{\scala_{j}}\left(\frac{\eigvm}{\Lambda_{j}}\right)}\Fcoem{\vf}\:\eigfm(\pV), \quad \pV\in\VG_j.\\
    \ipG{\vf,\frb{n}} &= \sqrt{\wV[j+1,{[\uG]}]}\sum_{\ell=1}^{\NV} \conj{\FT{\scalb^{(n)}_{j}}\left(\frac{\eigvm}{\Lambda_{j}}\right)}\Fcoem{\vf}\:\eigfm(\pV),\quad \pV\in\VG_{j+1}.
\end{aligned}
\end{equation}
Then,
\[
\begin{aligned}
\sum_{\pV\in\VG_j}  \Bigl|\ipG{\vf,\fra}\Bigr|^{2}
&=\sum_{\pV\in\VG_j} \wV\left|\sum_{\ell=1}^{\NV} \conj{\FT{\scala_{j}}\left(\frac{\eigvm}{\Lambda_{j}}\right)}\Fcoem{\vf}\:\eigfm(\pV)\right|^2 \\
&=\sum_{\ell=1}^\NV\sum_{\ell'=1}^\NV\Fcoem{\vf}\conj{\Fcoem[\ell']{\vf}}\;
  \conj{\FT{\scala_{j}}\left(\frac{\eigvm}{\Lambda_{j}}\right)}
  {\FT{\scala_{j}}}\left(\frac{\eigvm[\ell']}{\Lambda_{j}}\right)
  \sum_{\pV\in\VG_j}\wV\eigfm(\pV)\overline{\eigfm[\ell'](\pV)}\\
&=\sum_{\ell=1}^\NV\sum_{\ell'=1}^\NV
\Fcoem{\vf}\conj{\Fcoem[\ell']{\vf}}\;
\conj{\FT{\scala_{j}}\left(\frac{\eigvm}{\Lambda_{j}}\right)}
{\FT{\scala_{j}}}\left(\frac{\eigvm[\ell']}{\Lambda_{j}}\right)\QU{\ell'}(\QN[j])\\
\end{aligned}
\]
Since $\QU{\ell'}(\QN[J]) = \delta_{\ell,\ell'}$,
\[
\begin{aligned}
\sum_{\pV\in\VG_J} \Bigl|\ipG{\vf,\fra[J,\pV]}\Bigr|^{2}
           +\sum_{n=1}^{r_J}\sum_{\pV\in\VG_J}\Bigl|\ipG{\vf,\frb[J,\pV]{n}}\Bigr|^{2}
        = \sum_{\ell=1}^\NV |\vf_\ell|^2\left( \left|\FT{\scala_{j}}\left(\frac{\eigvm}{\Lambda_{J}}\right)\right|^{2}
        +\sum_{n=1}^{r_J}\left|\FT{\scalb^{(n)}_{j}}\left(\frac{\eigvm}{\Lambda_{J}}\right)\right|^{2}\right),
\end{aligned}
\]
which shows the equivalence between
 \eqref{thmeq:DFS:nrm2} and \eqref{thmeq:DFS:nrm:alpha:beta}. Moreover, \eqref{thmeq:DFS:2scale2} can be  rewritten as
\[
\begin{aligned}
&\sum_{\ell=1}^\NV\sum_{\ell'=1}^\NV \Fcoem{\vf}\overline{\Fcoem[\ell']{\vf}}\:
\overline{\FT{\scala_{j+1}}\left(\frac{\eigvm}{\Lambda_{j+1}}\right)}{\FT{\scala_{j+1}}}\left(\frac{\eigvm[\ell']}{\Lambda_{j+1}}\right)\QU{\ell'}(\QN[{j+1}])\\
=&\sum_{\ell=1}^{\NV}\sum_{\ell'=1}^\NV \Fcoem{\vf}\overline{\Fcoem[\ell']{\vf}}
\left[
\overline{\FT{\scala_{j}}\left(\frac{\eigvm}{\Lambda_{j}}\right)}
{\FT{\scala_{j}}}\left(\frac{\eigvm[\ell']}{\Lambda_{j}}\right)\QU{\ell'}(\QN[j])
+
\sum_{n=1}^{r_j}
\conj{\FT{\scalb^{(n)}_{j}}\left(\frac{\eigvm}{\Lambda_{j}}\right)}
{\FT{\scalb^{(n)}_{j}}}\left(\frac{\eigvm[\ell']}{\Lambda_{j}}\right)\QU{\ell'}(\QN[j+1])\right],
\end{aligned}
\]
which is equivalent to  \eqref{thmeq:DFS:2scale:alpha:beta}.

(iv) $\Longleftrightarrow$ (v). By  \eqref{eq:refinement:nonstationary}, we have
\[
\begin{aligned}
&\left[\overline{\FT{\scala_{j-1}}\left(\frac{\eigvm}{\Lambda_{j-1}}\right)}
{\FT{\scala_{j-1}}}\left(\frac{\eigvm[\ell']}{\Lambda_{j-1}}\right)\QU{\ell'}(\QN[{j-1}])
+
\sum_{n=1}^{r_{j-1}} \overline{\FT{\scalb^{(n)}_{j-1}}\left(\frac{\eigvm}{\Lambda_{j-1}}\right)}
{\FT{\scalb^{(n)}_{j-1}}}\left(\frac{\eigvm[\ell']}{\Lambda_{j-1}}\right)\QU{\ell'}(\QN[{j}])\right]
\\=&
\left[\overline{\FT{\maska_{j}}\left(\frac{\eigvm}{\Lambda_{j}}\right)}
{\FT{\maska_{j}}}\left(\frac{\eigvm[\ell']}{\Lambda_{j}}\right)\QU{\ell'}(\QN[{j-1}])
+\sum_{n=1}^{r_{j-1}} \overline{\FT{\maskb_{j}}\left(\frac{\eigvm}{\Lambda_{j}}\right)}
{\FT{\maskb_{j}}}\left(\frac{\eigvm[\ell']}{\Lambda_{j}}\right)\QU{\ell'}(\QN[{j}])\right]
\overline{\FT{\scala_{j}}\left(\frac{\eigvm}{\Lambda_{j}}\right)}
{\FT{\scala_{j}}}\left(\frac{\eigvm[\ell']}{\Lambda_{j}}\right),
\end{aligned}
\]
which implies \eqref{thmeq:DFS:2scale:alpha:beta}$\Longleftrightarrow$\eqref{thmeq:DFS:mask} and thus proves the equivalence  between (iv) and (v).
\end{proof}

The next corollary shows a condition on $\QU{\ell'} (\QN[j])$ such that the conditions (iv) and (v) of Theorem~\ref{thm:DFS} take a simplified form.

\begin{corollary}\label{cor:thm:DFS}
Retain all assumptions in Theorem~\ref{thm:DFS}, and in addition suppose that
\begin{equation}\label{tight:QN}
 \sigma_{\scala,\overline{\scala}}^{(j)}\subseteq \sigma_{\scala,\overline{\scala}}^{(j+1)}\quad
 \mbox{and}\quad\QU{\ell'} (\QN[j]) = \delta_{\ell,\ell'}\quad \forall (\ell,\ell')\in \sigma_{\scala,\overline{\scala}}^{(j)}, \;j=J_0,\ldots,J-1.
\end{equation}
Then, the following statements are equivalent.
\begin{enumerate}[{\rm(i)}]
\item
$\dfrsys(\{\Psi_j\}_{j=J_1}^J,\{\filtbk_j\}_{j=J_1+1}^{J})$  is a tight frame for $l_2(\gph)$ for all $J_1=J_0,\ldots,J$.

\item The identities in \eqref{thmeq:DFS:nrm:alpha:beta}, and for $j=J_0,\ldots,J-1$ and $\ell=1,\ldots, \NV$,
\begin{equation}
\label{thmeq:DFS:2scale:alpha:beta:simplified}
\Big|\FT{\scala_{j+1}}\left(\frac{\eigvm}{\Lambda_{j+1}}\right)\Big|^2
=\Big|\FT{\scala_{j}}\left(\frac{\eigvm}{\Lambda_{j}}\right)\Big|^2
+\sum_{n=1}^{r_j} \Big|\FT{\scalb^{(n)}_{j}}\left(\frac{\eigvm}{\Lambda_{j}}\right)\Big|^2.
\end{equation}

\item The identities in \eqref{thmeq:DFS:nrm:alpha:beta}, and for $j=J_0+1,\ldots, J$ and $\ell=1,\ldots,\NV$,
\begin{equation}\label{thmeq:DFS:mask:simplified}
\Big|\FT{\maska_{j}}\left(\frac{\eigvm}{\Lambda_{j}}\right)\Big|^2+
\sum_{n=1}^{r_{j-1}} \Big|\FT{\maskb_{j}}\left(\frac{\eigvm}{\Lambda_{j}}\right)\Big|^2=1.
\end{equation}
\end{enumerate}
\end{corollary}
\begin{proof}
If \eqref{tight:QN} holds, \eqref{thmeq:DFS:2scale:alpha:beta} and \eqref{thmeq:DFS:mask} are reduced to \eqref{thmeq:DFS:2scale:alpha:beta:simplified} and \eqref{thmeq:DFS:mask:simplified} respectively. The equivalence between (i), (ii) and (iii) follow from Theorem~\ref{thm:DFS}.
\end{proof}

\begin{remark} The characterization conditions in Theorems~\ref{thm:UFS} and \ref{thm:DFS} are simplified when $\fra[J+1,u]$ is the Dirac delta function $\delta_u$.
\end{remark}

Undecimated and decimated framelets are both defined on a graph $\gph$, but their construction utilises orthonormal eigenpairs for the $l_2$ spaces on the graph and the chain: $l(\gph)$ and $l(\gph_{J\to J_0})$. One can view an undecimated framelet system as a decimated system on a special coarse-grained chain, when each level is identical to $\gph$. Decimated framelets depend on the structure of the chain, that is, the connection between graphs at different levels of the chain. 
In the next section, we will show how to construct a coarse-grained chain of a graph, and study the impact of the chain structure on the generating set of the decimated framelet system.

\section{Construction of Decimated Tight Framelets on $\gph$}\label{sec:constr}
In the last section, we provide the characterization of undecimated and decimated framelet systems to be tight frames for $l_2(\gph)$. Based on the characterization, in this section, we detail the construction of tight decimated framelets on a given graph. We would first discuss the construction of coarse-grained chains of $\gph$ and  orthonormal eigen-pairs on $\gph$ with desired properties. Based on the coarse-grained chains and orthonormal eigen-pairs, together with a careful design of generating set $\Psi_j$ on $\R$, we then provide explicit construction of decimated tight framelets on $\gph$.

\subsection{Graph Clustering and Coarse-graining on $\gph$}
Graph partitioning and clustering are among central topics in the structured data analysis. The partitioned graph requires to contain various clusters for appropriate applications such as unsupervised or semi-supervised learning and data mining in various types of networks. One can use a top-down approach, which recursively splits a graph to subgraphs.  
This approach starts from the root (usually with one node) of a tree and proceeds to its branches (clusters) down to the bottom level.
In contrast, a bottom-up approach groups the leaves of a tree of the bottom level and then clusters the nodes level by level up to the coarsest. 
Another way of categorization for graph clustering algorithms takes account of whether the spectral method is used. For example, the clustering algorithms by using the graph Laplacian is a spectral method, see \cite{MeBeOs1992,BeFl2012,Garcia-Cardona_etal2014}. The clustering algorithm that utilises the weight function $\wG$ on a graph is a non-spectral method, see \cite{Coifman_etal2005,Dongen2000,GaCo2012,GaNaCo2010,ChDa2010,LaLe2006,BoVeZa2001,chen2014unsupervised}.

For a graph of triplet $\gph=(\VG,\EG,\wG)$, a clustering algorithm is usually based on a partition for $\gph$. Suppose $\VG_c:=\{[\uG]: \uG\in\VG\}$ the resulting partition with each $[\uG]\subseteq\VG$ a cluster on $\gph$. To obtain a coarse-grained chain of $\gph$, we can use the clustering algorithm hierarchically. In the first place, a connection relation between clusters needs to be defined on $\VG_c$. Following \cite{LaLe2006}, we define
\begin{equation}\label{defn:wG:coarse}
	\wG_c([\uG],[\vG]):=\sum_{\uG\in[\uG]}\sum_{\vG\in[\vG]}\frac{\wG(\uG,\vG)}{\vol(\gph)} \quad \mbox{for~} [\uG],[\vG]\in\VG_c.
\end{equation}
Then, $\wG_c$ becomes a weight function on $\VG_c\times \VG_c$. The weight $\wG_c$ is symmetric on $\VG_c$ and hence determines an (undirected) edge set $\EG_c$ by $\EG_c:=\{([\uG],[\vG]) \setsep \wG_c([\uG,\vG])>0\}$. We obtain $\gph_c:=(\VG_c,\EG_c,\wG_c)$, which is called a \emph{coarse-grained graph} of $\gph$. Using the clustering algorithm for $\gph_c$ then gives a coarse-grained graph $(\gph_c)_c$ of $\gph_c$. Recursively doing this step, we would obtain a \emph{coarse-grained chain} of the original graph $\gph$. We call the process of constructing a coarse-grained chain by a clustering algorithm the \emph{Coarse-Grained Chain (CGC)} algorithm.
The clustering algorithm that one uses determines the coarse-grained chain. Algorithm~\ref{algo:coarse-grain} shows an implementation of the CGC algorithm, which utilises a modified version of the non-spectral hierarchical clustering (NHC) algorithm given in \cite{ChMhZh2018}.

\IncMargin{1em}
\begin{algorithm}[th]
\SetKwData{step}{Step}
\SetKwInOut{Input}{Input}\SetKwInOut{Output}{Output}
\BlankLine
\Input{Graph $\gph=(\VG,\EG,\wG)$. The number of clusters on each level $(N_{J-1},\ldots,N_{J_0})$ with $1\leq N_{J_0}<N_{J_1}<\cdots<N_{J-1}<\NV$, where $\NV$ is the number of vertices in $\gph$.}

\Output{A coarse-grained chain $(\gph_J,\ldots,\gph_{J_0})$ of $\gph$ with $\gph_j=(\VG_j,\EG_j,\wG_j)$ such that $|\VG_j| = N_j$ for $j=J-1,\ldots,J_0$.}

Initialization: $\gph_J\leftarrow\gph$ and $j\leftarrow J$.\\

\While {$j>J_0$}{Compute graph distance $\distG_{\gph_j}$ of $\gph_j$.\\

Randomly choose $N_{j-1}$ vertices $u_1,\ldots,u_{k_{j-1}}$ from $V_j$ as centers. In the semi-supervised case, these centers are given in advance.\\

\While{true}{
Construct cluster $C_\ell$ for $\ell=1,\ldots,N_{j-1}$ such that $[\vG]\in V_j$ belongs to $C_\ell$ and the distance between $[\vG]$ and the cluster center $u_{\ell}$ is the minimal among all clusters:
 \[
 \ell = \mathrm{argmin}_{1\le \ell \le N_{j-1}}  \distG_{\gph_j}(u_{\ell},[v]).
 \]
 \\
 Update the centers: for each $C_\ell$, find a new center $u_{\ell} \in C_\ell$ such that $\sum_{v\in C_\ell} \distG_{\gph_j}(u_{\ell},v)$ is minimized.\\
 Break if all centers remain the same.
}
Construct a coarse-grained graph $\gph_{j-1}=(\VG_{j-1},\EG_{j-1},\wG_{j-1})$ of $\gph$, where
 \[
 \VG_j = \{\cup_{[v]\in C_\ell} [v]: \ell = 1,\ldots, N_{j-1}\},
 \]
 is the vertex set with $k_{j-1}$ vertices, and the weight function $\wG_{j-1}$ is defined by \eqref{defn:wG:coarse}.\\
$j\leftarrow j-1$.
}
\caption{Coarse-Grained Chain Algorithm (CGC)}
\label{algo:coarse-grain}
\end{algorithm}
\DecMargin{1em}\vspace{3mm}

\subsection{Orthonormal Bases for a Coarse-grained Chain via Graph Laplacian}\label{sec:orthonormal.basis.Laplacian}

Let $\gph_c=(\VG_c,\EG_c,\wG_c)$ be a coarse-grained graph of $\gph=(\VG,\EG,\wG)$, and $n_1:=|\VG|$, $n_0:=|\VG_c|$. Suppose $\bigl\{(\eigfm^{(0)},\eigvm^{(0)})\bigr\}_{\ell=1}^{n_0}$ and $\bigl\{(\eigfm^{(1)},\eigvm^{(1)})\bigr\}_{\ell=1}^{n_1}$ are orthonormal eigen-pairs for $l_2(\gph_c)$ and $l_2(\gph)$ respectively. The orthonormal eigen-pair $\bigl\{(\eigfm^{(0)},\eigvm^{(0)})\bigr\}_{\ell=1}^{n_0}$ can be extended to an orthonormal system on $\gph$ using the following Gram-Schmidt process.

Define vectors $\eigfm, \ell=1,\ldots,n_0$ on $\gph$ by
\[
\eigfm(\vG):=\frac{\eigfm^{(0)}([\vG])}{\sqrt{\#[\vG]}},\quad \vG\in \VG,\; \ell=1,\ldots, n_0.
\]
Then $\{(\eigfm,\eigvm)\}_{\ell=1}^{n_0}$ are orthonormal in $l_2(\gph)$.
Since $\{(\eigfm^{(1)},\eigvm^{(1)})\}_{\ell=1}^{n_1}$ is a set of orthonormal eigen-pairs for $l_2(\gph)$, there must exist a subsequence $\eigfm[\ell_{n_0+1}]^{(1)},\ldots,\eigfm[\ell_{n_1}]^{(1)}$ of $\eigfm[1]^{(1)},\ldots,\eigfm[n_1]^{(1)}$ such that
\begin{equation}\label{defn:lin.indep:G:Gc}
\{
\eigfm\setsep \ell=1,\ldots,n_0
\}
\cup
\{
\eigfm[\ell_k]^{(1)}\setsep k=n_0+1,\ldots,n_1
\}
\end{equation}
is linearly independent in $l_2(\gph)$. Applying the Gram-Schmidt process to the system of linearly independent vectors in \eqref{defn:lin.indep:G:Gc}, we would obtain the following orthonormal basis for $l_2(\gph)$:
\begin{equation}\label{defn:ONB:G:Gc}
	\{\eigfm\setsep \ell=1,\ldots,n_0\} \cup \{\eigfm \setsep \ell=n_0+1,\ldots,n_1\}.
\end{equation}

Now letting $\eigvm=\eigvm^{(1)}$, $\ell=1,\ldots,n_1$, we then have a set of orthonormal eigen-pairs $\{(\eigfm,\eigvm)\}_{\ell=1}^{n_1}$ for $l_2(\gph)$ satisfying $\eigfm(p) \equiv const$ for $p\in [\vG]\in \VG_c$ and $\ell=1,\ldots, n_0$. Here, the first $n_0$ vectors can be regarded as eigenvectors on the coarse-grained graph $\gph_c$.

Let $(\gph_J,\ldots,\gph_{J_0})$ be a coarse-grained chain associated with a sequence $\{(\eigfm^{(j)},\eigvm^{(j)}) \setsep \ell=1,\ldots,|\VG_j|=:N_j\}_{j = J_0}^{J}$ of orthonormal eigen-pairs of $\gph_{j}$, which can be constructed from positive semi-definite operators on the chain, such as graph Laplacians and diffusion operators. For example, if we use the graph Laplacian $\gL_j$ on $\gph_j$ in \eqref{defn:gL}, the sequence $\{\gL_{j}\setsep j=J_0,\ldots, J\}$ for a sequence of orthonormal eigen-pairs is given by
\[
\gL_j \eigfm^{(j)} = \eigvm^{(j)} \eigfm^{(j)},\quad \ell=1,\ldots, N_j, \; j=J_0,\ldots, J.
\]
Recursively using the Gram-Schmidt process for orthonormal eigen-pairs 
\[
\left(\left\{(\eigfm^{(j)},\eigvm^{(j)})\right\}_{\ell=1}^{N_j},
\left\{(\eigfm^{(j+1)},\eigvm^{(j+1)})\right\}_{\ell=1}^{N_{j+1}}\right)
\]
for $j=J_0,\ldots,J-1$, we obtain orthonormal eigen-pairs
$\{(\eigfm,\eigvm)\}_{\ell=1}^{\NV}$ for $l_2(\gph)$, which satisfy for each $j=J_0,\ldots,J$,
\begin{equation}\label{eq:ONB:const}
\eigfm(\vG) \equiv const\quad \forall \vG\in[\vG]_{\gph_j}\mbox{ and } \ell\le N_j.
\end{equation}
We group the $\{\eigfm\}_{\ell=1}^\NV$ as $\cup_{j=J_0}^J\{\eigfm \setsep \ell=N_{j-1}+1,\ldots, N_j\}$ where we have let $N_{J_0-1}:=0$. Specifically, the $j$th group $\{\eigfm \setsep \ell=N_{j-1}+1,\ldots, N_j\}$ is an orthogonal basis on the graph $\gph_j$ for $j=J_0,\ldots,J$.
We call $\{(\eigfm,\eigvm)\}_{\ell=1}^{\NV}$ \emph{global orthonormal eigen-pair} for the coarse-grained chain $\gph_{J\to J_0}$.
We present the detailed implementation of CGC eigen-pairs in Algorithm~\ref{alg:ONB:const}, which we call \emph{orthonormalization on the coarse-grained chain (ONBC)} algorithm.

\medskip
\IncMargin{1em}
\begin{algorithm}[th]
\SetKwData{step}{Step}
\SetKwInOut{Input}{Input}\SetKwInOut{Output}{Output}
\BlankLine
\Input{A coarse-grained chain $(\gph_J,\ldots,\gph_{J_0})$ of $\gph$ with $\gph_j=(\VG_j,\EG_j,\wG_j)$ associating with a sequence $\{(\eigfm^{(j)},\eigvm^{(j)}) \setsep \ell=1,\ldots,|\VG_j|\}_{j = J_0}^{J}$ of orthonormal eigen-pairs. }

\Output{An orthonormal basis $\{(\eigfm,\eigvm)\}_{\ell=1}^{|\VG|}$  for $l_2(\gph)$ satisfying \eqref{eq:ONB:const}.}

Initialization: $j\leftarrow J_0$.\\

\While {$j<J$}{
$n_0\leftarrow |\VG_j|, n_1\leftarrow |\VG_{j+1}|$.\\

{\bf Extension}: Define $\eigfm\in l_2(\gph_{j+1})$, $\ell=1,\ldots,n_0$ by
\begin{equation}\label{eq:u:extension}
\eigfm([\vG]):=\frac{\eigfm^{(j)}([\vG]_{\gph_j})}{\sqrt{N_{[\vG]_{\gph_j}}}},\quad [\vG]\in \VG_{j+1}, \;\ell=1,\ldots, n_0,
\end{equation}
where $N_{[\vG]_{\gph_j}} := \#\{[\uG]_{\gph_{j+1}}: \uG\in[\vG]_{\gph_j}\}$ is the number of vertices in $\gph_{j+1}$ which are also in the cluster $[\vG]_{\gph_j}$.

{\bf Gram-Schmidt process}: Choose $\eigfm[{\ell_k}]^{(j+1)}, k=n_0+1,\ldots,n_1$ from $\eigfm^{(j+1)}: \ell=1,\ldots, n_1$ so that
\[
\{\eigfm\setsep \ell=1,\ldots,n_0\}
\cup
\{\eigfm[\ell_k]^{(j+1)}\setsep k=n_0+1,\ldots,n_1\}
\]
is linearly independent for $l_2(\gph_{j+1})$. Apply the Gram-Schmidt process to the above vectors to obtain the full orthonormal basis for $l_2(\gph_{j+1})$,
\[
\{\eigfm\setsep \ell=1,\ldots,n_1\}.
\]

{\bf Update}:  Replace $\{(\eigfm^{(j+1)},\eigvm^{(j+1)})\}_{\ell=1}^{n_1}$ by $\{(\eigfm,\eigvm^{(j+1)})\}_{\ell=1}^{n_1}$.
\\ \hspace{1.6cm} $j\leftarrow j+1$.
}
\caption{Chain-based Graph Laplacian Basis}
\label{alg:ONB:const}
\end{algorithm}
\DecMargin{1em}\vspace{3mm}

\subsection{Orthonormal Bases on $\gph$ with Low Spoc: Haar Basis on Graphs}
\label{sec:haar.orth.basis}
We define a notion to measure the sparsity and locality for a vector on the graph.
\begin{definition}[SPOC]
	For a vector $\eigfm$ on $\gph$, we define the spoc as the number of distinct non-zero entries in the vector $\eigfm$, denoted by
\begin{equation}\label{defn:spoc}
	\spoc(\eigfm):=\#\{\eigfm(\vG)\neq0: \vG\in\VG\}.
\end{equation}
\end{definition}
Recall that the support of a vector $\eigfm$ is given by
\begin{equation}\label{defn:supp}
\supp(\eigfm):=\{\vG\in\VG: \eigfm(\vG)\neq 0\}.
\end{equation}
The spoc is different from support. For example, a vector $\eigfm\equiv const$ has spoc $1$ yet it has full support.
The computational complexity for the transforms for a basis on the graph is usually dependent on the spoc of the basis but not on its support.

For Algorithm~\ref{alg:ONB:const}, the computational complexity might go above linear, as the spoc of the global orthonormal basis by ONBC is not always bounded. Here, we construct a global orthonormal basis for the coarse-grained chain of a graph $\gph$ based on the interval decomposition method \citep{ChFiMh2015}. The resulting basis is Haar-like and satisfies condition~\eqref{eq:ONB:const}, and its spoc is at most $2$.

We first give the construction of a basis for a chain with two levels. To construct a basis on the chain with more levels, we can use this method recursively.
Let $\gph_c=(\VG_c,\EG_c,\wG_c)$ be a coarse-grained graph of $\gph=(\VG,\EG,\wG)$, and $n_1:=|\VG|$ and $n_0:=|\VG_c|$.
To construct an orthogonal basis on $\gph_c$, we sequence the vertices of $\gph_c$ by their degrees as
\[
\VG_c:=\{[\uG_j]_{\gph_c}\setsep j=1,\ldots,n_0\} \mbox{ and } \dG([\uG_j]_{\gph_c})\ge \dG([\uG_{j+1}]_{\gph_c}).
\]
For each vertex $[\uG_j]_{\gph_c}$, we associate with it a characteristic function $\chi_{[\uG_j]}$ on $\gph_c$:
\[
\chi_j^c([\vG]) :=
\begin{cases}
1, & [\vG] = [\uG_j]_{\gph_c},\\
0, & \mbox{otherwise.}
\end{cases}
\]
The system $\{\chi_j^c \setsep j=1,\ldots, n_0\}$ is then an orthogonal system on $\gph_c$. We need to include the constant vector $\eigfm[1]^c =\frac{1}{\sqrt{n_0}}{\bf 1}$ in the system. To achieve that, we replace by $\eigfm[1]^c= \frac{1}{\sqrt{n_0}}{\bf 1}$ the characteristic function $\chi_{[\uG_{n_0}]_{\gph_c}}$ on the set $\{\chi_{[\uG_j]_{\gph_c}} \setsep j=1,\ldots, n_0\}$ and make the set $\{\eigfm[1], \chi_1^c,\ldots, \chi_{n_0-1}^c\}$ orthonormal by the Gram-Schmidt process. It then gives a new orthonormal system, as follows.
\begin{equation}\label{defn:haar-orth-gph-c}
\begin{aligned}
\eigfm[1]^c & = \frac{1}{\sqrt{n_0}}{\bf 1}, \\
\eigfm[\ell]^c &= \sqrt{\frac{n_0-\ell+1}{n_0-\ell+2}}\left(\chi^c_{\ell-1}-\frac{1}{{n_0-\ell+1}}\sum_{j=\ell}^{n_0}\chi^c_j\right),\quad \ell=2,\ldots,n_0.
\end{aligned}
\end{equation}

\begin{proposition}\label{prop:haar-orth-gph-c}
The system $\{\eigfm[\ell]^c \setsep \ell=1,\ldots, n_0\}$ given in \eqref{defn:haar-orth-gph-c} is an orthonormal basis for $\gph_c$ with $\spoc(\eigfm^c)\le 2$ for all $\ell=1,\ldots, n_0$.
\end{proposition}
\begin{proof}
By \eqref{defn:haar-orth-gph-c}, it follows that $\spoc(\eigfm^c)\le 2$ and  $\|\eigfm[\ell]^c\| = 1$ for all $\ell=1,\ldots, n_0$. Next, we prove $\eigfm^c$, $\ell=1,\dots,n_0$, are orthogonal.
For $\ell=2,\ldots, n_0$,
\[
\begin{aligned}
\ipG{\eigfm[1]^c, \eigfm[\ell]^c}
=\sqrt{\frac{n_0-\ell+1}{n_0(n_0-\ell+2)}}\ipG{{\bf1}, \chi^c_{\ell-1}-\frac{1}{n_0-\ell+1}\sum_{j=\ell}^{n_0}\chi^c_j} = 0.
\end{aligned}
\]
For $2\le k <\ell \le n_0$,
\[
\begin{aligned}
\ipG{\eigfm[\ell]^c, \eigfm[k]^c}
 &=c_{k,\ell}\ipG{\chi^c_{\ell-1}-\frac{1}{n_0-\ell+1}\sum_{j=\ell}^{n_0}\chi^c_j,
\chi^c_{k-1}-\frac{1}{n_0-k+1}\sum_{j=k}^{n_0}\chi^c_j} \\
&=c_{k,\ell}\left(
-\frac{\sum_{j=k}^{n_0}\ipG{\chi^c_{\ell-1},\chi^c_j}}{n_0-k+1}
+\frac{\ipG{\sum_{j=\ell}^{n_0}\chi^c_j,\sum_{j=k}^{n_0}\chi^c_j}}{(n_{0}-k+1)(n_{0}-\ell+1)}\right)\\
&=c_{k,\ell}\left(
-\frac{1}{n_0-k+1}
+\frac{n_0-\ell+1}{(n_{0}-k+1)(n_{0}-\ell+1)}\right) = 0,\\
\end{aligned}
\]
where $c_{k,\ell} = \sqrt{\frac{n_0-\ell+1}{n_0-\ell+2}}\times  \sqrt{\frac{n_0-k+1}{n_0-k+2}}$.
Thus, $\{\eigfm^c \setsep \ell=1,\ldots, n_0\}$ is an orthonormal basis for $\gph_c$.
\end{proof}

We associate each element of $\{\eigfm^c \setsep \ell=1,\ldots,n_0\}$ with the vertex $[\uG_\ell]_{\gph_c}$ in $\gph_c$, and define a sequence of vectors by extending $\eigfm^c$, $\ell=1,\ldots,n_0$ onto $\gph$
\begin{equation}\label{defn:Haar-orth-gph-1}
\eigfm[\ell,1](\vG):=\frac{\eigfm^c([\vG])}{\sqrt{\#[\vG]}},\quad \vG\in \VG,\; \ell=1,\ldots, n_0,
\end{equation}
which satisfy
\begin{equation}\label{eq:haar-const-ext}
\eigfm[\ell,1](\vG) \equiv \mathrm{const} \quad \forall v\in [\uG_j]_{\gph_c}  \mbox{ and } j=1,\ldots,n_0.
\end{equation}
The resulting $\eigfm$, $\ell=1,\dots,n_0$, are orthogonal for $l_2(\gph)$.

For a cluster $[\uG_\ell]_{\gph_c}$, let $k_\ell:=\#[\uG_\ell]_{\gph_c}$. We sequence the vertices in $[\uG_\ell]_{\gph_c}$ by their degrees to obtain
\[
[\uG_\ell]_{\gph_c} = \{v_{\ell,1},\ldots,v_{\ell,k_\ell}\}\mbox{ and } \dG(\vG_{\ell,j})\ge \dG(\vG_{\ell,j+1}).
\]
For each vertex in the cluster $[\uG_\ell]_{\gph_c}$, we define a characteristic function $\chi_{\ell,j}$ on $\gph$ by
\[
\chi_{\ell,j}(\vG) :=
\begin{cases}
1 & \vG = \vG_{\ell,j},\\
0 & \mbox{otherwise.}
\end{cases}
\]
Using the Gram-Schmidt process for the system $\{\eigfm[\ell,1],\chi_{\ell,1},\ldots,\chi_{\ell,k_\ell-1}\}$, we then obtain
\begin{equation}\label{defn:Haar-orth-gph-2}
\eigfm[\ell,k] =
 \sqrt{\frac{k_\ell-k+1}{k_\ell-k+2}}\left(\chi_{\ell,k-1}-\frac{1}{k_\ell-k+1}\sum_{j=k}^{k_\ell}\chi_{\ell,j}\right),\quad k=2,\ldots,k_\ell.
\end{equation}

\begin{proposition}\label{prop:haar:sparse}
Let
\begin{equation}\label{eq:haar-orth-gph}
\{\eigfm[\ell,k]\setsep  k = 1,\ldots,k_\ell, \ell=1,\ldots,n_0\}
\end{equation}
be the set of vectors in \eqref{defn:Haar-orth-gph-1} and \eqref{defn:Haar-orth-gph-2}. Then it is an orthonormal basis on $\gph$ with $\spoc(\eigfm[\ell,k])\le 2$ for all $k=1,\ldots,k_\ell$ and $\ell=1,\ldots,n_0$.
\end{proposition}
\begin{proof}
Fixed a level $\ell$, in a similar way to the proof in Proposition~\ref{prop:haar-orth-gph-c}, one can show that $\{\eigfm[\ell,k] \setsep k = 1,\ldots, k_\ell\}$ is an orthonormal system for $l_2(\gph)$. Now for two different $\ell,\ell'$, by definition, $\supp(\eigfm[\ell,k])=[\uG_\ell]_{\gph_c}$ and $\supp(\eigfm[\ell',k']) = [\uG_{\ell'}]_{\gph_c}$, and hence $\supp(\eigfm[\ell,k])\cap \supp(\eigfm[\ell',k'])=\emptyset$ for all $k=2,\ldots,k_{\ell}$ and $k'=2,\ldots, k_{\ell'}$. Moreover, by \eqref{eq:haar-const-ext}, $\eigfm[\ell,1](\vG)$ is constant for all $\vG\in[\uG_{\ell'}]_{\gph_c}$, and $\eigfm[\ell',1](\vG')$ is constant for all $\vG'\in[\uG_{\ell}]_{\gph_c}$. Thus, $\ipG{\eigfm[\ell,k],\eigfm[\ell',k']}=\delta_{\ell,\ell'}\delta_{k,k'}$, which shows that \eqref{eq:haar-orth-gph} is an orthonormal basis for $l_2(\gph)$.
\end{proof}

To construct a global orthonormal basis for the coarse-grained chain $(\gph_J,\ldots,\gph_{J_0})$, we continue the above construction for $\gph_c$. Starting from $\gph_{J_0}$, we obtain an orthonormal basis on $\gph_{J_0}$ as $\{\eigfm^{(J_0)} \setsep \ell = 1,\ldots, N_{J_0}\}$. By the chain relation of $\gph_{J_0}$ and $\gph_{J_1}$, we can extend this basis to an orthonormal basis $\{\eigfm^{(J_1)} \setsep \ell = 1,\ldots, n_{J_1}\}$ on $\gph_{J_1}$. Continuing carrying on this process, we then obtain an orthonormal basis for each $\gph_j$,
\[
\{\eigfm^{(j)} \setsep \ell = 1,\ldots, N_j\}
\]
with $\spoc(\eigfm^{(j)})\le 2$ for all $\ell$ and $j=J_0,\ldots, J$.

In particular, when $j=J$, $\{\eigfm:=\eigfm^{(J)} \setsep \ell = 1,\ldots, N_J=:N\}$ becomes an orthonormal basis for $l_2(\gph)$, which satisfies $\spoc(\eigfm)\le 2$ for all $\ell=1,\ldots,N$. Moreover, we can group $\{\eigfm\}_{\ell=1}^\NV$ by their associated levels in the chain as $\cup_{j=J_0}^J\{\eigfm \setsep \ell=N_{j-1}+1,\ldots, N_j\}$. Here, we have let $N_{J_0-1}:=0$, and the $j$th group $\{\eigfm \setsep \ell=N_{j-1}+1,\ldots, N_j\}$ is the extension with respect to  the orthogonal basis on the graph $\gph_j$ for $j=J_0,\ldots, J$. By this rearrangement, $\{\eigfm\}_{\ell=1}^N$ then satisfies the condition in \eqref{eq:ONB:const}. We call the resulting orthonormal basis $\{\eigfm\}_{\ell=1}^N$ \emph{Haar global orthonormal basis} for the chain.

We show the detailed steps of constructing the global orthonormal basis in Algorithm~\ref{alg:HONB:const}, and call the algorithm \emph{Haar orthonormalization on the coarse-grained chain}, or HONBC.

\subsection{Construction of Decimated Tight Framelets on $\gph$}
\label{sec:constr.dtf}
In this section, we construct decimated tight framelets using band-limited filters on $\gph$. To this end, we introduce some auxiliary functions for the construction of generating sets $\Psi_j$ in Theorem~\ref{thm:DFS} and then derive a filter bank $\filtbk_j$. Together with the global orthonormal basis satisfying \eqref{eq:ONB:const}, we can define framelets $\fra$ and $\frb{n}$ on $\gph$.

Here we define the filter bank based on polynomial splines in \cite{Han1997,Han2013,HaZh2010,HaZhZh2016}. Let $P_m(x)$ be a polynomial given by
\[
P_m(x):=\left(\frac{1+x}{2}\right)^m\sum_{k=0}^{m-1}{m-1+k\choose k}\left(\frac{1-x}{2}\right)^k,
\]
which satisfies $P_m(x)+P_m(-x)=1$ for any $m\in\N$ \citep{Daubechies1992}. We define a bump function $\nu_{[c_L,c_R];\veps_L,\veps_R}(\xi)$ for $c_L<c_R$ and positive numbers $\veps_L, \veps_R$ satisfying $\veps_L+\veps_R\le c_R-c_L$ by
\begin{equation}\label{eq:bumpnu}
\nu_{[c_L,c_R];\veps_L,\veps_R}(\xi):=
\begin{cases}
0, & \xi\le c_L-\veps_L \mbox{ or } \xi\ge c_R+\veps_R,\\
\sin(\frac{\pi}{2}P_m(\frac{\xi-c_L+\veps_L}{2\veps_L})), & c_L-\veps_L <  \xi < c_L+\veps_L,\\
1, &  c_L+\veps_L \le \xi\le c_R-\veps_R,\\
\cos(\frac{\pi}{2}P_m(\frac{\xi-c_R+\veps_R}{2\veps_R})), &  c_R-\veps_R< \xi\le c_R+\veps_R.
\end{cases}
\end{equation}

\IncMargin{1em}
\begin{algorithm}[H]
\label{algo:haar_basis}
\SetKwData{step}{Step}
\SetKwInOut{Input}{Input}\SetKwInOut{Output}{Output}
\BlankLine
\Input{A coarse-grained chain $(\gph_J,\ldots,\gph_{J_0})$ of $\gph$ with $\gph_j=(\VG_j,\EG_j,\wG_j)$.}

\Output{Haar-like orthonormal basis $\{\eigfm^{(j)}\setsep \ell=1,\ldots,N_j\}$ for $l_2(\gph_j)$ with low spoc, $j=J_0,\ldots,J$.}

Initialization: $j\leftarrow J_0-1$. $\VG_{j} \leftarrow [\uG_1]_{\gph_j}:=[V]$. $\eigfm[1]^{(j)} \equiv 1$. ($\eigfm[1]^{(j)}$ is associated with the vertex $[\uG_1]_{\gph_j}$.) \\

\While {$j<J$}{
$n_0\leftarrow |\VG_j|,\; n_1\leftarrow |\VG_{j+1}|$.
(Note that $\eigfm^{(j)}$ is associated with $
[\uG_\ell]_{\gph_j} = \{[\uG]_{\gph_{j+1}}\setsep \uG\in[\uG_\ell]_{\gph_j}\}$.)\\

{\bf Extension}: Define $\eigfm[\ell,1]\in l_2(\gph_{j+1})$, $\ell=1,\ldots,n_0$ by
\begin{equation}\label{eq:haar:u:extension}
\eigfm[\ell,1]([\vG]):=\frac{\eigfm^{(j)}([\vG]_{\gph_j})}{\sqrt{N_{[\vG]_{\gph_j}}}},\quad [\vG]\in \VG_{j+1}, \ell=1,\ldots, n_0,
\end{equation}
where $N_{[\vG]_{\gph_j}} := \#\{[\uG]_{\gph_{j+1}}:  \uG\in[\vG]_{\gph_j}\}$ is the number of vertices in $\gph_{j+1}$ which lie in the cluster $[\vG]_{\gph_j}$.

{\bf Gram-Schmidt process}:

\For{$\ell=1,\ldots,n_0$}
{
Sort the vertices in $[\uG_\ell]_{\gph_j}$ by their degrees as
$[\uG_\ell]_{\gph_j}=\{[\vG_{\ell,1}]_{\gph_{j+1}},\ldots,[\vG_{\ell,k_\ell}]_{\gph_{j+1}}\}$ and define $\chi_{\ell,k}$ as characteristic function with respect to the vertex $[\vG_{\ell,k}]_{\gph_{j+1}}$,
where $k_\ell$ is the number of children in $\gph_{j+1}$ of $[p_{\ell}]_{\gph_j}$. Compute
\[
\eigfm[\ell,k] =
 \sqrt{\frac{k_\ell-k+1}{k_\ell-k+2}}\left(\chi_{\ell,k-1}-\frac{1}{k_\ell-k+1}\sum_{j=k}^{k_\ell}\chi_{\ell,j}\right),\quad k=2,\ldots,k_\ell,
\]
and associate it with the node $[\vG_{\ell,k}]_{\gph_{j+1}}$.
}

{\bf Update}:  Rearrange $\{(\eigfm[\ell,k], [\vG_{\ell,k}]_{\gph_{j+1}}) \setsep k=1,\ldots,k_\ell, \ell=1,\ldots, n_0\}$ in the following order
 \[
 \bigl\{\eigfm[\ell,1] \setsep \ell\in\{1,\ldots,n_0\}\bigr\}\cup   \bigl\{\eigfm[\ell,2] \setsep \ell\in\{1,\ldots,n_0\}\bigr\}\cup\cdots\cup \bigl\{\eigfm[\ell,k_{max}] \setsep \ell\in\{1,\ldots,n_0\}\bigr\}
 \]
 to form an orthonormal basis $\{(\eigfm^{(j+1)},[\vG_\ell]_{\gph_{j+1}}) \setsep \ell=1,\ldots, n_1\}$ for $l_2(\gph_{j+1})$ such that the first $n_0$ vectors are orthogonal for $l_2(\gph_{j})$.

 $j\leftarrow j+1$ .
}
\caption{Haar-Like Orthonormal Basis on the Coarse-Grained Chain (HONBC)}
\label{alg:HONB:const}
\end{algorithm}
\DecMargin{1em}\vspace{3mm}

The scaling functions and filters for the decimated framelets are defined by $\nu_{[c_L,c_R];\veps_L,\veps_R}$ and the coarse-grained chain of the graph.
Let $0<\veps<1$ and define $\scala_\veps$ by
\[
\FT{\scala_\veps}(\xi):=\nu_{[-\frac{1+\veps}{2},\frac{1+\veps}{2}];\frac{1-\veps}{2},\frac{1-\veps}{2}}(\xi), \quad \xi\in\R.
\]
Its support is $\supp\: \FT{\scala_\veps} \subseteq [-1,1]$, and $\FT{\scala_\veps}(\xi) \equiv 1$ for all $\xi\in[-\veps,\veps]$.
This is our low-pass scaling functions. 

We now construct the high-pass scaling functions by using the low-pass $\FT{\scala_\veps}$.
Let $J,J_0$, $J\ge J_0$ be two fixed integers corresponding to the finest and coarsest scales.
We choose $J-J_0+1$ \emph{scaling factors} $\Lambda_j, j=J_0, \dots,J$ such that
\begin{equation*}
    \begin{array}{ll}
    \Lambda_j>1, & j=J_0,\dots,J,\\
    \Lambda_{j+1}>\Lambda_j/\varepsilon, & j=J_0,\dots,J-1.
    \end{array}
\end{equation*}
For each level $j=J_0,\dots,J$, let
$r_j$ be a positive integer, we then construct $r_j$ high-passes functions $\scalb_{j}^{(n)}\in L_2(\R)$, $n=1,\ldots, r_j$, such that
\begin{equation}\label{eq:decomp.scala.b}
	|\FT{\scala_\veps}(\xi/\Lambda_{j+1})|^2-|\FT{\scala_\veps}(\xi/\Lambda_{j})|^2 =
	\sum_{n=1}^{r_j}\Big|\FT{\scalb_{j}^{(n)}}(\xi/\Lambda_{j})\Big|^2.
\end{equation}
To construct the above function $\scalb_{j}^{(n)}$ for scale $j$, we choose positive numbers $\veps_1,\ldots,\veps_{r_j}$ and $c_1,\ldots c_{r_j}$ such that
\[
    c_1<\cdots<c_{r_j-1}<1<c_{r_j}, \; c_{r_j}=1+\veps_{r_j} \mbox{ and } \veps_{n-1}+\veps_{n}\le c_{n}-c_{n-1}, \; n=2,\ldots,r_j.
\]
We need to define partition functions $\gamma^{(1)},\ldots,\gamma^{(r_j)}$ by
\[
\gamma^{(1)}:=\nu_{[-c_1,c_1],\varepsilon_1,\varepsilon_1} \mbox{ and }
\gamma^{(n)}:=\nu_{[c_{n-1},c_{n}],\varepsilon_{n-1},\varepsilon_{n}},\; n=2,\ldots,r_j.
\]
It can be verified that the $\gamma^{(n)}$ satisfies the partition of unity:
\[
\sum_{n=1}^{r_j}|\gamma^{(n)}(\xi)|^2 = 1 \;\;\forall \xi\in[0,1].
\]
We then define the \emph{high-pass scaling functions} for scale $j$ by, for $n=1,\ldots,r_j$, 
\[
\FT{\scalb_{j}^{(n)}}(\xi/\Lambda_{j}):=
\sqrt{|\FT{\scala_\veps}(\xi/\Lambda_{j+1})|^2-|\FT{\scala_\veps}(\xi/\Lambda_{j})|^2}\cdot\gamma^{(n)}(\xi/\Lambda_{j+1}),
\]
which satisfy \eqref{eq:decomp.scala.b}. 
The associated \emph{framelet generating set} is then
\begin{equation}
\label{defn:Psi_j:split}
\Psi_j:=\{\scala_\veps; \scalb_j^{(1)},\ldots, \scalb_j^{(r_j)}\}, \quad j=J_0,\ldots, J.
\end{equation}
From this, we define the filter bank, as follows.
\begin{definition}[Filter bank for chain]\label{def:filter bank}
    For $j=J_0+1,\ldots, J$, the filter bank for the chain $\gph_{J\to J_0}$ of graph $\gph$ is
\begin{equation}\label{defn:filtbk_j:split}
\filtbk_j:=\{\maska_{j};\maskb[1]_j,\ldots,\maskb[r_{j-1}]_j\}
\end{equation}
with
\begin{equation}\label{defn:filtbk_j:split2}
\FT{\maska_{j}}(\xi/\Lambda_{j})
:=\frac{\FT{\scala_\veps}(\xi/\Lambda_{j-1})}{\FT{\scala_\veps}(\xi/\Lambda_{j})},\quad
\FT{\maskb_{j}}(\xi/\Lambda_{j})
:=\frac{\FT{\scalb_{j}^{(n)}}(\xi/\Lambda_{j-1})}{\FT{\scala_\veps}(\xi/\Lambda_{j})},\quad \xi\in \R, \; n=1,\ldots, r_{j-1}.
\end{equation}
\end{definition}
Then, \eqref{eq:refinement:nonstationary} holds with $\maska_{j}$ and $\maskb_{j}$ in \eqref{defn:filtbk_j:split2}.
\begin{remark}
    Different from the undecimated case, the scaling factor $\Lambda_j$ should be adaptive to the chain structure. In practice, we can simply choose $\Lambda_j = N_j$.
\end{remark}

\subsection*{Example of filter bank}
We show an example of filter bank with 1, 2 or 3 high passes. The filters in the filter bank can be obtained by the function in \eqref{eq:bumpnu}, using the parameters in the following table. Here, besides the parameters $c_R, c_L, \varepsilon_R, \varepsilon_L$ used in $\nu$ in \eqref{eq:bumpnu}, there are extra free parameters $\zeta_c^a, \zeta_c^{b^{(1)}}, \zeta_c^{b^{(2)}}$ which can adjust the intersection points of the low-pass and high-pass filter curves. The $\zeta_c^a, \zeta_c^{b^{(1)}}, \zeta_c^{b^{(2)}}$ are needed to lie in $(0,0.5)$. For simplicity in the case with 3 high passes, we define $N^{b^{(1)}}$ and $N^{b^{(2)}}$ by
\begin{equation*}
    N^{b^{(1)}} = N_{j-1} + 0.3(N_j-N_{j-1}),\quad 
    N^{b^{(2)}} = N_{j-1} + 0.8(N_j-N_{j-1}).
\end{equation*}
\begin{equation*}
\begin{array}{lcccccc}
\toprule
\multicolumn{3}{c}{\textbf{Filter bank with 1 high pass}} & &\multicolumn{3}{c}{\textbf{Filter bank with 2 high passes}}\\[1mm] 
     & a & b & & a & b^{(1)} & b^{(2)}\\ \cmidrule{1-3} \cmidrule{5-7} 
    c_{R} & 0.5(1+N_{j-1})\zeta_c^a & 2N_j & & 0.5(1+N_{j-1})\zeta_c^a & 2(N_j+N_{j-1})\zeta_c^{b} & 2N_j\\
    \varepsilon_{R} & N_{j-1} - c_{R}^a & N_j/4 & & N_{j} - c_R^{a} & N_j-c_R^{b} & 1\\
    c_{L} & -c_{R}^a & c_R^a & & -c_{R}^a & c_R^{a} & c_R^{b^{(1)}}\\
    \varepsilon_{L} & \varepsilon_{R}^a & \varepsilon_R^a & & \varepsilon_{R}^a & \varepsilon_R^a & \varepsilon_R^{b^{(1)}}\\
    \bottomrule
\end{array}
\end{equation*}
\begin{equation*}
\begin{array}{lcccc}
\multicolumn{5}{c}{\textbf{Filter bank with 3 high passes}}\\[1mm]
     & c_{R} & \varepsilon_{R}  & c_{L} & \varepsilon_{L}\\ \midrule 
    a & 0.5N_{j-1}(1+\zeta_c^a) & N_{j-1} - c_{R}^a & -c_{R}^a & \varepsilon_{R}^a\\[1mm]
    b^{(1)} & 0.5(N^{b^{(1)}}+N_{j-1})+0.5 \zeta_{c}^{b^{(1)}}(N^{b^{(1)}}-N_{j-1}) & N^{b^{(1)}}-c_R^{b^{(1)}} & c_R^a & \varepsilon_R^a\\[1mm]
    b^{(2)} & 0.5(N^{b^{(2)}}+N^{b^{(1)}})+0.5 \zeta_{c}^{b^{(2)}}(N^{b^{(2)}}-N^{b^{(1)}}) & N^{b^{(2)}}-c_R^{b^{(2)}} & c_R^{b^{(1)}} & \varepsilon_R^{b^{(1)}}\\[1mm]
    b^{(3)} & 2N_j & 1 & c_R^{b^{(2)}} & \varepsilon_R^{b^{(2)}}\\
    \bottomrule\\
\end{array}
\end{equation*}

Figure~\ref{fig:filter bank 3high} shows a filter bank with three high passes based on the above construction for the chain $\gph_{4\to0}$ with 500, 250, 100, 40 and 8 nodes for levels from the finest to coarsest. The supports of low-pass and high-pass filters are determined by the chain.
For the $j$th level, the low-pass filter is supported on $[0,N_j]$; the high passes are supported on subintervals of $(0,N_j)$, where the lower end of the support interval is greater than 0. The intersection points between the low pass and high passes are tunable parameters determined by the parameters $\zeta_c^a$, $\zeta_c^{b^{(1)}}$ and $\zeta_c^{b^{(2)}}$.

\begin{figure*}[h]
    \centering
    \scriptsize
\begin{tabular}{cccc}
  \includegraphics[width=.225\textwidth]{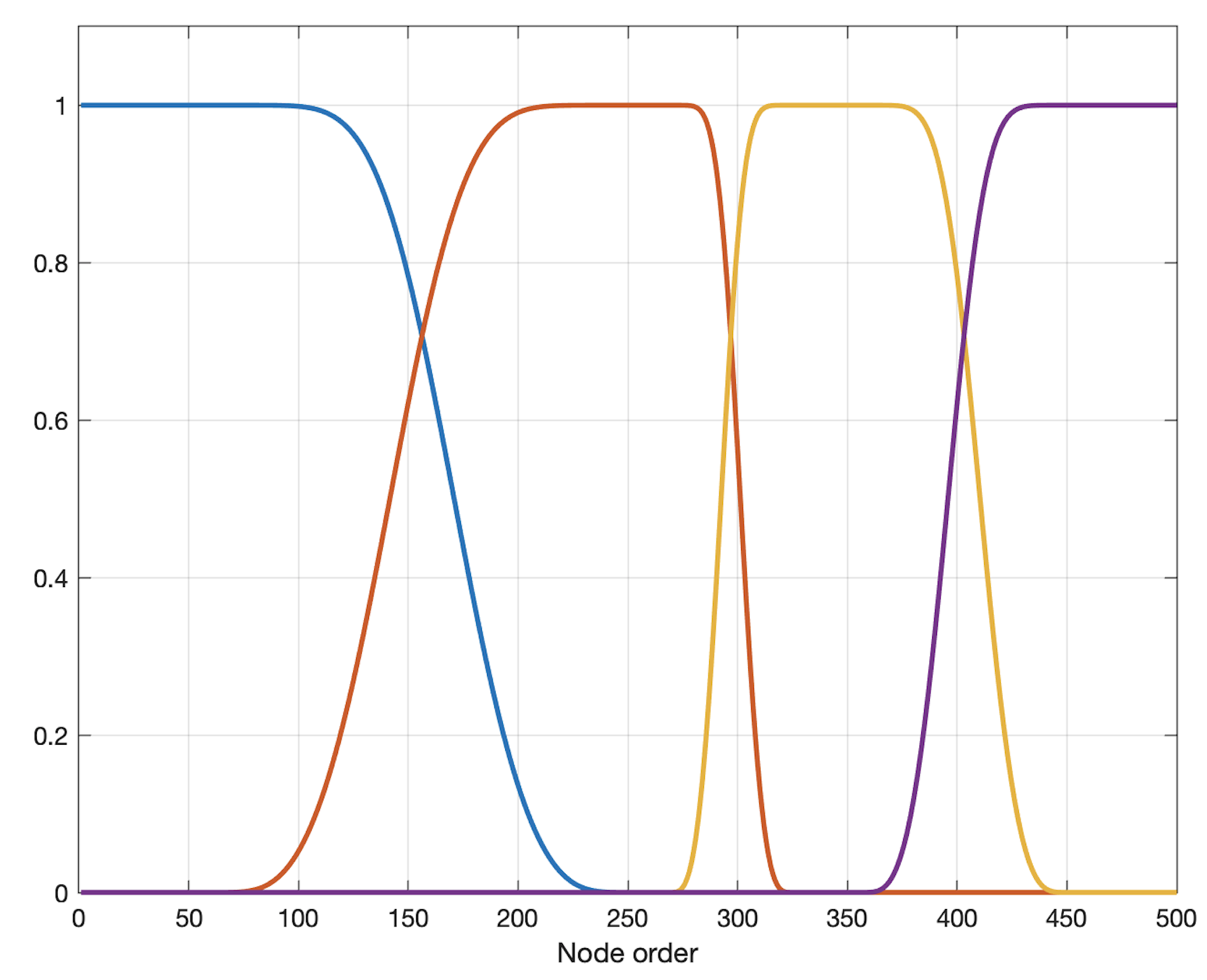}
  &\includegraphics[width=.225\textwidth]{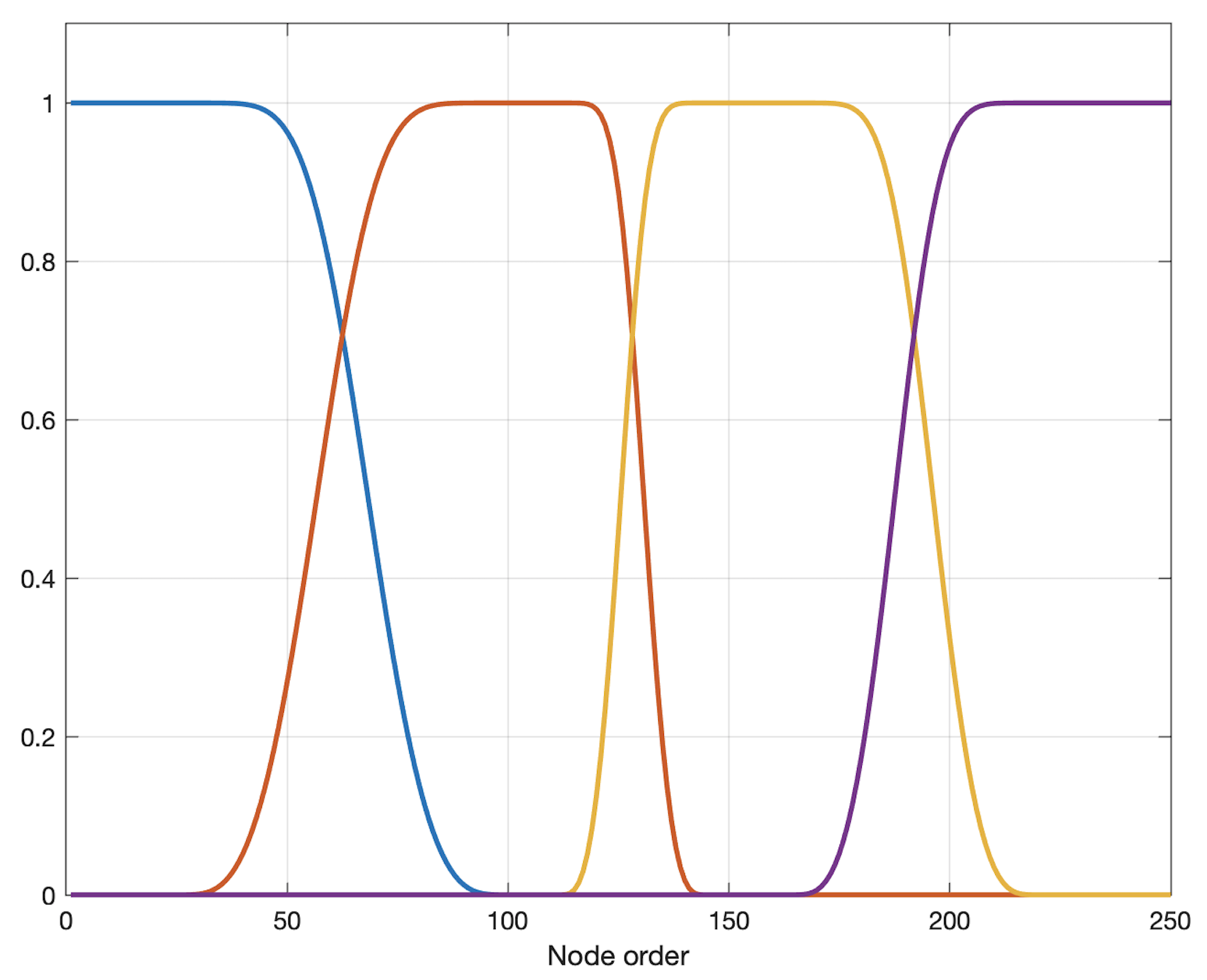}
  &\includegraphics[width=.225\textwidth]{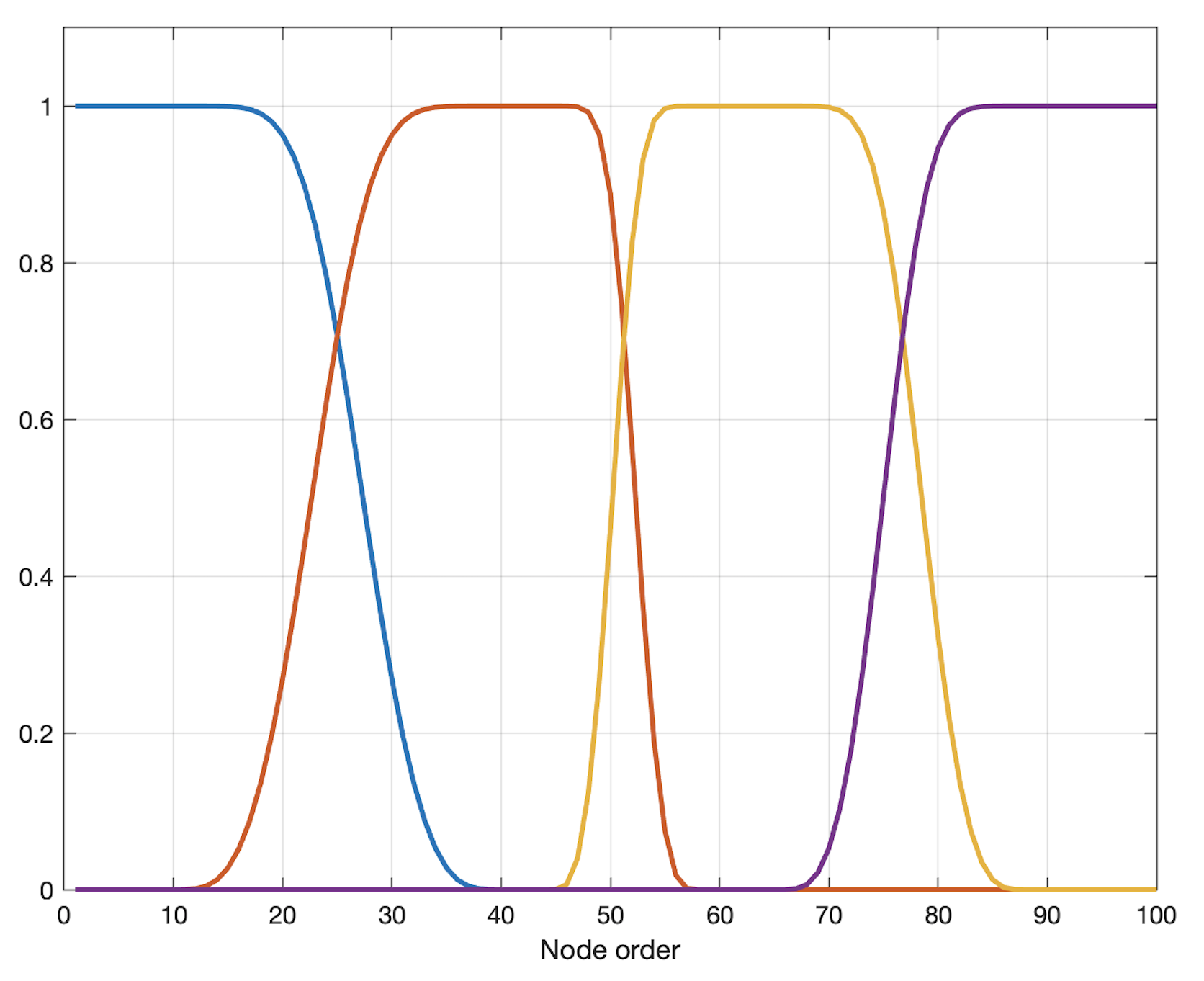}
  &\includegraphics[width=.225\textwidth]{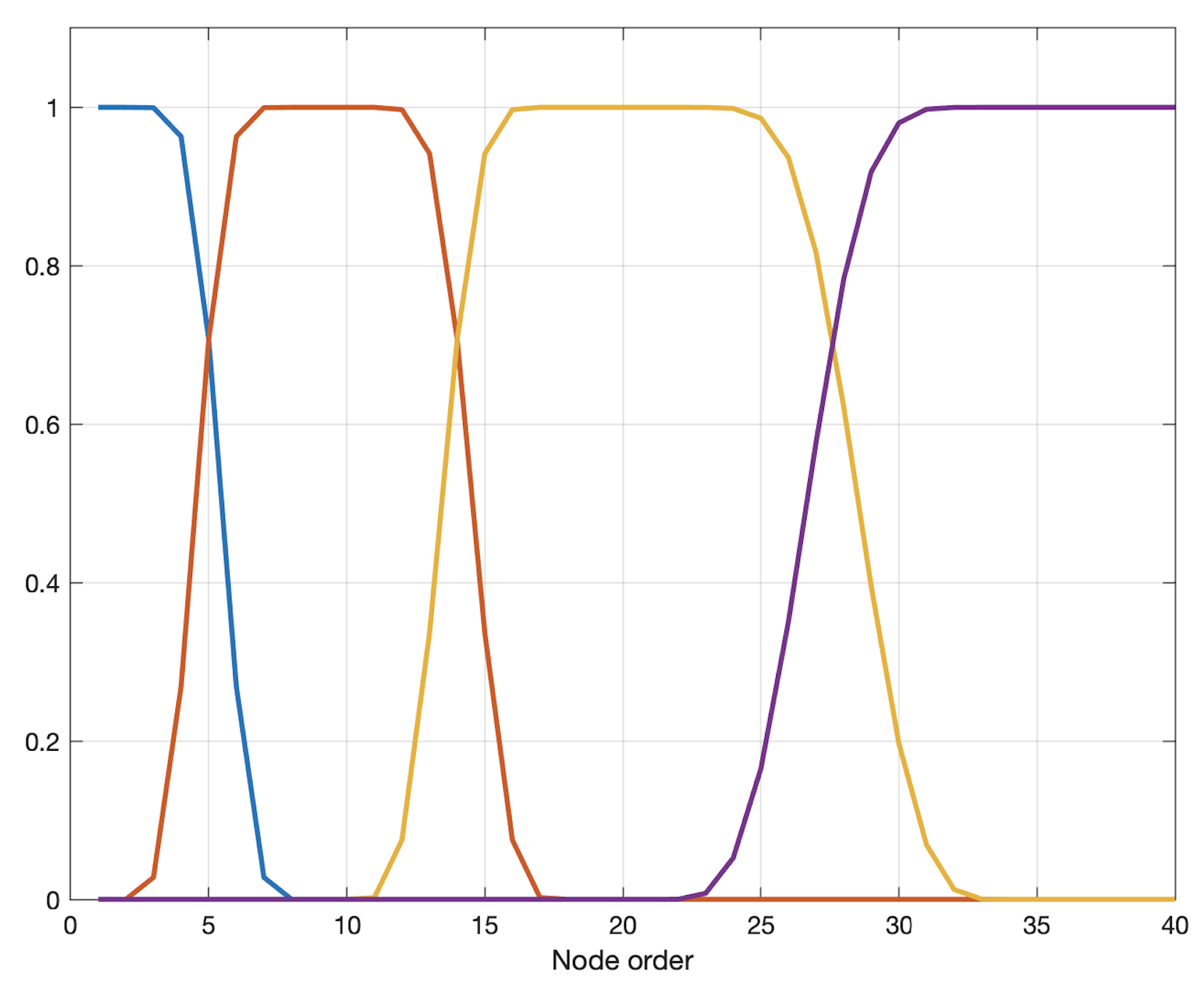}
\end{tabular}
    \caption{Filter banks with three high passes. The chain has 5 levels and the graphs of the chain have 500, 250, 100, 40 and 8 nodes from the finest to coarsest. Note that we do not need a filter bank for the coarsest level.
    From left to right are the pictures of the filters for the chain level from the finest to the second last coarsest. The top and bottom use the intersection point parameters $\zeta_c^a=0.25$, $\zeta_c^{b^{(1)}}=0.25$ and $\zeta_c^{b^{(2)}}=0.25$. The support of low-pass and high-passes are exactly controlled by the graph size at each level of the chain.}\label{fig:filter bank 3high}
\end{figure*}

With the filter bank in Definition~\ref{def:filter bank}, we can construct a tight decimated framelet system $\dfrsys\bigl(\{\Psi_j\}_{j=J_1}^J,\{\filtbk_j\}_{j=J_1+1}^{J}\bigr)$, $J_1=J_0,\ldots,J$, for $l_2(\gph_{J\to J_0})$ with the framelets given by, for $j=J_0,\ldots,J$,
\begin{equation}
\label{defn:fra.frb2}
\begin{aligned}
    \fra(\vG) &:= \sqrt{\wV}\sum_{\ell=1}^{\NV} \FT{\scala_\veps}\left(\frac{\eigvm}{\Lambda_{j}}\right)\conj{\eigfm(\pV)}\eigfm(\vG),  \quad [\uG]\in\VG_j,\\
    \frb{n}(\vG) &:= \sqrt{\wV[j+1,{[\uG]}]}\sum_{\ell=1}^{\NV} \FT{\scalb_{j}^{(n)}}\left(\frac{\eigvm}{\Lambda_j}\right)\conj{\eigfm(\pV)}\eigfm(\vG),\quad \pV\in\VG_{j+1}, \; n =1,\ldots,r_j,
\end{aligned}
\end{equation}
where $\wV[j,{[\uG]}]:={\#[\uG]_{\gph_j}}$.
\begin{theorem}
\label{thm:DFS:construction}
Let $\{(\eigfm,\eigvm)\}_{\ell=1}^N$ be a global orthonormal eigen-pair for a coarse-grained chain $\gph_{J\rightarrow J_0}$ of $\gph$. Let
\[
\dfrsys(\{\Psi_j\}_{j=J_1}^J,\{\filtbk_j\}_{j=J_1+1}^{J}), \quad J_1=J_0,\ldots, J
\]
be a sequence of the decimated framelet systems associated with the generating sets and filter banks in \eqref{defn:Psi_j:split} and \eqref{defn:filtbk_j:split} and with $\fra$ and $\frb{n}$ in \eqref{defn:fra.frb2}. Then, the decimated framelet system $\dfrsys(\{\Psi_j\}_{j=J_1}^J,\{\filtbk_j\}_{j=J_1+1}^{J})$ is a tight frame for $l_2(\gph)$ for each $J_1=J_0,\ldots,J$.
\end{theorem}

\begin{proof}
For the framelets $\fra$ and $\frb{n}$ given in \eqref{defn:fra.frb2}, the weighted sum of the product of two eigenvectors by the quadrature weights $\QN[j]$ is
\begin{equation}\label{tight:QN:2}
\begin{aligned}
\QU{\ell'}(\QN[j])
=&\sum_{\pV\in\VG_j}\wV\:\eigfm(\pV)\overline{\eigfm[\ell'](\pV)}
\\=&\sum_{\pV\in\VG}\sum_{\vG\in[\uG]_{\gph_j}}\eigfm(\vG)\overline{\eigfm[\ell'](\vG)}
\\=&\sum_{\vG\in\VG}\eigfm(\vG)\overline{\eigfm[\ell'](\vG)} =\delta_{\ell,\ell'}
\end{aligned}
\end{equation}
for all pairs in $\{(\ell,\ell')\setsep 1\le \ell,\ell'\le N_j\}= \sigma_{\scala,\overline{\scala}}^{(j)}$, where the second equality has used $\wV[j,{[\uG]}]:={\#[\uG]_{\gph_j}}$.
By our construction follows $\sigma_{\scala,\overline{\scala}}^{(j)}\subseteq\sigma_{\scala,\overline{\scala}}^{(j+1)}$. Then,
the tightness of $\dfrsys(\{\Psi_j\}_{j=J_1}^J,\{\filtbk_j\}_{j=J_1+1}^{J})$ for each $J_1=J_0,\ldots,J$ follows by Corollary~\ref{cor:thm:DFS}.
\end{proof}
\section{Fast Decimated $\gph$-Framelet Transforms}\label{sec:fmtG}

In this section, we describe the multi-level decimated $\gph$-framelet transforms for a coarse-grained chain of the graph $\gph$ based on the decimated tight framelet system constructed in Section~\ref{sec:dfs}. We would introduce the discrete Fourier transforms for a global orthonormal basis based on a decimated coarse-grained chain, and their fast algorithm. With decimated filter banks, it then gives fast $\gph$-framelet transforms.

\subsection{Discrete Fourier Transforms on $\gph$}
Let $\{\eigfm\}_{\ell=1}^{\NV}$ be a global orthonormal basis for a coarse-grained chain $\gph_{J\rightarrow J_0}$ of a graph $\gph$. The \emph{discrete Fourier transform} (DFT) of a vector $\vf\in l_2(\gph)$ is to compute the Fourier coefficient vector $\FT{\vf}:=(\FT{\vf_\ell})_{\ell=1}^{\NV}$:
\begin{equation}\label{defn:fft:adj:G}
\FT{\vf_\ell}:=\ipG{\vf,\eigfm}=\sum_{\vG\in\VG}\vf(\vG)\overline{\eigfm(\vG)},\quad \ell=1,\ldots,\NV.
\end{equation}
The \emph{adjoint discrete Fourier transform} (ADFT) of a coefficient vector $\vc:=(c_\ell)_{\ell=1}^\NV$ with respect to $\{\eigfm\}_{\ell=1}^{\NV}$ is to compute the vector $\vf\in l_2(\gph)$:
\begin{equation}\label{defn:fft:G}
\vf(\vG) := \sum_{\ell=1}^\NV c_\ell\: \eigfm(\vG),\quad \vG\in\VG.
\end{equation}
\paragraph{Computational cost for DFT and ADFT}
We denote $\flop^*(\mathrm{DFT}_\gph)$ and $\flop^*(\mathrm{ADFT}_\gph)$ the minimal total number of summations and multiplications for the DFT and ADFT in \eqref{defn:fft:adj:G} and \eqref{defn:fft:G} for a given orthonormal basis over all possible algorithmic realizations.
Direct evaluation of \eqref{defn:fft:adj:G} or \eqref{defn:fft:G} requires total number of $\NV^2$ summations and $\NV^2$ multiplications. Thus, $\flop^*(\mathrm{DFT}_\gph)= \bigO(N^2)$ and $\flop^*(\mathrm{ADFT}_\gph)= \bigO(N^2)$. The property \eqref{eq:ONB:const} suggests that for specific bases, the computational complexity for DFT's can be significantly reduced, as we discuss now.

Let $\{\eigfm\}_{\ell=1}^{|\VG|}$ be a global orthonormal basis for a coarse-grained chain $\gph_{J\rightarrow J_0}$ of the graph $\gph$. Let $N_j:=|\VG_j|$ for $j=J_0,\ldots,J$ and $N_{J_0-1}:=0$, and the indicator function $\ve_{[\vG]}:\VG\rightarrow\R$ by $\ve_{[\vG]}(\uG)=1$ for all $\uG\in[\vG]$ and $0$ for $\uG\in V\backslash[\vG]$, and $\ve_\gph:=\Id_\VG$ the identity function on $\gph$. One can then show that $\ve_\gph:= \sum_{[\vG]\in\VG_j}\ve_{[\vG]} \equiv 1$ for any coarse-grained graph $\gph_j$ of $\gph$.

We first consider the computational steps of the discrete Fourier transform in \eqref{defn:fft:adj:G}. Suppose $N_{j-1}<\ell\le N_j$ for some $J_0\le j\le J$. Then, by \eqref{eq:ONB:const}, 
\[
\begin{aligned}
\sum_{\vG\in\VG} \vf(\vG)\eigfm(\vG)
&=\sum_{[\vG]\in\VG_j}\sum_{\uG\in[\vG]}\vf(\uG)\eigfm(\uG)\\
&=\sum_{[\vG]\in\VG_j}\eigfm([\vG])\sum_{\uG\in[\vG]}\vf(\uG)\\
&=\sum_{[\vG]\in\VG_j}\eigfm([\vG])s(\vf,[\vG])\\
&=\sum_{[\vG]\in\VG_j}t_\ell(\vf,[\vG]),
\end{aligned}
\]
where we have let
\begin{equation*}
    s(\vf,[\vG]):=\sum_{\vG\in[\vG]}\vf(\vG),\quad
      t_\ell(\vf,[\vG]):=\eigfm([\vG])s(\vf,[\vG]).
\end{equation*} 
We can evaluate $\FT{\vf}_{\ell}$, $\ell=1,\ldots,\NV$ in \eqref{defn:fft:adj:G} in two steps, for which we state a pseudocode in Algorithm~\ref{alg:DFT:G}.
\begin{enumerate}[\rm(1)]
\item Evaluate for $j=J_0,\ldots,J-1$,
\begin{equation}\label{defn:sumsG}
	s(\vf,[\vG]) \quad \forall [\vG] \in \VG_j.
\end{equation}
With the hierarchical structure of the chain, we can evaluate the $s(\vf,[\vG]_{\gph_j})$ by using the $s(\vf)$ over nodes at the finer $(j+1)$th level, that is,
\[
s(\vf,[\vG]_{\gph_j}) = \sum_{[\uG]_{\gph_{j+1}}\subseteq [\vG]_{\gph_j}}s(\vf,[\uG]_{\gph_{j+1}}).
\]
Here the summation over all the children of $[v]_{\gph_j}$.
Consequently, the total number of summations required to compute $\{s(\vf,[\vG])\setsep [\vG] \in \VG_j,\; j=J_0,\ldots,J-1\}$ is no more than $\sum_{j=J_0}^{J-1} N_{j+1}$.

\item For each $\ell$, let $j$ be the integer such that $N_{j-1}<\ell\le N_j$. The evaluation of\\ $\sum_{[\vG]\in\VG_j}\eigfm([\vG])s(\vf,[\vG])$ includes the following steps.
\begin{enumerate}
\item Compute the product $t_\ell(\vf,[\vG])=\eigfm([\vG])s(\vf,[\vG])$ for all $[\vG]\in\VG_j$,  which requires no more than $\spoc(\eigfm)$ multiplications.

\item Evaluate the sum $\sum_{[\vG]\in\VG_j}t_\ell(\vf,[\vG])$, which needs no more than $\spoc(\eigfm)$ summations.
\end{enumerate}
The total number of multiplication and summation operations is at most $\sum_\ell\spoc(\eigfm)$.
\end{enumerate}

\IncMargin{1em}
\begin{algorithm}[th]
\SetKwData{step}{Step}
\SetKwInOut{Input}{Input}\SetKwInOut{Output}{Output}
\BlankLine
\Input{A global orthonormal basis $\{(\eigfm,\eigvm)\}_{\ell=1}^{|\VG|}$ on $\gph_{J\rightarrow J_0}$ satisfying \eqref{eq:ONB:const}, and a vector $\vf$ on $\gph$.}

\Output{Fourier coefficients $\FT{\vf}$ given in \eqref{defn:fft:adj:G}.}

Initialization: $N_j\leftarrow|\VG|$. $s(\vf,[\vG]_{\gph_{J+1}})\leftarrow \vf(\vG),\; \vG\in\VG$.

\For{$j$=$J$ \KwTo $J_0$}
{

$N_j\leftarrow |\VG_{j-1}|$.\\
 $s(\vf,[\vG]_{\gph_j})\leftarrow \sum_{[\uG]_{\gph_{j+1}}\subseteq [\vG]_{\gph_j}}s(\vf,[\uG]_{\gph_{j+1}}),\; [\vG]_{\gph_j}\in \VG_{j}$.

\For{$\ell$=$N_{j-1}+1$ \KwTo $N_j$}
{
$\FT{f}_\ell\leftarrow \sum_{[\vG]\in\VG_j}\eigfm([\vG])s(\vf,[\vG])$, $[\vG]\in V_j$.
}
}

\caption{Discrete Fourier Transform on $\gph$ (DFT)}
\label{alg:DFT:G}
\end{algorithm}
\DecMargin{1em}\vspace{3mm}

For the computation of the ADFT in \eqref{defn:fft:G}, by using the clustering feature of the chain, we can rewrite
\[
\begin{aligned}
\vf = \sum_{\ell} c_\ell \eigfm
&=\sum_{j=J_0}^J \sum_{\ell=N_{j-1}+1}^{N_j} c_\ell \eigfm
=\sum_{j=J_0}^J \sum_{\ell=N_{j-1}+1}^{N_j} \sum_{[\vG]\in\VG_j}c_\ell \eigfm([\vG])\ve_{[\vG]}\\
&=\sum_{j=J_0}^J \sum_{\ell=N_{j-1}+1}^{N_j} \sum_{[\vG]\in\VG_j}t_\ell([\vG],\vc)\ve_{[\vG]}\\
&=\sum_{j=J_0}^J\sum_{[\vG]\in\VG_j}\left(\sum_{\ell=N_{j-1}+1}^{N_j}t_\ell([\vG],\vc)\right)\ve_{[\vG]}\\
&=\sum_{j=J_0}^J\sum_{[\vG]\in\VG_j}s([\vG],\vc)\ve_{[\vG]},
\end{aligned}
\]
where we have let 
\begin{equation*}
	t_\ell([\vG],\vc):=c_\ell\eigfm([\vG]),\quad s([\vG]_{\gph_j},\vc):=\sum_{\ell=N_{j-1}+1}^{N_j}t_\ell(\vc,[\vG]_{\gph_j}).
\end{equation*}
It suggests that the evaluation of $\vf(\vG)$ for $\vG\in\VG$ in \eqref{defn:fft:G} can be split into the following three steps, for which we show a pseudocode in Algorithm~\ref{alg:ADFT:G}.
\begin{enumerate}[\rm(1)]
\item For each $j$, let $\ell$ be the integer such that $N_{j-1}<\ell\le N_j$. Compute the product
$$t_\ell([\vG],\vc)=c_\ell\eigfm([\vG])$$
for all $[\vG]\in\VG_j$, which requires multiplications no more than $\spoc(\eigfm)$ steps. The total number of multiplications is thus no greater than $\sum_\ell\spoc(\eigfm)$.
\item For each $j=J_0,\ldots,J$, evaluate the sums
    \begin{equation*}
    s([\vG],\vc)=\sum_{\ell=N_{j-1}+1}^{N_j}t_\ell(\vc,[\vG]),\quad [\vG]\in\VG_j.
    \end{equation*}
Again, the total number of summations is no more than $\sum_\ell \spoc(\eigfm)$.
\item Evaluate $\vf:=\sum_{j=J_0}^{J_1}\sum_{[\vG]\in\VG_j}s([\vG],\vc)\ve_{[\vG]}$. We exploit the hierarchical chain structure to compute it efficiently. Let
\begin{equation*}
    \vf_{J_1}:=\sum_{j=J_0}^{J_1}\sum_{[\vG]\in\VG_j}s([\vG],\vc)\ve_{[\vG]},\quad J_1=J_0,\ldots,J,
\end{equation*}
which can be regarded as a function on $\gph_{J_1}$. We observe that
\begin{equation*}
\vf_{J_1+1}=\vf_{J_1}+\sum_{[\vG]\in\VG_{J_1+1}}s([\vG],\vc)\ve_{[\vG]},
\end{equation*}
for which the total number of summations is no more than $\sum_{j=J_0}^{J-1}N_{j+1}$.
\end{enumerate}

\IncMargin{1em}
\begin{algorithm}[th]
\SetKwData{step}{Step}
\SetKwInOut{Input}{Input}\SetKwInOut{Output}{Output}
\BlankLine
\Input{A global orthonormal basis $\{(\eigfm,\eigvm)\}_{\ell=1}^{|\VG|}$ on $\gph_{J\rightarrow J_0}$ satisfying \eqref{eq:ONB:const} and $\vc=(c_\ell)_{\ell=1}^\NV$.}

\Output{$\vf$ given in \eqref{defn:fft:G}. }

Initialization: $N_{J_0-1}\leftarrow0$. $\vf_{J_0-1}([\vG])\leftarrow 0,\; [\vG]\in \gph_{J_0}$. \\

\For{$j$=$J_0$ \KwTo $J$}
{
$N_j\leftarrow |\VG_j|$.\\
 $s([\vG],\vc)\leftarrow 0$ for all $[\vG]\in\VG_j$.\\

\For{$\ell$=$N_{j-1}+1$ \KwTo $N_j$}
{
$s([\vG],\vc)\leftarrow s([\vG],\vc)+c_\ell\eigfm([\vG]),\; [\vG]\in\VG_j$.
}
$\vf_{j}([\vG])\leftarrow \vf_{j-1}([\vG]) + s([\vG],\vc),\; [\vG]\in \VG_{j}$.
}
$\vf \leftarrow\vf_J$
\caption{Adjoint Discrete Fourier Transform on $\gph$ (ADFT) }
\label{alg:ADFT:G}
\end{algorithm}
\DecMargin{1em}\vspace{3mm}

\medskip

From the above analysis, we obtain the upper bounds for the computational complexity of the discrete Fourier transform and its adjoint for a coarse-grained chain.
\begin{theorem}\label{thm:flop}
Let $\{\eigfm\}_{\ell=1}^{\NV}$ be a global orthonormal basis for the coarse-grained chain $\gph_{J\rightarrow J_0}$ of a graph $\gph$. Then,
\[
\flop^*(\mathrm{DFT}_\gph)= \bigO\left(\sum_{\ell=1}^{\NV} \spoc(\eigfm)\right)\quad\mbox{and}\quad
\flop^*(\mathrm{ADFT}_\gph)= \bigO\left(\sum_{\ell=1}^{\NV} \spoc(\eigfm)\right),
\]
where the constants in the big $\mathcal{O}$ is independent of the $N$ and $\spoc(\eigfm)$ for any $\ell=1,\dots,N$.
\end{theorem}
Thus, when the $\spoc(\eigfm)=\bigo{}{1}$. The DFT and ADFT have linear computational complexity.
\subsection{Decimated $\gph$-Framelet Transforms}
\label{sec:decomp.reconstr}
In this section, we study the multi-level decimated framelet transforms on $\gph$, which include the \emph{framelet decomposition} and \emph{reconstruction} algorithms. For a signal $\vf$ on a graph $\gph$ and a sequence of decimated tight framelets $\dfrsys\bigl(\{\Psi_j\}_{j=J_1}^J,\{\filtbk_j\}_{j=J_1+1}^{J}; \gph_{J\rightarrow J_0}, \QN[J\rightarrow J_0]\bigr)$, $J_1=J_0,\ldots, J$ in \eqref{defn:DFS}, the \emph{framelet decomposition} algorithm produces a sequence of the vectors as the \emph{framelet approximation} and \emph{detail coefficients}
\begin{equation}
\label{defn:fmtG:coefs}
\{\fracoev[J_0]\}\cup\{\frbcoev{n}: n=1,\ldots, r_j,\; j=J_0,\ldots,J-1\}
\end{equation}
where for level $j=J_0,\ldots,J$, $\fracoev$ is the vector of the approximation framelet coefficients for $\gph_j$, and $\frbcoev{n}$, $n=1,\ldots,r_j$, is the vector of the detail framelet coefficients for $\gph_{j+1}$:
\begin{equation}\label{defn:fmtG:coefs2}
\begin{aligned}
\fracoev({[\uG]}) := &\ipG{\vf,\fra},\quad [\uG]\in\VG_j, \\
\frbcoev{n}([\uG]) := &\ipG{\vf,\frb{n}},\quad [\uG]\in\VG_{j+1},\;  n=1,\dots,r_j.
\end{aligned}
\end{equation}
The \emph{framelet reconstruction} algorithm is to reconstruct $\vf$ with the framelet coefficients in \eqref{defn:fmtG:coefs}.
We give a constructive implementation for the framelet decomposition and reconstruction by using generalized Fourier transforms on the graph, as follows.

Let $\{(\eigfm,\eigvm)\}_{\ell=1}^N$ be a global orthonormal basis for a coarse-grained chain $\gph_{J\rightarrow J_0}$ of the graph $\gph$. For $j=J_0,\ldots,J$, let $\QN[j]:=\{\wV: \pV\in\VG_j\}$ be the set of weights on $\gph_j$ and $\QN[J\rightarrow J_0]:=(\QN[J],\ldots,\QN[J_0])$ for the coarse-grained chain $\gph_{J\rightarrow J_0}$ which satisfies \eqref{tight:QN}. For a finite index set $\Omega$, we denote by $l_2(\Omega):=\{\vc: \Omega\rightarrow\C\}$ all complex-valued sequences supported on $\Omega$.
For $j=J_0,\ldots,J$, let 
\begin{equation*}
    \Omega_j:=\{\ell: 1\le \ell \le N_j\},
\end{equation*}
where $N_j:=|\VG_j|$, and $l_2(\Omega_j)$ and $l_2(\VG_j)$ are the sequences supported on $\Omega_j$ and $\VG_j$ respectively.

\begin{definition}[Discrete Fourier transform]
	The (generalized) discrete Fourier transform (DFT) $\fft[j]^*: l_2(\VG_j)\rightarrow l_2(\Omega_j)$ on $\gph_j$ is
\begin{equation}\label{defn:adjfft}
(\fft[j]^*\fracoev[])_\ell := \sum_{[\uG]\in\VG_j}  \fracoev[]([\uG])\: \sqrt{\wV}\:\conj{\eigfm(\pV)}, \quad \ell\in\Omega_j.
\end{equation}
We say the sequence $\fft[j]^*\fracoev[]$ a $(\VG_j,\Omega_j)$-sequence and let $l_2(\VG_j,\Omega_j)$ be the set of all $(\VG_j,\Omega_j)$-sequences.
\end{definition}

\begin{definition}[Adjoint discrete Fourier transform]
	Define $\fft[j]: l_2(\Lambda_j)\rightarrow l_2(\VG_j)$ the (generalized) adjoint discrete Fourier transform (ADFT) (operator) on $\gph_j$ by
\begin{equation}\label{defn:fft}
[\fft[j] \vc]([\uG]) :=   \sum_{\ell\in\Omega_j}c_\ell\: \sqrt{\wV}\:\eigfm(\pV),\quad [\uG]\in\VG_j, \; \vc=(c_\ell)_{\ell=1}^{N_j}\in l_2(\Omega_j).
\end{equation}
We say the sequence $(\fft[j]\vc)$ a $(\Omega_j,\VG_j)$-sequence. Let $l_2(\Omega_j,\VG_j)$ be the set of all $(\Omega_j,\VG_j)$-sequences.
\end{definition}
The following proposition shows that $\fft[j]$ and $\fft[j]^*$ are invertible when the weight sequence for the chain satisfies \eqref{tight:QN}.
\begin{proposition}\label{prop:fft}
Let $\{(\eigfm,\eigvm)\}_{\ell=1}^N$ be a global orthonormal basis for a coarse-grained chain $\gph_{J\rightarrow J_0}$ of the graph $\gph$. Let $\QN[J\rightarrow J_0]$ be a weight sequence for $\gph_{J\rightarrow J_0}$ which satisfies \eqref{tight:QN}. Let $\fft[j]^*$ and $\fft[j]$ be the DFT and ADFT for $\{(\eigfm,\eigvm)\}_{\ell=1}^N$ given in \eqref{defn:adjfft} and \eqref{defn:fft}.
Then, $\fft[j]^*$ and $\fft[j]$ satisfy
\[
    \fft[j]^*\fft[j] = \Id_{\VG_j},\quad\quad\fft[j]\fft[j]^* = \Id_{\Omega_j},
\]
where $\Id_{\VG_j}$ and $\Id_{\Omega_j}$ are the identity operators on $l_2(\VG_j)$ and $l_2(\Omega_j)$.
\end{proposition}
\begin{proof}
Let $\vc=(c_\ell)_{\ell=1}^{N_j}$ be a sequence in $l_2(\Omega_j)$. Then, for $\ell=1,\dots,N_j$,
\begin{align*}
[\fft[j]^*(\fft[j]\vc)]_\ell
& = \sum_{[\uG]\in\VG_j}  [\fft[j]\vc]([\uG])\: \sqrt{\wV}\:\conj{\eigfm(\pV)}\\
& = \sum_{[\uG]\in\VG_j}  \left(\sum_{\ell'=1}^{N_j}c_{\ell'}\: \sqrt{\wV}\:\eigfm[\ell'](\pV)\right) \sqrt{\wV}\:\conj{\eigfm(\pV)}\\
& = \sum_{\ell'=1}^{N_j}c_{\ell'} \left(\sum_{[\uG]\in\VG_j}{\wV}\:\eigfm[\ell'](\pV)\:\conj{\eigfm(\pV)}\right) \\
& = \sum_{\ell'=1}^{N_j}c_{\ell'}\:\QU{\ell'}(\QN[j])\\
& = c_\ell.
\end{align*}
Hence, $[\fft[j]^*(\fft[j]\vc)]=\vc$ for all $\vc\in l_2(\Omega_j)$. In a similar way, we can prove
$\fft[j]\fft[j]^*=\Id_{\Omega_j}$.
\end{proof}

For every $(\Omega_j,\VG_j)$-sequence $\fracoev[]$, there exists a \emph{unique} sequence $\vc\in l_2(\Omega_j)$ such that $\fft[j]\vc = \fracoev[]$. This implies that we can well define \emph{Fourier coefficients} by DFT as $\FT{\fracoev[]}:=\vc=\fft[j]^*\fracoev[]$ for the graph signal $\fracoev[]$.

Based on the discrete Fourier transform operators, we can define convolution, downsampling, and upsampling operators.
Let $\mask\in l_1(\Z)$ be a mask and $\fracoev[]\in l_2(\Omega_j,\VG_j)$ be a $(\Omega_j,\VG_j)$-sequence. Let $\FT{\fracoev[]}:=(\dfcav[\ell])_{\ell\in\Omega_j}$ be its discrete Fourier coefficient sequence.
\begin{definition}[Discrete convolution]
	The discrete convolution $\fracoev[] \dconv \mask$ is defined as the following sequence in $l_2(\Omega_j,\VG_j)$:
\begin{equation}\label{defn:dis.conv}
    [\fracoev[] \dconv \mask]([\uG]) := \sum_{\ell\in \Omega_j} \dfcav[\ell]\: {\FS{\mask}}\left(\frac{\eigvm}{\Lambda_j}\right)\sqrt{\wV}\:\eigfm(\pV),\quad [\uG]\in\VG_j.
\end{equation}
\end{definition}
That is, in the Fourier domain $(\widehat{\fracoev[] \dconv \mask})_\ell = \dfcav[\ell]\:{\FS{\mask}}\left(\frac{\eigvm}{2^{j}}\right)$ for $\ell\in\Omega_j$.
\begin{definition}[Downsampling]
	We define the downsampling operator $\downsmp:l_2(\Omega_j,\VG_j)\rightarrow l_2(\Omega_{j-1},\VG_{j-1})$ for a $(\Omega_j,\VG_j)$-sequence $\fracoev[]$ by
\begin{equation}\label{defn:dwnsmp}
    [\fracoev[]\downsmp]([\uG]) := \sum_{\ell\in\Omega_{j}} \dfcav[\ell] \:\sqrt{\wV[j-1,{[\uG]}]}\:\eigfm({[\uG]}),\quad [\uG]\in\VG_{j-1}.
\end{equation}
\end{definition}

\begin{definition}[Upsampling]
	The upsampling operator $\:\upsmp: l_2(\Omega_{j-1},\VG_{j-1})\rightarrow l_2(\Omega_j,\VG_j)$ for a sequence $\fracoev[]\in l_2(\Omega_{j-1},\VG_{j-1})$ is defined by
\begin{equation}\label{defn:upsmp}
[\fracoev[]\upsmp]([\uG]) :=
\sum_{\ell\in\Omega_{j-1}}\dfcav[\ell] \sqrt{\wV}\:\eigfm(\pV),\quad [\uG]\in\VG_j.
\end{equation}
\end{definition}

For a mask $\mask\in l_1(\Z)$, we denote $\mask^\star$ the filter such that its Fourier series $\FT{\mask^\star}(\xi) = \conj{\FT{h}(\xi)}$, $\xi\in\R$. With the above notions, we have the following theorem which shows the framelet decomposition and reconstruction for the framelet approximation and detail coefficients at different scales.
\begin{theorem}
\label{thm:dec:rec}
Let $\dfrsys(\{\Psi_j\}_{j=J_1}^J,\{\filtbk_j\}_{j=J_1+1}^{J})$,  $J_1=J_0,\ldots, J$ be the sequence of the decimated tight framelet systems given in \eqref{defn:DFS}, and the framelet approximation and detail coefficients given in \eqref{defn:fmtG:coefs}. Then,
\begin{enumerate}
\item[{\rm (i)}] $\fracoev$ is a $(\Omega_j,\VG_j)$-sequence for all $j = J_0,\ldots,J$, and $\frbcoev[j]{n}, n=1,\ldots, r_j$ are $(\Omega_{j+1},\VG_{j+1})$-sequences for all $j = J_0-1,\ldots,J$.
\item[{\rm (ii)}]  For any $J_0+1\le j\le \ord$, the decomposition is given by
\begin{equation}
\label{defn:dec:tz}
  \fracoev[j-1] = (\fracoev\dconv \maska_j^\star)\downsmp,\quad \frbcoev[j-1]{n} =  (\fracoev\dconv (\maskb_j)^\star),\quad  n=1,\ldots,r_{j-1}.
\end{equation}
\item[{\rm (iii)}]  For any $J_0+1\le j\le \ord$, the reconstruction is given by
\begin{equation}\label{defn:reconstr.j.j1}
  \fracoev  =  (\fracoev[j-1]\upsmp) \dconv \maska_{j}+\sum_{n=1}^{r_{j-1}}  \frbcoev[j-1]{n} \dconv \maskb_{j}.
\end{equation}
\end{enumerate}
\end{theorem}
\begin{remark}
For scale (or level) $j$, the low-pass sequence $\fracoev$ is in $l_2(\Omega_j,\VG_j)$ while the high-pass sequences $\frbcoev[j]{n}$ are in $l_2(\Omega_{j+1},\VG_{j+1})$. Figure~\ref{fig:algo:fb} shows the framelet decomposition and reconstruction for the one-level $\gph$-framelet transforms based on a filter bank $\{\maska_j;\maskb[1]_j,\ldots,\maskb[r_j]_j \}$, which consists of discrete convolution, upsampling and downsampling on $\gph$, as stated in Theorem~\ref{thm:dec:rec}. Here the operator $\dconv\maskb_j$ ranges over $n=1,\ldots,r$ and the operation $+_r$ is the summation over the low-pass filtered coefficient sequence and all $r$ high-pass filtered coefficient sequences.
\end{remark}
\begin{proof}[Proof of Theorem~\ref{thm:dec:rec}]
For $\fracoev$ and $\frbcoev{n}$ in \eqref{defn:fmtG:coefs}, by \eqref{eq:fr.coeff} and $\supp\:\FT{\scala_\veps}\subseteq[-1,1]$, we have
\begin{equation*}
    \fracoev([\uG]) =\sum_{\ell\in\Omega_j}\Fcoem{\vf}\: \conj{\FT{\scala_\veps}\left(\frac{\eigvm}{\Lambda_j}\right)}\sqrt{\wV}\: \eigfm(\pV)
\end{equation*}
and
\begin{equation*}
    \frbcoev[j-1]{n}([\uG]) =\sum_{\ell\in\Omega_j}\Fcoem{\vf}\: \conj{\FT{\scalb^n}\left(\frac{\eigvm}{\Lambda_{j-1}}\right)}\sqrt{\wV}\: \eigfm(\pV).
\end{equation*}
Hence, the statement in (i) holds. Moreover,  the discrete Fourier coefficient sequences $\FT{\fracoev}=(\dfcav[j,\ell])_{\ell\in\Omega_j}$ and $\FT{\frbcoev[j-1]{n}}=(\dfcbv[j-1,\ell]{n})_{\ell\in\Omega_j}$ are given by $\dfcav[\ell] = \Fcoem{\vf}\: \conj{\FT{\scala_\veps}\left(\frac{\eigvm}{\Lambda_j}\right)}$
and $\dfcbv[j-1,\ell]{n} = \Fcoem{\vf}\: \conj{\FT{\scalb_{j}^{(n)}}\left(\frac{\eigvm}{\Lambda_{j-1}}\right)}$, $\ell\in\Omega_j$.

Since $\fracoev[j-1]$ is a $(\Lambda_{j-1},\VG_{j-1})$-sequence, \eqref{eq:refinement:nonstationary} implies that for $[\uG]\in\VG_{j-1}$, the approximate framelet coefficients
\begin{align*}
\fracoev[j-1]([\uG]) &= \sum_{\ell\in\Omega_{j-1}}\Fcoem{\vf}\: \conj{\FT{\scala_\veps}\left(\frac{\eigvm}{\Lambda_{j-1}}\right)}\sqrt{\wV[j-1,{[\uG]}]}\:\eigfm(\pV)\notag\\
&= \sum_{\ell\in\Omega_{j-1}}\Fcoem{\vf}\:\conj{\FT{\scala_\veps}\left(\frac{\eigvm}{\Lambda_{j}}\right)}\:
    \conj{\FT{\maska_{j}}\left(\frac{\eigvm}{\Lambda_{j}}\right)}\sqrt{\wV[j-1,{[\uG]}]}\: \eigfm(\pV)\\
    &= \sum_{\ell\in\Omega_{j-1}}  \dfcav[j,\ell]\:\conj{\FT{\maska_{j}}\left(\frac{\eigvm}{\Lambda_{j}}\right)}\sqrt{\wV[j-1,{[\uG]}]}\: \eigfm(\pV)\\
    &= [(\fracoev \dconv \maska_{j}^\star)\downsmp]([\uG]).
\end{align*}
And for the detail framelet coefficients, for $[\uG]\in\VG_j$ and $n=1,\dots,r_{j-1}$,
\begin{align*}
\frbcoev[j-1]{n} ([\uG])
&=\sum_{\ell\in\Omega_j}\Fcoem{\vf}\: \conj{\FT{\scalb_{j}^{(n)}}\left(\frac{\eigvm}{\Lambda_{j-1}}\right)} \sqrt{\wN[j,{[\uG]}]}\: \eigfm(\pV)\\
&= \sum_{\ell\in\Omega_{j}}\Fcoem{\vf}\:\conj{\FT{\scala_\veps}\left(\frac{\eigvm}{\Lambda_{j}}\right)}\:
    \conj{\FT{\maskb[n]_{j}}\left(\frac{\eigvm}{\Lambda_{j}}\right)}\sqrt{\wV[j,{[\uG]}]}\: \eigfm(\pV)\\
&= [(\fracoev \dconv (\maskb_{j})^\star)](\pV).
\end{align*}
This gives \eqref{defn:dec:tz}.

Using the equations $\fracoev[j-1]=(\fracoev\dconv (\maska_{j})^\star)\downsmp$ and $\frbcoev[j-1]{n}=\fracoev\dconv {(\maskb_{j})^\star}$, we write
\begin{align*}
\widetilde{\fracoev[]}
&:=(\fracoev[j-1]\upsmp) \dconv \maska_{j}+\sum_{n=1}^{r_{j-1}}  \frbcoev[j-1]{n} \dconv \maskb_{j}\\
&=\bigl(((\fracoev\dconv (\maska_{j})^\star)\downsmp)\upsmp\bigr) \dconv \maska_{j}+\sum_{n=1}^{r_{j-1}}  \bigl(\fracoev\dconv {(\maskb_{j})^\star}\bigr) \dconv \maskb_{j}
\end{align*}
where $r_{j-1}$ is the number of high-pass filters at level $j-1$. This together with \eqref{defn:dis.conv}, \eqref{defn:dwnsmp}, \eqref{defn:upsmp}, and \eqref{thmeq:DFS:mask:simplified} gives, for $[\uG]\in\VG_j$,
\begin{align*}
\widetilde{\fracoev[]}([\uG])
& = \sum_{\ell\in\Omega_j}\dfcav[j,\ell]\left( \left|\FS{\maska_{j}}\left(\frac{\eigvm}{\Lambda_{j}}\right)\right|^2+\sum_{n=1}^{r_{j-1}}
 \left|\FS{\maskb_{j}}\left(\frac{\eigvm}{\Lambda_{j}}\right)\right|^2\right) \sqrt{\wV}\:{\eigfm}(\pV)\\
& =\sum_{\ell\in\Omega_j}\dfcav[j,\ell] \sqrt{\wV}\:{\eigfm}(\pV)\\
& = \fracoev([\uG]).
\end{align*}
Thus, \eqref{defn:reconstr.j.j1} holds.
\end{proof}

Recursively using the decimated $\gph$-framelet transforms in \eqref{defn:dec:tz} and \eqref{defn:reconstr.j.j1}, we would obtain the multi-level decimated $\gph$-framelet transforms on the graph, in which the graph signal has zero-loss.

\begin{definition}[Decimated $\gph$-framelet transforms]
For a sequence of graph data $\fracoev[J_1]\in l(\Lambda_{J_1},\Omega_{J_1})$, $J_0\le J_1\le \ord$ on $\gph$, the multi-level framelet decomposition on $\gph$ is to compute for $j = \ord,\ldots, {J_1}+1$,
\begin{equation}
\label{defn:dec:tz:J}
  \fracoev[j-1] = (\fracoev\dconv \maska_j^\star)\downsmp,\quad \frbcoev[j-1]{n} =  (\fracoev\dconv (\maskb_j)^\star),\quad  n=1,\ldots,r_{j-1}.
\end{equation}
For a sequence $(\frbcoev[J-1]{1},\ldots, \frbcoev[J-1]{r_{J-1}}, \ldots, \frbcoev[J_0]{1},\ldots,\frbcoev[J_0]{r_{J_0}}, \fracoev[J_0])$ of the framelet coefficients derived from a multi-level decomposition,
the multi-level $\gph$-framelet reconstruction is to evaluate for $j = {J_0}+1,\ldots,J$,
\begin{equation}
\label{defn:rec:sd:J}
    \fracoev  =  (\fracoev[j-1]\upsmp) \dconv \maska_{j}+\sum_{n=1}^{r_{j-1}}  \frbcoev[j-1]{n} \dconv \maskb_{j}.
\end{equation}
We call \eqref{defn:dec:tz:J} and \eqref{defn:rec:sd:J} multi-level decimated $\gph$-framelet transforms.
\end{definition}
Figure~\ref{fig:multi-level-FMT} below illustrates a flowchart for a two-level decomposition and reconstruction of the $\gph$-framelet transforms.
\begin{figure}[th]
\begin{minipage}{\textwidth}
	\centering
\begin{minipage}{\textwidth}
\includegraphics[width=\textwidth]{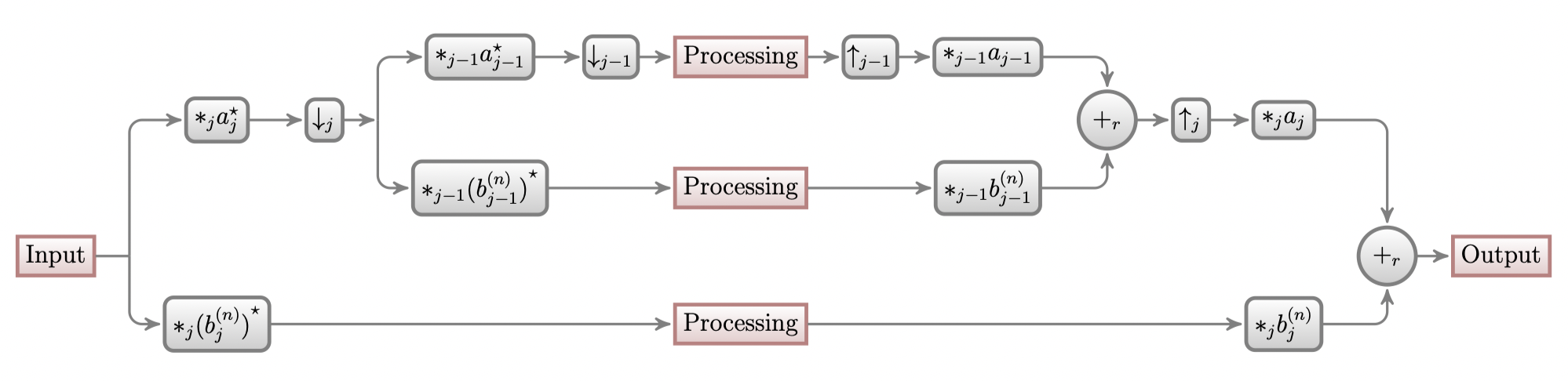}
\end{minipage}
\begin{minipage}{\textwidth}
\caption{Two-level $\gph$-framelet decomposition and reconstruction based on the filter banks $\{\maska_{j-1};\maskb[1]_{j-1},\ldots,\maskb[r_{j-1}]_{j-1}\}$ and $\{\maska_{j};\maskb[1]_{j},\ldots,\maskb[r_{j}]_{j}\}$. Graph signal input is decomposed into the low-pass and high-pass coefficients in the first level decomposition. The low-pass coefficient is then decomposed into low-pass and high-pass coefficients at the next scale. The reconstruction uses the low-pass coefficients at the coarsest level (or the first level) of the chain and all high-pass coefficients.}
\label{fig:multi-level-FMT}
\end{minipage}
\end{minipage}
\end{figure}

\begin{definition}[Framelet analysis and synthesis operators]
The multi-level decimated $\gph$-framelet transforms define the multi-level $\gph$-framelet analysis operator
\begin{equation*}
\label{defn:analOp}
\analOp: l_2(\Lambda_J,\Omega_J) \rightarrow l_2(\Omega_{J-1})^{1\times r_{J-1}}\times l_2(\Omega_{J-2})^{1\times r_{J-2}}\times\cdots\times l_2(\Omega_{J_0})^{1\times r_{J_0}}\times l_2(\Omega_{J_0})
\end{equation*}
with
\begin{equation}
\label{defn:dec:tr:J}
\analOp\fracoev[J] = (\frbcoev[J-1]{1},\ldots, \frbcoev[J-1]{r_{J-1}}, \ldots, \frbcoev[J_0]{1},\ldots,\frbcoev[J_0]{r_{J_0}}, \fracoev[J_0]),\quad \fracoev[J]\in l_2(\Lambda_J,\Omega_J);
\end{equation}
and the multi-level $\gph$-framelet synthesis operator
\begin{equation*}
\label{defn:synOp}
\synOp:  l_2(\Omega_{J-1})^{1\times r_{J-1}}\times l_2(\Omega_{J-2})^{1\times r_{J-2}}\times\cdots\times l_2(\Omega_{J_0})^{1\times r_{J_0}}\times l_2(\Omega_{J_0}) \rightarrow l_2(\Lambda_J,\Omega_{J})
\end{equation*}
with
\[
\synOp( \frbcoev[J-1]{1},\ldots, \frbcoev[J-1]{r_{J-1}}, \ldots, \frbcoev[J_0]{1},\ldots,\frbcoev[J_0]{r_{J_0}}, \fracoev[J_0]) = \fracoev[J].
\]
\end{definition}
Under the condition of Theorem~\ref{thm:dec:rec}, the analysis and synthesis operators are invertible on $l_2(\Lambda_j,\Omega_j)$ for $J_0\le j\le J$, that is, $\synOp\analOp = \Id|_{l_2(\Lambda_j,\Omega_j)}$. The type of the filter banks and the number of high passes in each filter bank in the multi-level $\gph$-framelet transforms may vary at multi scales.

\subsection{Fast $\gph$-Framelet Transforms}
In this section, we show that the decomposition and reconstruction of the decimated $\gph$-framelet transforms can be implemented in linear computational complexity, that is, in steps proportional to the size of the graph by fast discrete Fourier transforms for the coarse-grained chain of the graph $\gph$.
We call this algorithm \emph{fast $\gph$-framelet transform}, or F$\mathcal{G}$T.
The fast computation is due to the following relation between the decimated $\gph$-framelet transforms and discrete Fourier transforms on $\gph$.
\begin{proposition}\label{prop:fgt:fft}
 For $j=J_0+1,\ldots,J$, the $\gph$-framelet decomposition and reconstruction at level $j$ can be written as
\[
  \fracoev[j-1] = \fft[j-1](\widehat{\fracoev\dconv (\maska_{j})^{\star}}),\quad \frbcoev[j-1]{n} = \fft[j](\widehat{\fracoev\dconv {(\maskb_{j})}^\star}),\quad n = 1,\ldots, r_{j-1}
\]
and 
\[
  \fracoev = \left(\fft[j]^*(\fracoev[j-1])\right) \dconv \maska_{j} + \sum_{n=1}^{r_{j-1}} \left(\fft[j]^*(\frbcoev[j-1]{n})\right) \dconv \maskb_{j}.
\]
\end{proposition}
By Proposition~\ref{thm:flop}, the discrete Fourier transforms are implementable fast when the global orthornormal basis is appropriately chosen.
Algorithms~\ref{alg:decomp.multi.level} and \ref{alg:reconstr.multi.level} below give a pseudocode for the decomposition and reconstruction of \fgt based on the formula in Proposition~\ref{prop:fgt:fft}.
Using the Haar global orthonormal basis for the chain(see Section~\ref{sec:haar.orth.basis}), the discrete Fourier transforms for the input data with size $N$ has the computational cost $\bigo{}{N}$.
This and Algorithms~\ref{alg:decomp.multi.level} and \ref{alg:reconstr.multi.level} then show that the computational cost for decimated $\gph$-framelet transforms is $\bigo{}{N}$.

\medskip
\IncMargin{1em}
\begin{algorithm}[t]
\SetKwData{step}{Step}
\SetKwInOut{Input}{Input}\SetKwInOut{Output}{Output}
\BlankLine
\Input{$\fracoev[J]$, which is a $(\Lambda_J,\Omega_J)$-sequence; filter bank }

\Output{$(\dfcbv[J-1]{1},\ldots, \dfcbv[J-1]{r_{J-1}}, \ldots, \dfcbv[J_0]{1},\ldots,\dfcbv[J_0]{r_{J_0}},\dfcav[J_0])$}

$\fracoev[J] \longrightarrow \dfcav[J]$ \tcp*[f]{fast DFT}\\

\For{$j\leftarrow \ord$ \KwTo $J_0+1$}{
    $\dfcav[j-1] \longleftarrow \dfcav[j,\cdot] \: \conj{\FT{\maska_{j}}}\left(\eigvm[\cdot]/\Lambda_{j}\right)$ \tcp*[f]{discrete convolution at level $j$}\\[1mm]
    \tcp*[f]{downsample from level $j$ to $j-1$}\\[1mm]

\For{$n\leftarrow 1$ \KwTo $r_{j-1}$}{
    $\dfcbv[j-1]{n} \longleftarrow  \dfcav[j,\cdot] \: \conj{\FT{(\maskb_j})}\left(\eigvm[\cdot]/\Lambda_{j}\right)$ \tcp*[f]{discrete convolution at level $j$}
    $\frbcoev[j-1]{n} \longleftarrow \dfcbv[j-1]{n}$ \tcp*[f]{fast DFT}\\
    }
$\fracoev[J_0]\longleftarrow \dfcav[J_0]$\tcp*[f]{fast DFT}
}

\caption{Decomposition for \fgt}
\label{alg:decomp.multi.level}
\end{algorithm}
\DecMargin{1em}\vspace{3mm}

\IncMargin{1em}
\begin{algorithm}[t]
\SetKwData{step}{Step}
\SetKwInOut{Input}{Input}\SetKwInOut{Output}{Output}
\BlankLine
\Input{$(\dfcbv[J-1]{1},\ldots, \dfcbv[J-1]{r_{J-1}}, \ldots, \dfcbv[J_0]{1},\ldots,\dfcbv[J_0]{r_{J_0}},\dfcav[J_0])$}
\Output{$\fracoev[J]$}
$\dfcav[J_0] \longleftarrow \fracoev[J_0]$ \tcp*[f]{fast DFT}\\
\For{$j\leftarrow J_0+1$ \KwTo $\ord$}{
\For{$n\leftarrow 1$ \KwTo $r_{j-1}$}{
$\dfcbv[j-1]{n}\longleftarrow \frbcoev[j-1]{n}$ \tcp*[f]{fast DFT}\\
}
$\dfcav \longleftarrow (\dfcav[j-1,\cdot])\:\FT{\maska_{j}}\left(\eigvm[\cdot]/\Lambda_{j}\right)  + \sum_{n=1}^{r_{j-1}}\dfcbv[j-1,\cdot]{n} \;\FT{(\maskb_j)}\left(\eigvm[\cdot]/\Lambda_{j}\right)$
}
$\fracoev[J] \longleftarrow \dfcav[J]$ \tcp*[f]{fast ADFT}\\
\caption{Reconstruction for \fgt}\label{alg:reconstr.multi.level}
\end{algorithm}
\DecMargin{1em}

\section{Toy Example to Illustrate Framelet Construction}\label{sec:toy}
In this section, we show full construction of the decimated framelet system on a graph using the example in Figure~\ref{fig:toyG}. The graph $\gph=(\VG,\EG,\wG)$ has vertices and edges
\[
\VG:=\{a,b,c,d,e,f\}\mbox{ and }
E=\{(a,b),(a,c),(c,d),(c,e),(c,f),(d,e)\}.
\]

We apply Algorithm~\ref{algo:coarse-grain} for clustering. As shown by Figure~\ref{fig:coarse-grain-toy}, the resulting coarse-grained chain of $\gph$ is $\gph_{3\to0}$: the graph $\gph=:\gph_3$ which is at the bottom level 3 includes all six nodes of the original graph; at level 2 is $\gph_2$ which has 3 clusters with three initial centers $\{a\}, \{c\}, \{f\}$; the next level is $\gph_1$ with 2 clusters; and eventually the coarsest level $\gph_0$ has 1 cluster, which is the root of the chain (or tree). We give the detailed description of the coarse-grained chain $\gph_{3\rightarrow 0}=(\gph_3,\gph_2,\gph_1,\gph_0)$, as follow.

\begin{figure}[t]
\begin{minipage}{\textwidth}
	\centering
\begin{minipage}{\textwidth}
\centering
\includegraphics[width=0.5\textwidth]{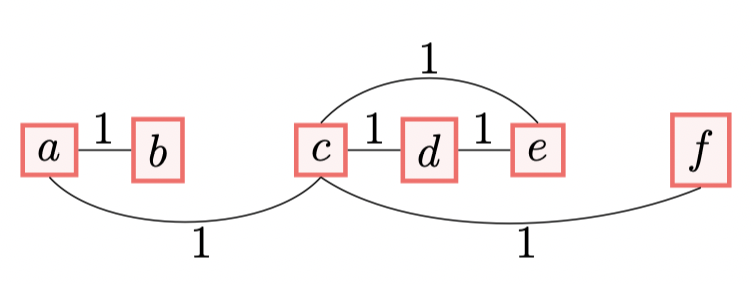}
\end{minipage}\vspace{1mm}
\begin{minipage}{0.85\textwidth}
\caption{Graph $\gph$, where the vertices are represented by the boxes and the edges are by the lines for the pairs of connected vertices, and the weight for each edge is $1$.}
\label{fig:toyG}
\end{minipage}
\end{minipage}
\end{figure}
\begin{enumerate}[(1)]
\item At the finest level 3, $\gph_3:=\gph$, of which each vertex is a leaf and a cluster of singleton. The graph $\gph$ is associated with the adjacency matrix $\wG$, the degree matrix $\dG$, and the graph Laplacian matrix $\gL:=\dG-\wG$:
\begin{equation}\label{ex:w-d-L-3}
\hspace{-1cm}\begin{array}{l}
\wG =
\left[
\begin{matrix}
0 & 1 & 1  & 0 & 0 & 0\\
1 & 0 & 0  & 0 & 0 & 0\\
1 & 0 & 0  & 1 & 1 & 1\\
0 & 0 & 1  & 0 & 1& 0\\
0 & 0 & 1  & 1 & 0 & 0\\
0 & 0 & 1  & 0 & 0 & 0\\
\end{matrix}
\right]
\dG=\left[
\begin{matrix}
2 & 0 & 0  & 0 & 0 & 0\\
0 & 1 & 0  & 0 & 0 & 0\\
0 & 0 & 4  & 0 & 0 & 0\\
0 & 0 & 0  & 2 & 0 & 0\\
0 & 0 & 0  & 0 & 2 & 0\\
0 & 0 & 0  & 0 & 0 & 1\\
\end{matrix}
\right]
\gL=\left[
\begin{matrix}
2 & -1 & -1  & 0 & 0 & 0\\
-1 & 1 & 0  & 0 & 0 & 0\\
-1 & 0 & 4  & -1 & -1 & -1\\
0 & 0 & -1  & 2 & -1& 0\\
0 & 0 & -1  & -1 & 2 & 0\\
0 & 0 & -1  & 0 & 0 & 1\\
\end{matrix}
\right].
\end{array}
\end{equation}
Here, the row or column is with respect to the vertices in the order $a,b,c,d,e,f$.

\item At level 2, we obtain three clusters $[a]_{\gph_2}=\{a,b\}$, $[c]_{\gph_2}=\{c,d,e\}$, and $[f]_{\gph_2}=\{f\}$ for the coarse-grained graph $\gph_2:=(\VG_2,\EG_2,\wG_2)$ of $\gph_3$, where $\VG_2  =\{[a]_{\gph_2},[c]_{\gph_2},[f]_{\gph_2}\}$,
    $\EG_2=\{([a]_{\gph_2},[a]_{\gph_2}),([a]_{\gph_2},[c]_{\gph_2}),([c]_{\gph_2},[c]_{\gph_2})$, $([c]_{\gph_2},[f]_{\gph_2})\}$,
    and
\begin{equation}\label{ex:w-d-L-2}
\wG_2=\frac{1}{12}
\left[
\begin{matrix}
2 & 1 & 0\\
1 & 6 & 1\\
0 & 1 & 0\\
\end{matrix}
\right]\;
\dG_2=\frac{1}{12}\left[
\begin{matrix}
3 & 0 & 0\\
0 & 8 & 0\\
0 & 0 & 1\\
\end{matrix}
\right]\;
\gL_2=\frac{1}{12}\left[
\begin{matrix}
1 & -1 & 0\\
-1 & 2 & -1\\
0 & -1 & 1\\
\end{matrix}
\right].
\end{equation}

\item At level 1,  we obtain two clusters $[a]_{\gph_1}=\{a,b\}$ and $[c]_{\gph_1}=\{c,d,e,f\}$ for the coarse-grained graph $\gph_1:=(\VG_1,\EG_1,\wG_1)$ of $\gph_2$, where
$\VG_1=\bigl\{[a]_{\gph_1},[c]_{\gph_1}\bigr\}$,\\ $\EG_1=\bigl\{([a]_{\gph_1},[a]_{\gph_1}),([a]_{\gph_1},[c]_{\gph_1}),([c]_{\gph_1},[c]_{\gph_1})\bigr\}$,
and
\begin{equation}\label{ex:w-d-L-1}
\wG_1=\frac{1}{12}
\left[
\begin{matrix}
2 & 1 \\
1 & 8 \\
\end{matrix}
\right]\;
\dG_1=\frac{1}{12}\left[
\begin{matrix}
3 & 0 \\
0 & 9 \\
\end{matrix}
\right]\;
\gL_1=\frac{1}{12}
\left[
\begin{matrix}
1 & -1\\
-1 & 1\\
\end{matrix}
\right].
\end{equation}

\item At the coarsest level 0, we reach the root $\gph_0:=(\VG_0,\EG_0,\wG_0)$, where $\VG_0=\{[a,b,c,d,e,f]=:[a]_{\gph_0}\}$ has only one cluster $[a]_{\gph_0}$ which contains all vertices from $\gph$, $\EG_0=\{([a]_{\gph_0}, [a]_{\gph_0})\}$, and $\wG_0=1/12$.

Next, we build the Haar global orthonormal eigen-pairs $\{(\eigfm,\eigvm)\}_{\ell=1}^6$ for $L_2(\gph)$. It utilises the hierarchical information of the chain $\gph_{3\rightarrow 0}$, and the resulting Haar basis satisfies
for $j=0,1,2,3$,
\begin{equation}
\label{property:constant}
\eigfm(\vG)\equiv \mathrm{const} \quad \forall v\in[\vG]_{\gph_j}\mbox{ and } \forall \ell \le |\VG_j|.
\end{equation}

\item At level 0, $\gph_0$ is a graph of singleton. In this case, $\lambda_1^{\gph_0}=0$ and $\eigfm[1]^{\gph_0}=1$. We then let
\[
\eigfm[1] = \frac{1}{\sqrt{6}}\left[\begin{matrix}
1 & 1 & 1 & 1 & 1 & 1
\end{matrix}
\right]^\top.
\]

\item At level 1, the eigenvalues of $\gL_1$ as in \eqref{ex:w-d-L-1} are $\lambda_1^{\gph_1}=0$ and $\lambda_2^{\gph_1}=\frac{1}{6}$. The eigenvectors of $\gL_1$ with respect to $0, \frac{1}{6}$ are
\[
\eigfm[1]^{\gph_1}=
\frac{1}{\sqrt{2}}\left[
\begin{matrix}
1 & 1
\end{matrix}
\right]^\top,\quad
\eigfm[2]^{\gph_1}=
\frac{1}{\sqrt{2}}\left[
\begin{matrix}
1 & -1
\end{matrix}
\right]^\top.
\]
By the discussion in Section~\ref{sec:haar.orth.basis}, we extend $\eigfm[2]^{\gph_1}$, with respect to clusters $[a]_{\gph_1}$ and $[c]_{\gph_1}$, to $\eigfm[2]^{(1)}$ on $\gph$:
\[
\eigfm[2]^{(1)} = \frac{1}{\sqrt{6}}\left[\begin{matrix}
1 & 1 & -1 & -1 & -1 & -1
\end{matrix}
\right]^\top.
\]
Apply the Gram-Schmidt orthonormalization process to $\{\eigfm[1],\eigfm[2]^{(1)}\}$, we then obtain a new vector
\[
\eigfm[2] = \frac{1}{2\sqrt{3}}\left[\begin{matrix}
2 & 2 & -1 & -1 & -1 & -1
\end{matrix}
\right]^\top.
\]

\item At level 2, the eigenvalues of $\gL_2$ in \eqref{ex:w-d-L-2} are $\lambda_1^{\gph_2}=0$,  $\lambda_2^{\gph_2}=\frac{1}{12}$, $\lambda_3^{\gph_2}=\frac{1}{4}$. The eigenvectors of $\gL_2$ with respect to $0$, $\frac{1}{12}$, $\frac{1}{4}$ are
\begin{equation*}
\eigfm[1]^{\gph_2}=
\frac{1}{\sqrt{3}}\left[
\begin{matrix}
1 & 1 & 1
\end{matrix}
\right]^\top\quad
\eigfm[2]^{\gph_2}=
\frac{1}{\sqrt{2}}\left[
\begin{matrix}
1 & 0 & -1
\end{matrix}
\right]^\top\quad
\eigfm[3]^{\gph_2}=
\frac{1}{\sqrt{6}}\left[
\begin{matrix}
1 & -2 & 1
\end{matrix}
\right]^\top.
\end{equation*}
We extend $\eigfm[3]^{\gph_2}$, with respect to clusters $[a]_{\gph_2}$,  $[c]_{\gph_2}$ and $[f]_{\gph_2}$, to $\eigfm[3]^{(2)}$ on $\gph$ as
\[
\eigfm[3]^{(2)} = \frac{1}{\sqrt{6}}\left[\begin{matrix}
1 & 1 & -1 & -1 & -1 & 1
\end{matrix}
\right]^\top.
\]
Apply the Gram-Schmidt orthonormalization process to $\{\eigfm[1],\eigfm[2],\eigfm[3]^{\gph_3}\}$, we then obtain a new vector $\eigfm[3]$:
\[
\eigfm[3] = \frac{1}{2\sqrt{3}}\left[\begin{matrix}
0 & 0 & -1 & -1 & -1 & 3
\end{matrix}
\right]^\top.
\]
\item
Continue the above similar steps, at level 3, from the graph Laplacian in \eqref{ex:w-d-L-3}, we obtain an orthonormal basis $\{\eigfm[1],\ldots,\eigfm[6]\}$ for $L_2(\gph)$ satisfying \eqref{property:constant}  as
\[
\begin{aligned}
\eigfm[1]&=
\frac{1}{\sqrt{6}} \left[
\begin{matrix}
1&1&1&1&1 &1\\
\end{matrix}
\right]^\top\\
\eigfm[2]&
=\frac{1}{2\sqrt{3}}\left[
\begin{matrix}
2&2&-1&-1&-1&-1]\\
\end{matrix}
\right]^\top
\\
\eigfm[3]&=
\frac{1}{2\sqrt{3}}\left[
\begin{matrix}
0 & 0 & -1 &-1 &-1 &3\\
\end{matrix}
\right]^\top\\
\eigfm[4]&=\frac{1}{\sqrt{6}}
\left[
\begin{matrix}
0 &0&2&-1&-1&0\\
\end{matrix}
\right]^\top\\
\eigfm[5]&=
\frac{1}{\sqrt{2}}\left[
\begin{matrix}
0 & 0 & 0 & 1 & -1& 0\\
\end{matrix}
\right]^\top\\
\eigfm[6]&=
\frac{1}{\sqrt{2}} \left[
\begin{matrix}
1& -1& 0 & 0 & 0 & 0\\
\end{matrix}
\right]^\top,\\
\end{aligned}
\]
and we let $\eigvm=\ell-1$ for $\ell=1,\ldots, 6$.

Based on the orthonormal eigen-pair $\{(\eigfm,\eigvm)\}_{\ell=1}^6$, we next construct decimated framelet systems as in \eqref{defn:DFS}.

\item At level 3, $\gph_3\equiv\gph$ and $[p]_{\gph_3} = \{p\}\in \VG_3$ are singletons. Let $\wV=1$ for all $p\in\VG$, and $\FT{\scala_{3}}\left(\frac{\eigvm}{\Lambda_{3}}\right)\equiv 1$ for all $\ell$, then by \eqref{defn:fra.frb},
\[
\fra[3,\pV](\vG) =\fra[3,p] = \delta_{\pV,\vG}\quad \forall p,\vG\in\VG.
\]
There is no framelet $\frb{(n)}(\vG)$ at this level. The system $\{\fra[3,\pV]: \pV\in\VG_3\}=\{\delta_{p,v}: p,v\in\VG\}$ is the trivial orthonormal basis.

\item At level 2, by the discussion in Section~\ref{sec:constr.dtf}, the scaling functions are given by
\begin{align*}
\FT{\scala_{2}}\left(\frac{\eigvm}{\Lambda_{2}}\right) &=
\begin{cases}
1 & \ell=1,2\\
\frac{1}{\sqrt{2}} & \ell=3\\
0 & \mbox{otherwise.}
\end{cases}\\[2mm]
\FT{\scalb_{2}^{(1)}}\left(\frac{\eigvm}{\Lambda_{2}}\right) &=
\begin{cases}
\frac{1}{\sqrt{2}} & \ell=3,5\\
1 & \ell=4\\
0 & \mbox{otherwise.}
\end{cases}\quad
\FT{\scalb_{2}^{(2)}}\left(\frac{\eigvm}{\Lambda_{2}}\right)=
\begin{cases}
\frac{1}{\sqrt{2}}  & \ell=5\\
1 & \ell=6\\
0 & \mbox{otherwise.}
\end{cases}
\end{align*}
It can be verified that $\bigl|\FT{\scala_{2}}(\frac{\eigvm}{\Lambda_{2}})\bigr|^2
+\bigl|\FT{\scalb_{2}^{(1)}}(\frac{\eigvm}{\Lambda_{2}})\bigr|^2
+\bigl|\FT{\scalb_{2}^{(2)}}(\frac{\eigvm}{\Lambda_{2}})\bigr|^2\equiv 1$ for all $\ell$. By \eqref{defn:fra.frb}, we then obtain $\fra[2,\pV]$, $\frb[2,\pV]{(1)}$ and $\frb[2,\pV]{(2)}$ in Table~\ref{tab:phi2psi2}, where the corresponding weights on $\VG_2$ are
\[
\wV[2,{[a]_{\gph_2}}] = 2\quad \wV[2,{[c]_{\gph_2}}] = 3 \quad \wV[2,{[f]_{\gph_2}}] = 1.
\]
\begin{table}[th]
\begin{center}
\begin{tabular}{c|cccccc}
  \hline
& $a$ & $b$ & $c$ & $d$ & $e$ & $f$\\
\hline
$\fra[2,{[a]_{\gph_2}}]$ &$\frac{1}{\sqrt{2}}$ &  $\frac{1}{\sqrt{2}}$  & 0 & 0 & 0 & 0\\
$\fra[2,{[c]_{\gph_2}}]$ & 0 & 0 & $\frac{\sqrt{3}}{4}+\frac{\sqrt{6}}{24}$ & $\frac{\sqrt{3}}{4}+\frac{\sqrt{6}}{24}$ &$\frac{\sqrt{3}}{4}+\frac{\sqrt{6}}{24}$ & $\frac{\sqrt{3}}{4}-\frac{\sqrt{6}}{8}$\\
$\fra[2,{[f]_{\gph_2}}]$ & 0 & 0 & $\frac14-\frac{\sqrt{2}}{8}$ & $\frac14-\frac{\sqrt{2}}{8}$ & $\frac14-\frac{\sqrt{2}}{8}$ & $\frac14+\frac{3\sqrt{2}}{8}$\\
\hline
$\frb[2,{[a]}_{\gph_3}]{(1)}$ & 0 & 0  & 0 & 0 & 0 & 0\\
$\frb[2,{[b]}_{\gph_3}]{(1)}$ &0  & 0  & 0 & 0 & 0 & 0\\
$\frb[2,{[c]}_{\gph_3}]{(1)}$ & 0 & 0  & $\frac{2}{3}+\frac{\sqrt{2}}{24}$ & $-\frac{1}{3}+\frac{\sqrt{2}}{24}$  & $-\frac{1}{3}+\frac{\sqrt{2}}{24}$ & $-\frac{\sqrt{2}}{8}$\\
$\frb[2,{[d]}_{\gph_3}]{(1)}$ & 0 & 0  & $-\frac{1}{3}+\frac{\sqrt{2}}{24}$ & $\frac{1}{6}+\frac{7\sqrt{2}}{24}$  & $\frac{1}{6}-\frac{5\sqrt{2}}{24}$ & $-\frac{\sqrt{2}}{8}$\\
$\frb[2,{[e]}_{\gph_3}]{(1)}$ & 0 & 0  & $-\frac{1}{3}+\frac{\sqrt{2}}{24}$ & $\frac{1}{6}-\frac{5\sqrt{2}}{24}$  & $\frac{1}{6}+\frac{7\sqrt{2}}{24}$ & $-\frac{\sqrt{2}}{8}$\\
$\frb[2,{[f]}_{\gph_3}]{(1)}$ & 0 & 0  & $-\frac{\sqrt{2}}{8}$ &  $-\frac{\sqrt{2}}{8}$ & $-\frac{\sqrt{2}}{8}$ & $\frac{3\sqrt{2}}{8}$\\
\hline
$\frb[2,{[a]}_{\gph_3}]{(2)}$ & $\frac12$ & $-\frac12$  & 0 & 0 & 0 & 0\\
$\frb[2,{[b]}_{\gph_3}]{(2)}$ & $-\frac12$  & $\frac12$  & 0 & 0 & 0 & 0\\
$\frb[2,{[c]}_{\gph_3}]{(2)}$ & 0 & 0  &  0 & 0 & 0 &0 \\
$\frb[2,{[d]}_{\gph_3}]{(2)}$ & 0 & 0  & 0 & $\frac{\sqrt{2}}{4}$  & $-\frac{\sqrt{2}}{4}$ & 0 \\
$\frb[2,{[e]}_{\gph_3}]{(2)}$ & 0 & 0  & 0 & $-\frac{\sqrt{2}}{4}$  & $\frac{\sqrt{2}}{4}$ & 0 \\
$\frb[2,{[f]}_{\gph_3}]{(2)}$ & 0 & 0  & 0 & 0 & 0 & 0\\
  \hline
\end{tabular}
\end{center}
\caption{Decimated framelets $\fra[2,\pV_{\gph_2}], \frb[2,\pV_{\gph_3}]{(1)}$ and $\frb[2,\pV_{\gph_3}]{(2)}$ at level $j=2$}
\label{tab:phi2psi2}
\end{table}

\item At level 1, set
\[
\FT{\scala_{1}}\left(\frac{\eigvm}{\Lambda_{1}}\right)=
\begin{cases}
1 & \ell=1\\
\frac{1}{\sqrt{2}} & \ell=2\\
0 & \mbox{otherwise.}
\end{cases}\quad
\FT{\scalb_{1}^{(1)}}\left(\frac{\eigvm}{\Lambda_{1}}\right)=
\begin{cases}
\frac{1}{\sqrt{2}} & \ell=2\\
\frac{1}{\sqrt{2}} & \ell=3\\
0 & \mbox{otherwise.}
\end{cases}
\]
which satisfies $\bigl|\FT{\scala_{1}}(\frac{\eigvm}{\Lambda_{1}})\bigr|^2
+\bigl|\FT{\scalb_{1}^{(1)}}(\frac{\eigvm}{\Lambda_{1}})\bigr|^2= \bigl|\FT{\scala_{2}}(\frac{\eigvm}{\Lambda_{2}})\bigr|^2$ for all $\ell$.
Letting the weights on $\VG_1$ as $\wV[1,{[a]_{\gph_1}}] = 2, \wV[2,{[c]_{\gph_1}}] = 4$, by \eqref{defn:fra.frb}, we then obtain $\fra[1,\pV]$ and $\frb[1,\pV]{(1)}$ in Table~\ref{tab:phi1psi1}.
\begin{table}[th]
\begin{center}
\begin{tabular}{c|cccccc}
  \hline
& $a$ & $b$ & $c$ & $d$ & $e$ & $f$\\
\hline
$\fra[1,{[a]_{\gph_1}}]$ & $\frac13+\frac{\sqrt{2}}{6}$ &  $\frac13+\frac{\sqrt{2}}{6}$  & $-\frac16+\frac{\sqrt{2}}{6}$ & $-\frac16+\frac{\sqrt{2}}{6}$ & $-\frac16+\frac{\sqrt{2}}{6}$ & $-\frac16+\frac{\sqrt{2}}{6}$\\
$\fra[1,{[c]_{\gph_1}}]$ & $\frac13-\frac{\sqrt{2}}{6}$ &  $\frac13-\frac{\sqrt{2}}{6}$  & $\frac13+\frac{\sqrt{2}}{12}$ & $\frac13+\frac{\sqrt{2}}{12}$ & $\frac13+\frac{\sqrt{2}}{12}$ & $\frac13+\frac{\sqrt{2}}{12}$\\
\hline
$\frb[1,{[a]_{\gph_2}}]{(1)}$ & $\frac13$ & $\frac13$ & $-\frac16$ & $-\frac16$ & $-\frac16$ & $-\frac16$ \\
$\frb[1,{[c]_{\gph_2}}]{(1)}$ &$-\frac{\sqrt{6}}{12}$  & $-\frac{\sqrt{6}}{12}$  & $\frac{\sqrt{6}}{12}$ &$\frac{\sqrt{6}}{12}$ & $\frac{\sqrt{6}}{12}$ & $-\frac{\sqrt{6}}{12}$\\
$\frb[1,{[f]_{\gph_2}}]{(1)}$ & $-\frac{\sqrt{2}}{12}$ & $-\frac{\sqrt{6}}{12}$ & $-\frac{\sqrt{2}}{12}$ & $-\frac{\sqrt{2}}{12}$ & $-\frac{\sqrt{2}}{12}$ &$-\frac{5\sqrt{2}}{12}$\\
  \hline
\end{tabular}
\end{center}
\caption{Decimated framelets $\fra[1,\pV_{\gph_1}]$ and $\frb[1,\pV_{\gph_2}]{(1)}$ at level $j=1$}
\label{tab:phi1psi1}
\end{table}

\item At level 0,  set
\[
\FT{\scala_{0}}\left(\frac{\eigvm}{\Lambda_{0}}\right)=
\begin{cases}
1 & \ell=1\\
0 & \mbox{otherwise.}
\end{cases}\quad
\FT{\scalb_{0}^{(1)}}\left(\frac{\eigvm}{\Lambda_{0}}\right)=
\begin{cases}
\frac{1}{\sqrt{2}} & \ell=2\\
0 & \mbox{otherwise.}
\end{cases}
\]
Here $\bigl|\FT{\scala_{0}}(\frac{\eigvm}{\Lambda_{0}})\bigr|^2
+\bigl|\FT{\scalb_{0}^{(1)}}(\frac{\eigvm}{\Lambda_{0}})\bigr|^2= \bigl|\FT{\scala_{1}}(\frac{\eigvm}{\Lambda_{1}})\bigr|^2$ for all $\ell$.
Let the weights on $\VG_0$ as $\wV[1,{[a]_{\gph_0}}] = 6$, then by \eqref{defn:fra.frb}, we obtain $\fra[0,\pV]$ and $\frb[0,\pV]{(1)}$ in Table~\ref{tab:phi0psi0}.
\begin{table}[th]
\begin{center}
\begin{tabular}{c|cccccc}
  \hline
& $a$ & $b$ & $c$ & $d$ & $e$ & $f$\\
\hline
$\fra[1,{[a]_{\gph_1}}]$ &$\frac{1}{6}$ &$\frac{1}{6}$ &$\frac{1}{6}$ &$\frac{1}{6}$ &$\frac{1}{6}$ &$\frac{1}{6}$\\
\hline
$\frb[1,{[a]_{\gph_2}}]{(1)}$ & $\frac13$ & $\frac13$ & $-\frac16$ & $-\frac16$ & $-\frac16$ & $-\frac16$\\
$\frb[1,{[c]_{\gph_2}}]{(1)}$ &$-\frac{\sqrt{2}}{6}$  &$-\frac{\sqrt{2}}{6}$  & $\frac{\sqrt{2}}{12}$ &$\frac{\sqrt{2}}{12}$ & $\frac{\sqrt{2}}{12}$ & $\frac{\sqrt{2}}{12}$\\
  \hline
\end{tabular}
\end{center}
\caption{Decimated framelets $\fra[0,\pV_{\gph_0}]$ and $\frb[0,\pV_{\gph_1}]{(1)}$ at level $j=0$}
\label{tab:phi0psi0}
\end{table}
\end{enumerate}
It can be verified that conditions in \eqref{tight:QN} and \eqref{thmeq:DFS:2scale:alpha:beta:simplified} hold for the above framelets constructed through the above steps (1)--(12). Hence, by Theorem~\ref{thm:DFS} the framelet system for $J_1=0,1,2,3$,
\[
\{\fra[{J_1},\pV] \setsep \pV\in\VG_{J_1}\}\cup \{\frb{(n)} \setsep \pV\in\VG_{j+1}, \: j = J_1,\ldots,J\}
\]
is a decimated tight frame for $L_2(\gph)$.

\section{Numerical Examples}\label{sec:numer}
In this section, we present three experiments for the {\fgt} algorithm. We show the computational complexity analysis of {\fgt} in Section~\ref{exp1} and the multiscale analysis by decimated framelets for the real-world traffic network in Section~\ref{exp2}. Also, we use {\fgt} to define a spectral graph convolution in Section~\ref{sec:fgconv}, which shows good performance in graph-level classification. The Python and Matlab codes for {\fgt} can be downloaded from Github\footnote{\url{https://github.com/YuGuangWang/FGT}}.

\subsection{Filter Bank for Decimated Framelets}
We show examples of the filter bank for decimated framelets on a graph with different numbers of high passes. The constructed filter banks use the parameters in the example shown in Section~\ref{sec:constr.dtf}. Figure~\ref{fig:filter_banks} shows three types of filter banks, with the numbers of high passes 1, 2, or 3. For each case, we show two examples of filter bank with a slightly different choice of intersection points (between the curves of low-pass and high-pass filters). The explicit values for the parameters ($\zeta_c^a, \zeta_c^b, \zeta_c^{b^{(1)}}$ and $\zeta_c^{b^{(2)}}$ are given in the caption of the figure.
From left to right, we show the filters from the finest level to the coarsest level. Note that the low-pass filter's support is strictly $[0, N_j]$ for the chain's $j$th level. The lower end of any high-pass filter's support is always positive, and the support is dependent on the chain structure.

\begin{figure*}[]
    \centering
    \scriptsize
Filter bank with 3 high passes
\begin{tabular}{cccc}
  \includegraphics[width=.19\textwidth]{figures/framelet_3high_pass_level3_ac025_b1c025_b2c025.png}
  &\includegraphics[width=.19\textwidth]{figures/framelet_3high_pass_level2_ac025_b1c025_b2c025.png}
  &\includegraphics[width=.19\textwidth]{figures/framelet_3high_pass_level1_ac025_b1c025_b2c025.png}
  &\includegraphics[width=.19\textwidth]{figures/framelet_3high_pass_level0_ac025_b1c025_b2c025.png}\\
  \includegraphics[width=.19\textwidth]{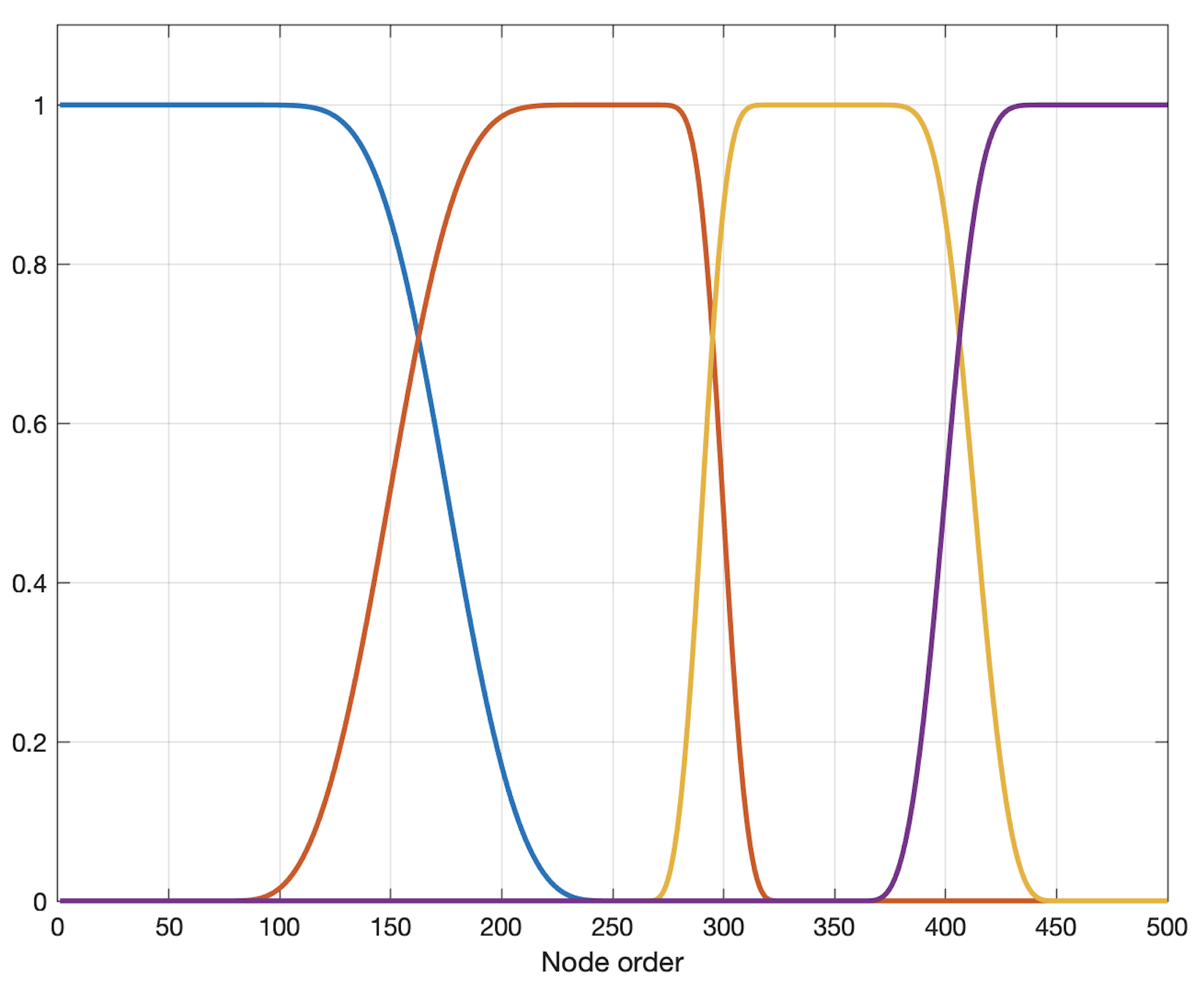}
  &\includegraphics[width=.19\textwidth]{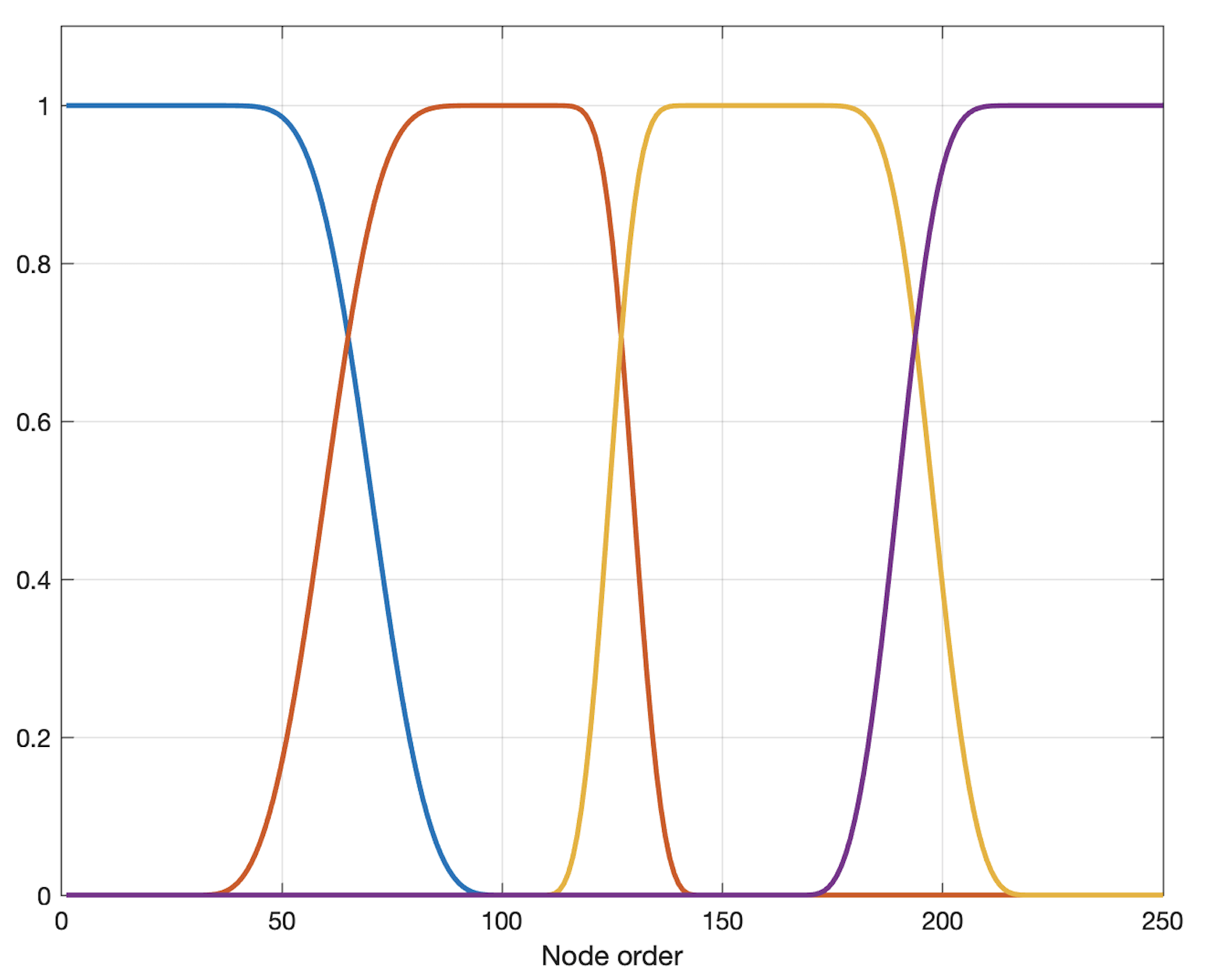}
  &\includegraphics[width=.19\textwidth]{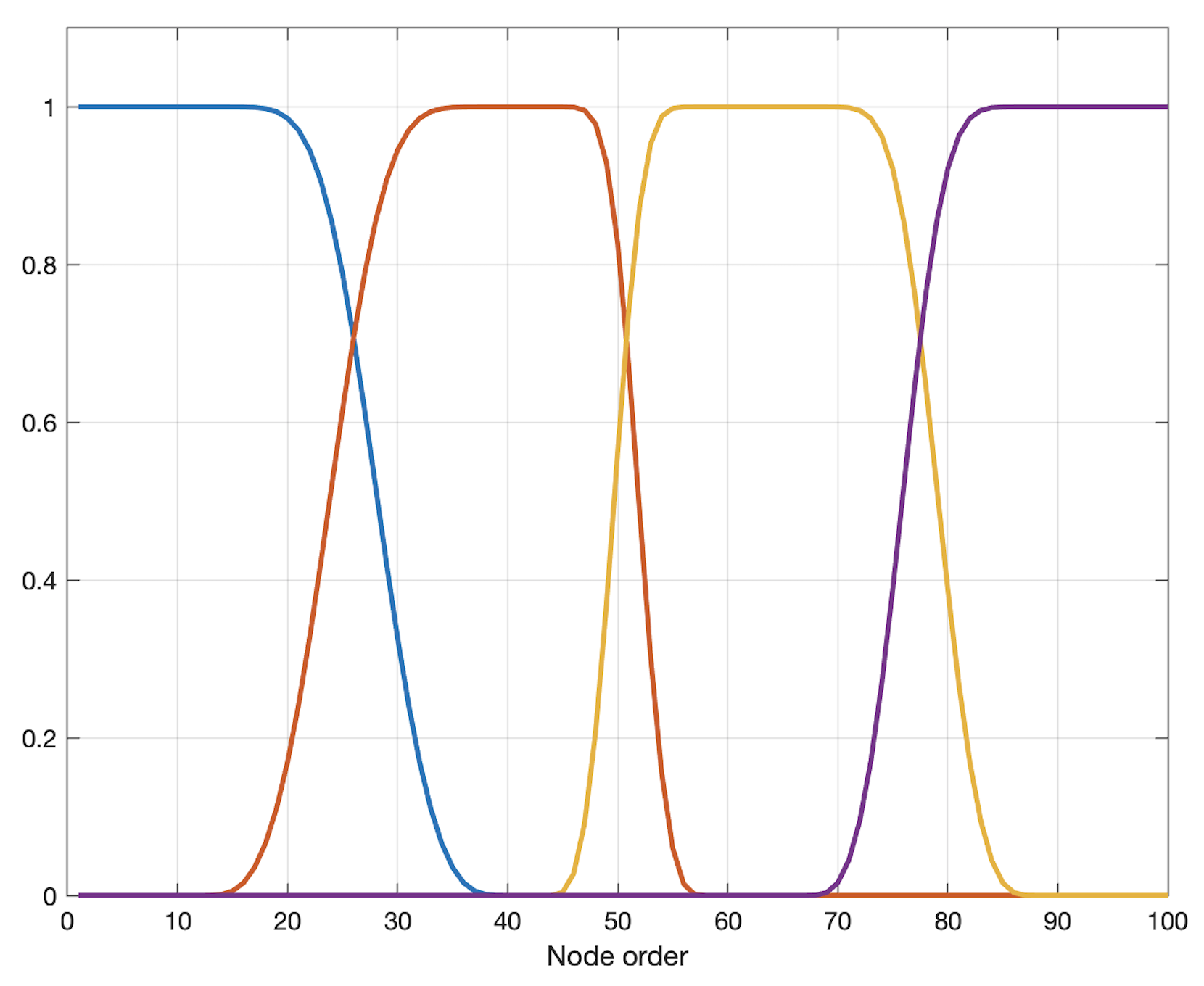}
  &\includegraphics[width=.19\textwidth]{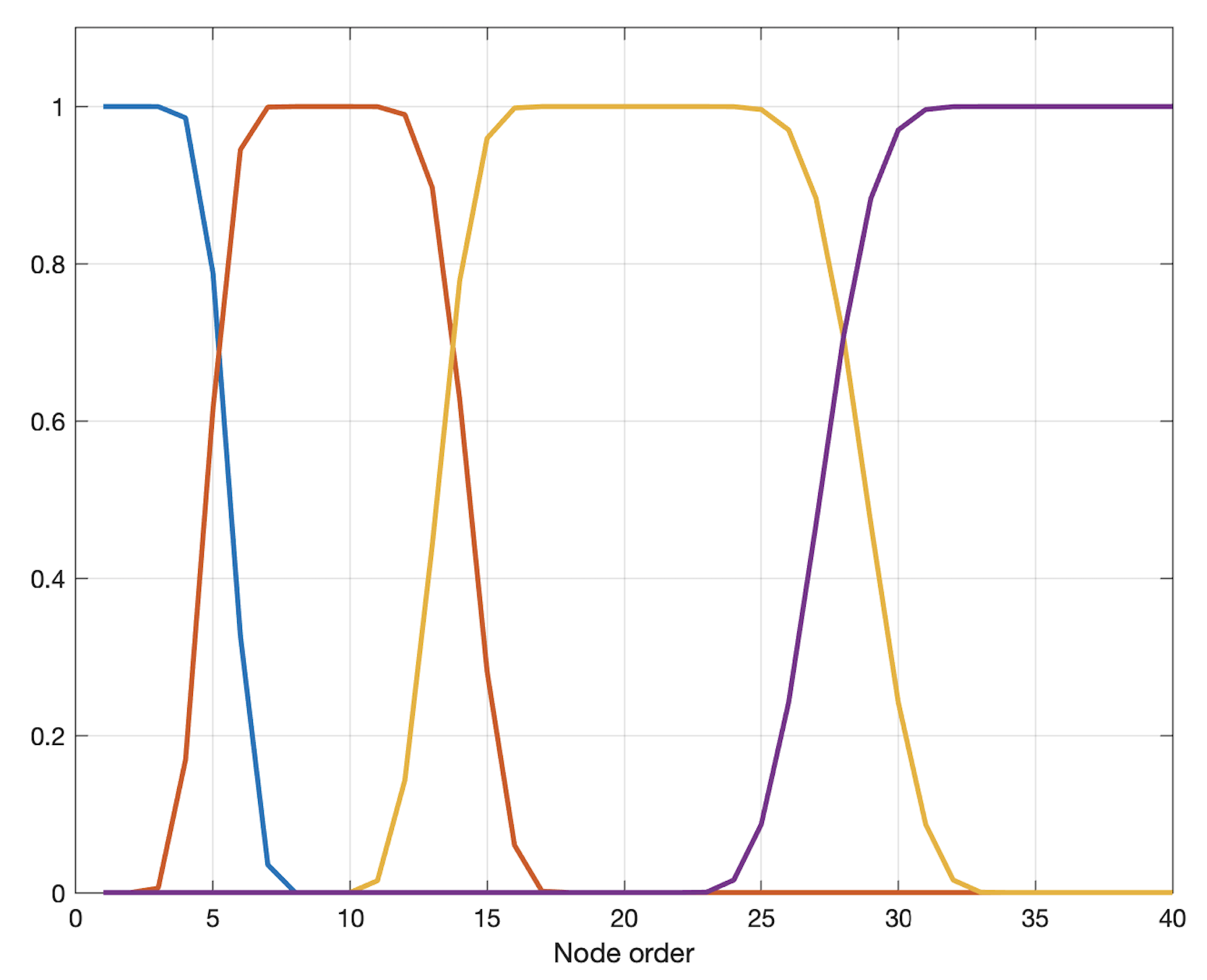}
\end{tabular}\\[2mm]
Filter bank with 2 high passes
\begin{tabular}{cccc}
  \includegraphics[width=.19\textwidth]{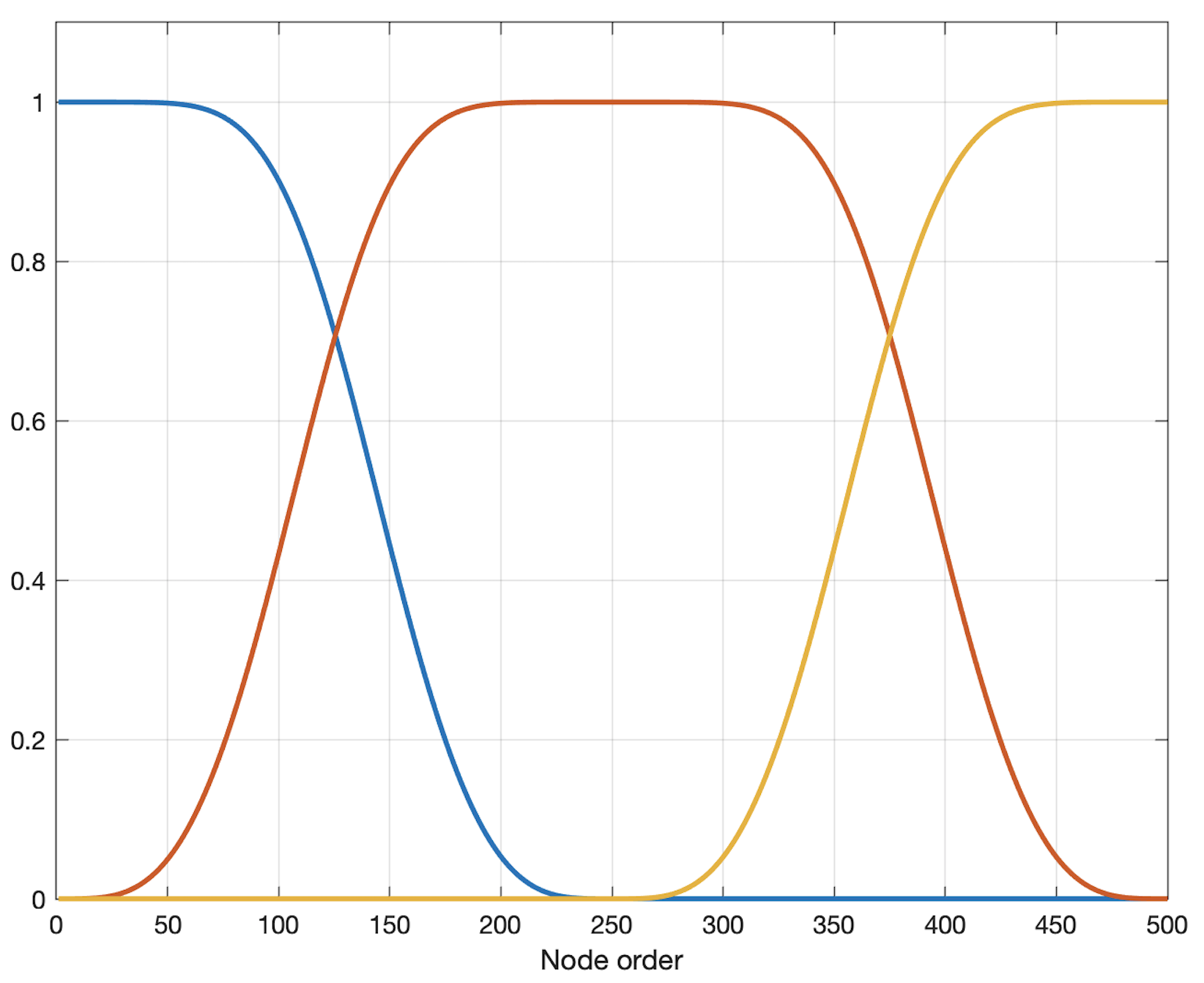}
  &\includegraphics[width=.19\textwidth]{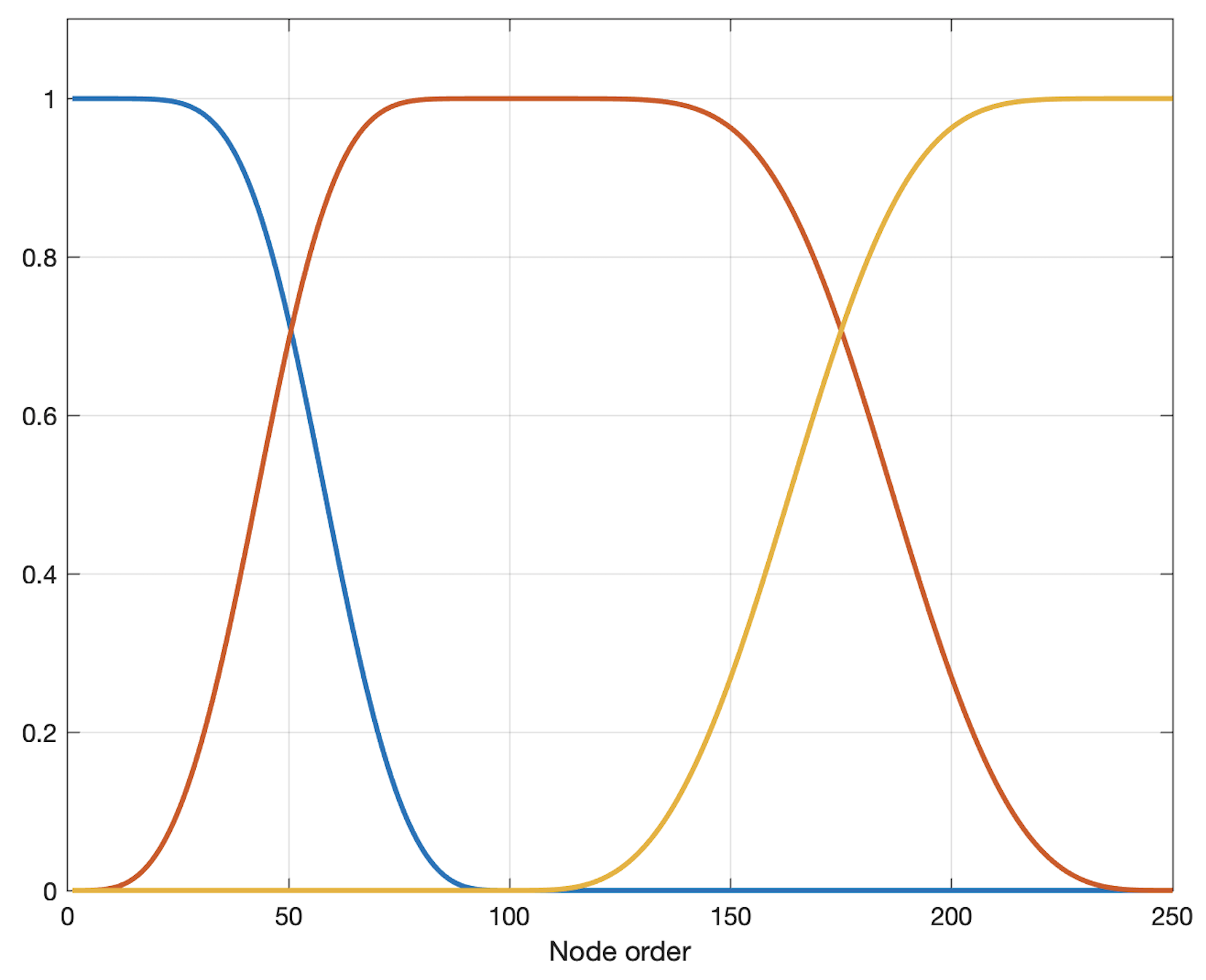}
  &\includegraphics[width=.19\textwidth]{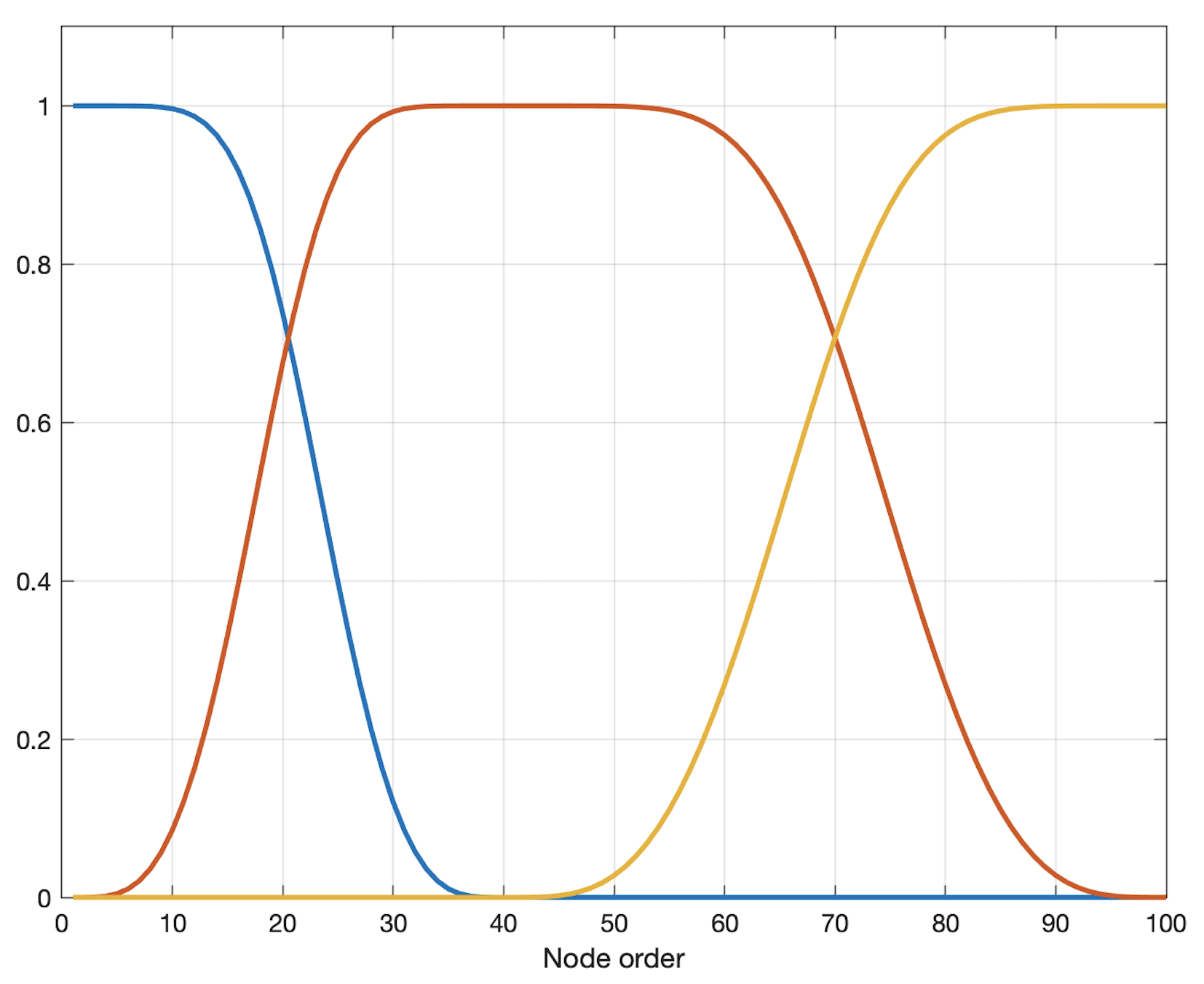}
  &\includegraphics[width=.19\textwidth]{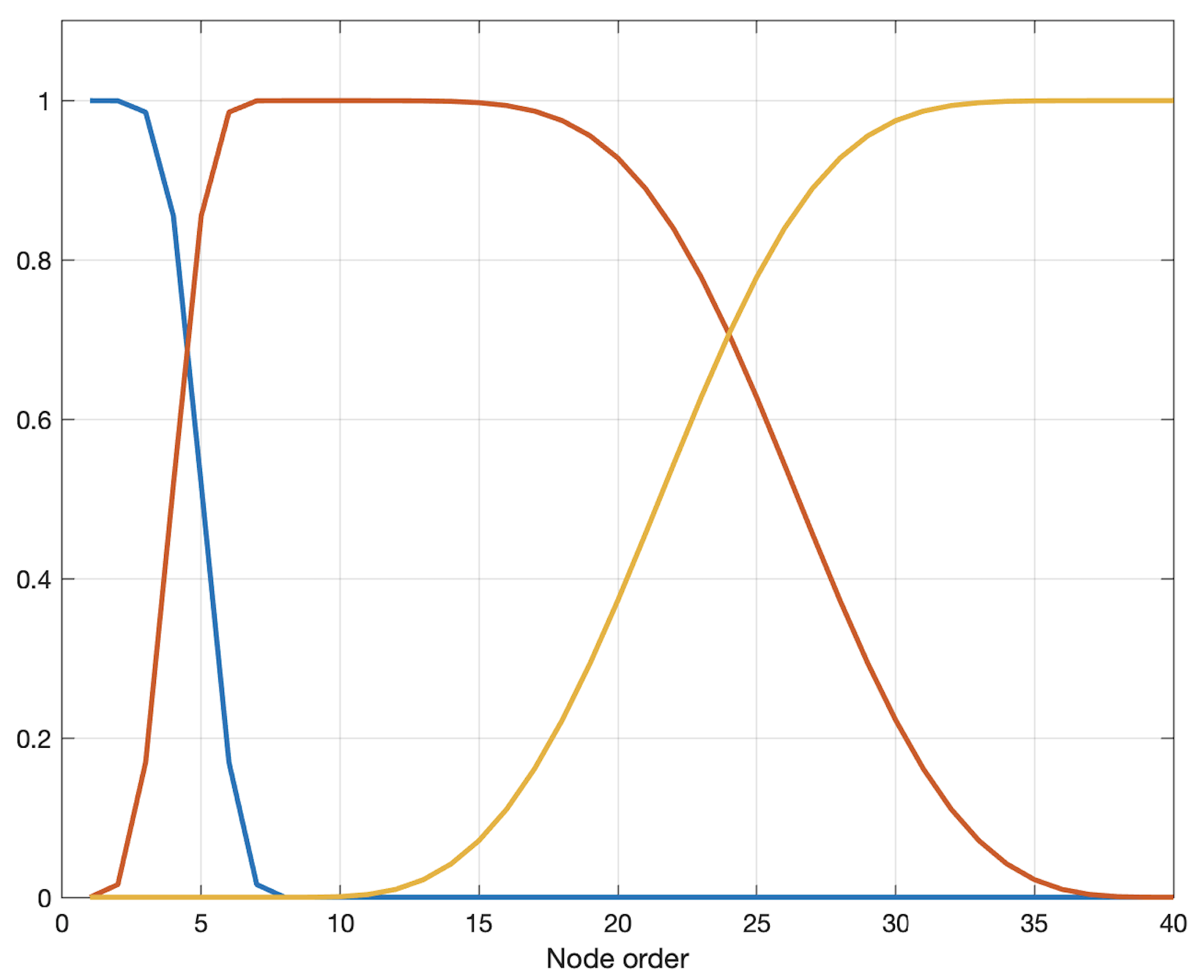}\\
  \includegraphics[width=.19\textwidth]{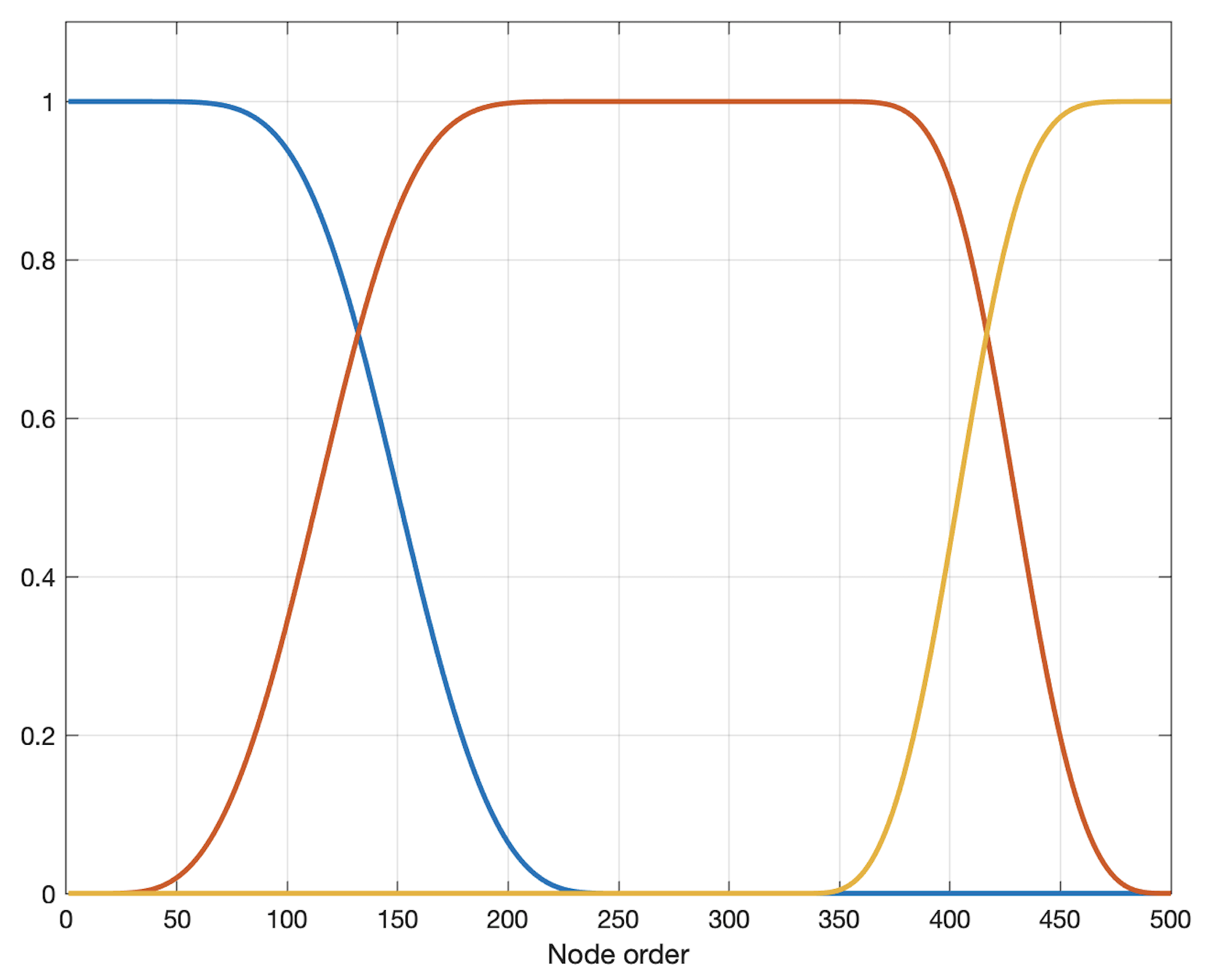}
  &\includegraphics[width=.19\textwidth]{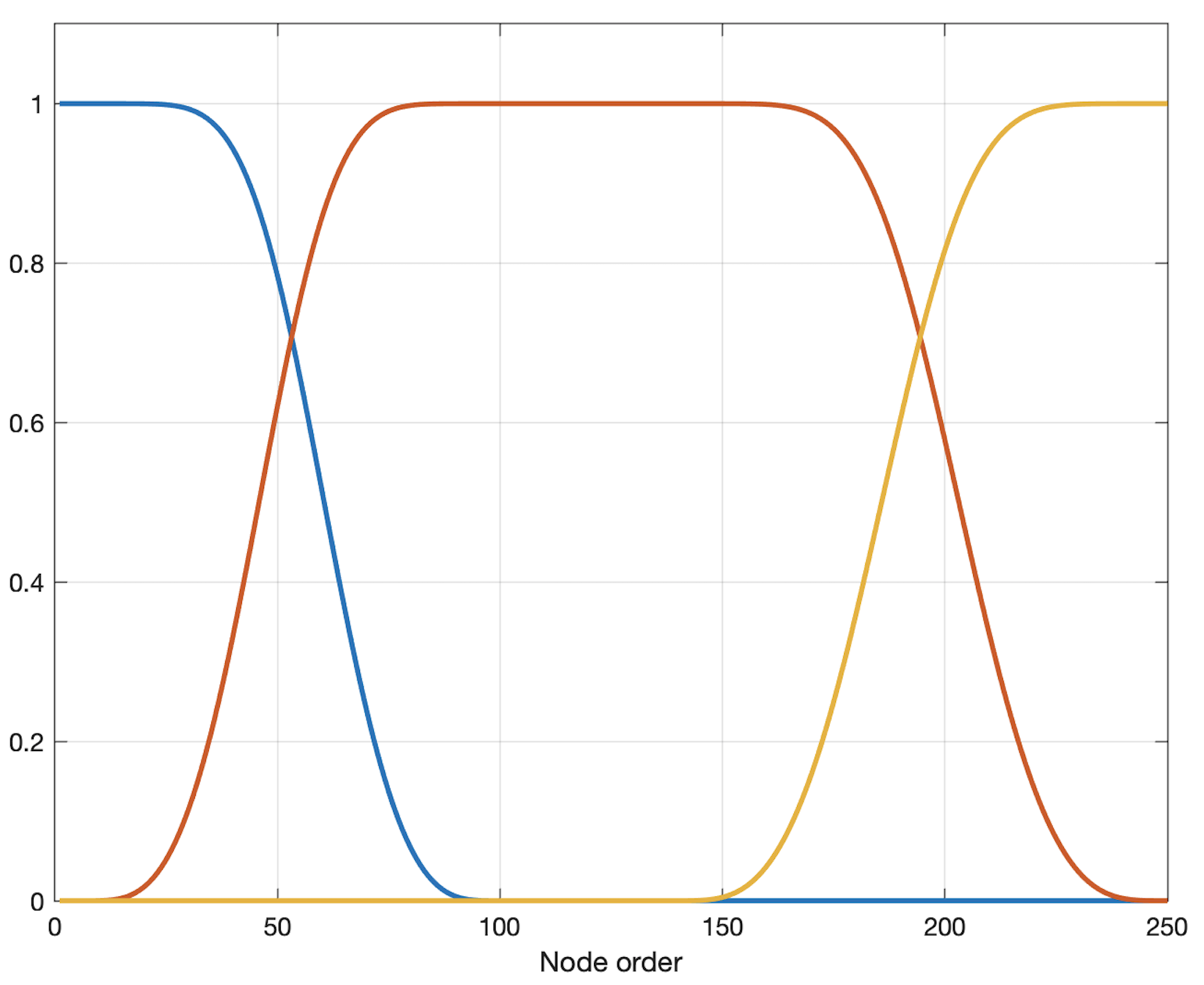}
  &\includegraphics[width=.19\textwidth]{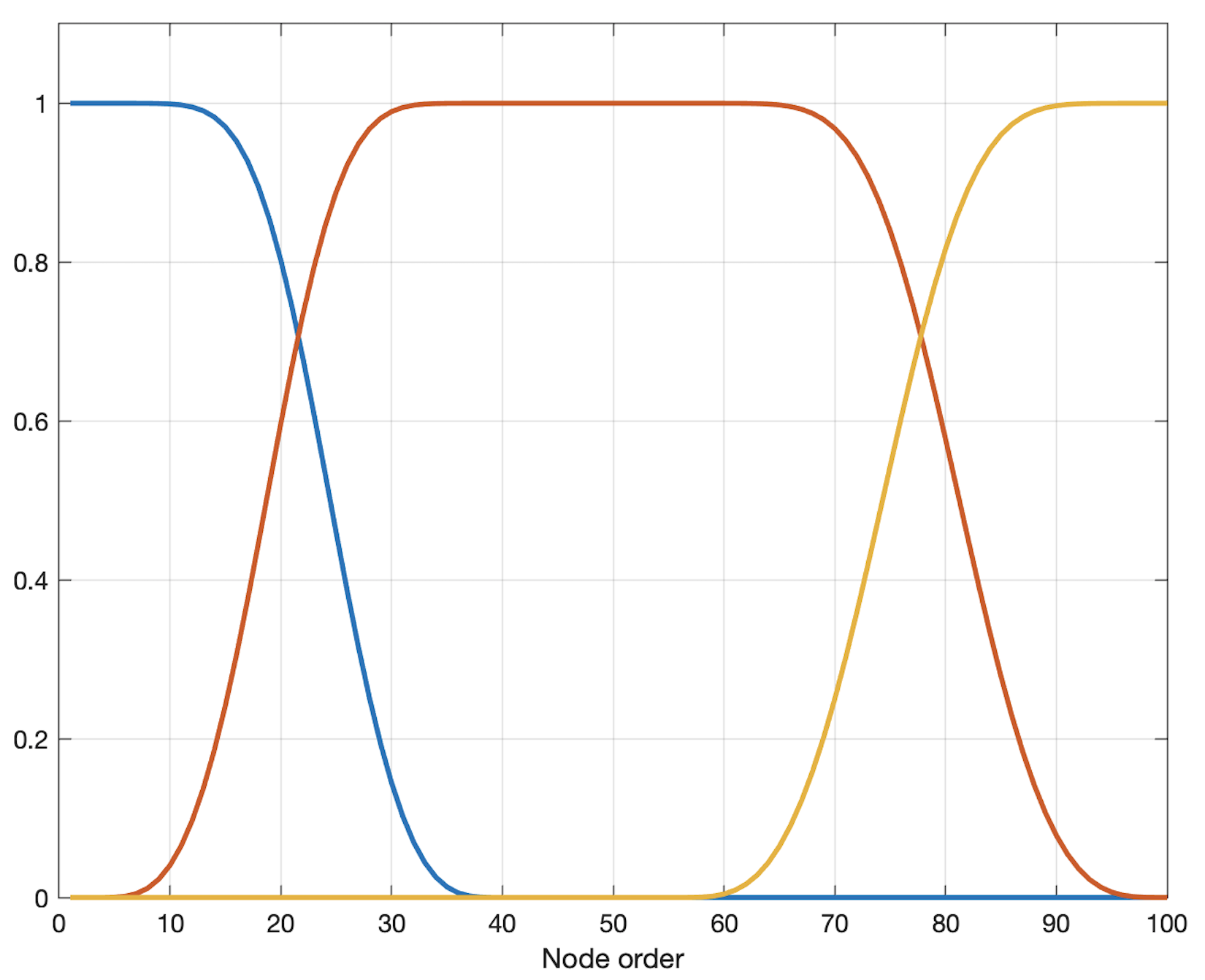}
  &\includegraphics[width=.19\textwidth]{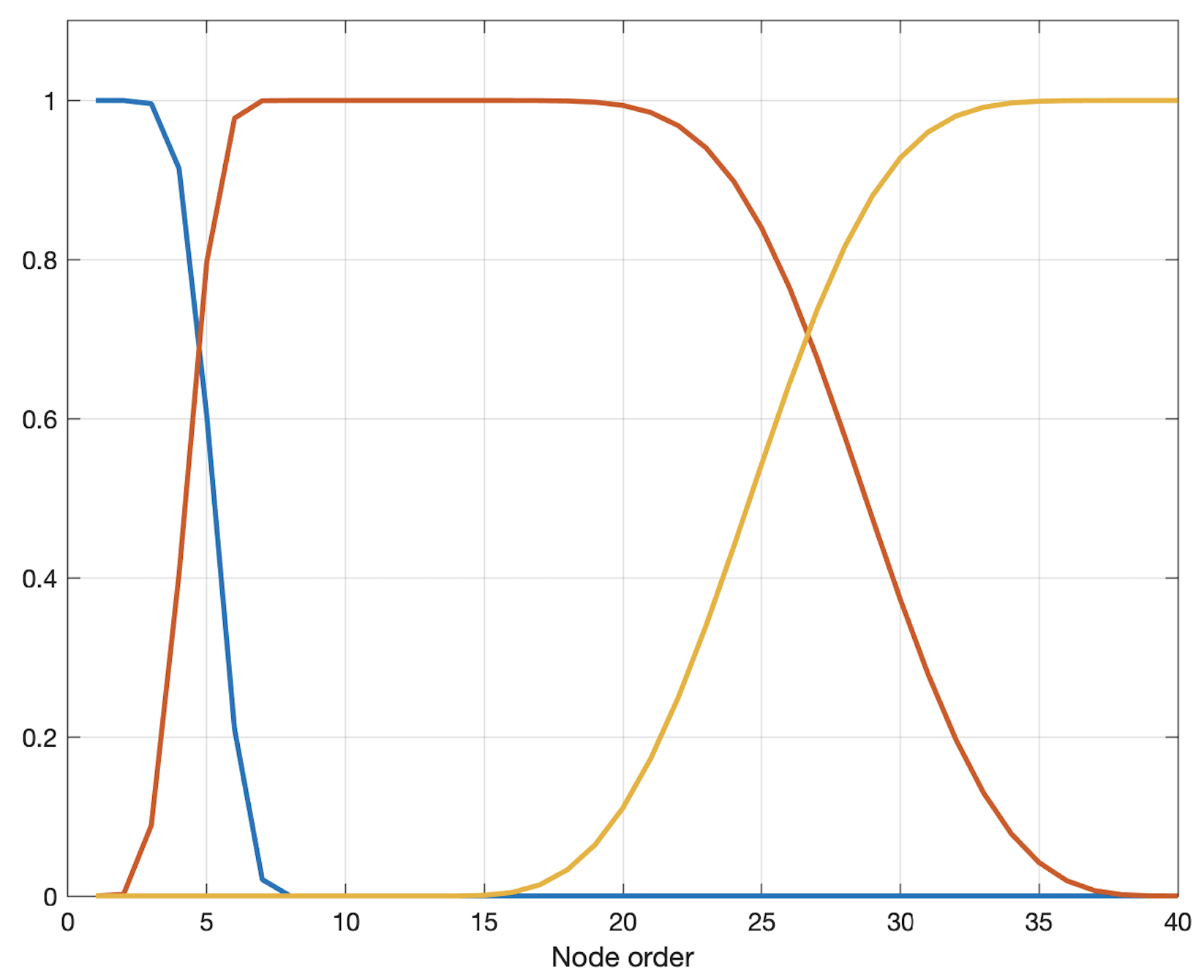}
\end{tabular}\\[2mm]
Filter bank with 1 high pass
\begin{tabular}{cccc}
  \includegraphics[width=.19\textwidth]{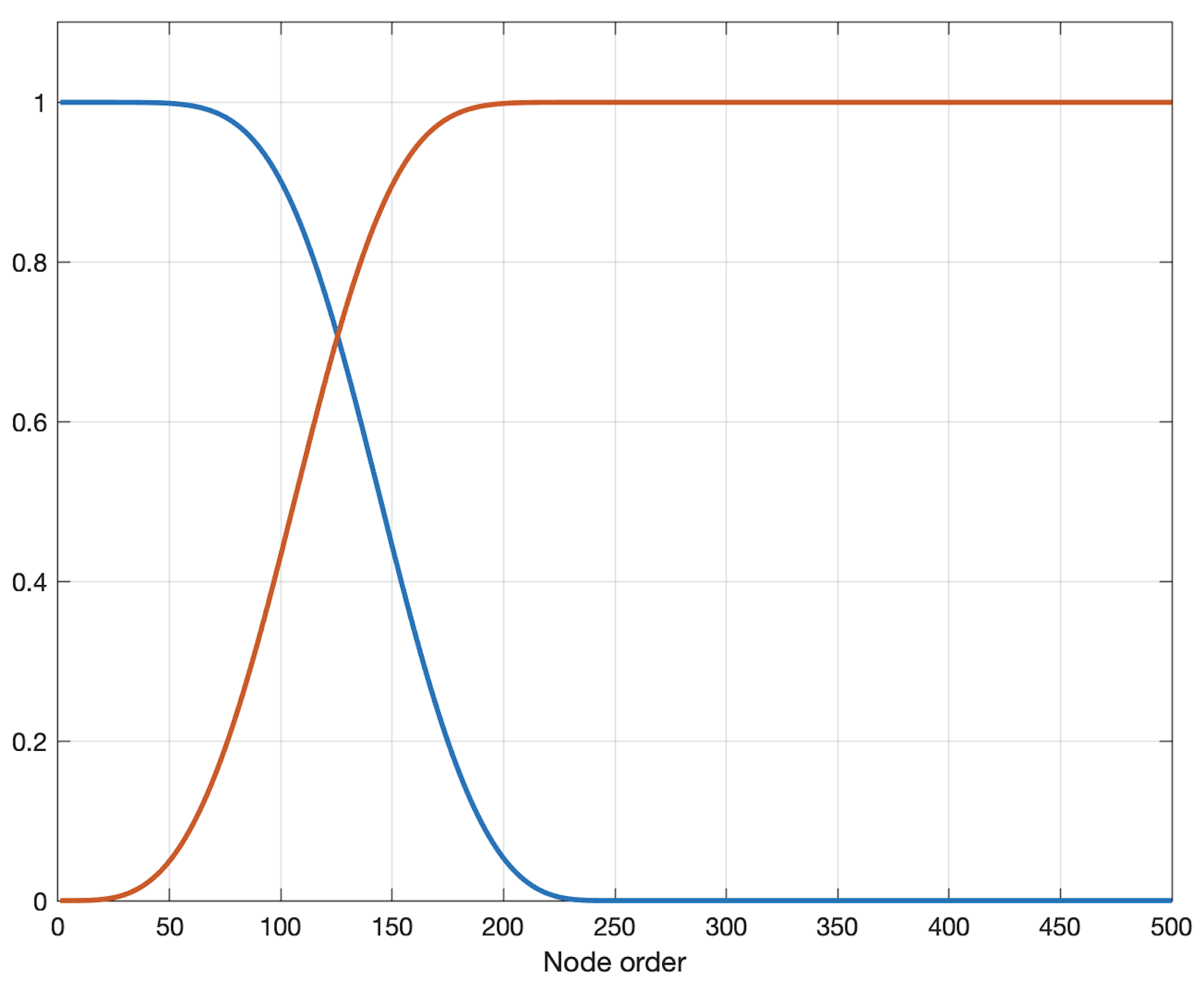}
  &\includegraphics[width=.19\textwidth]{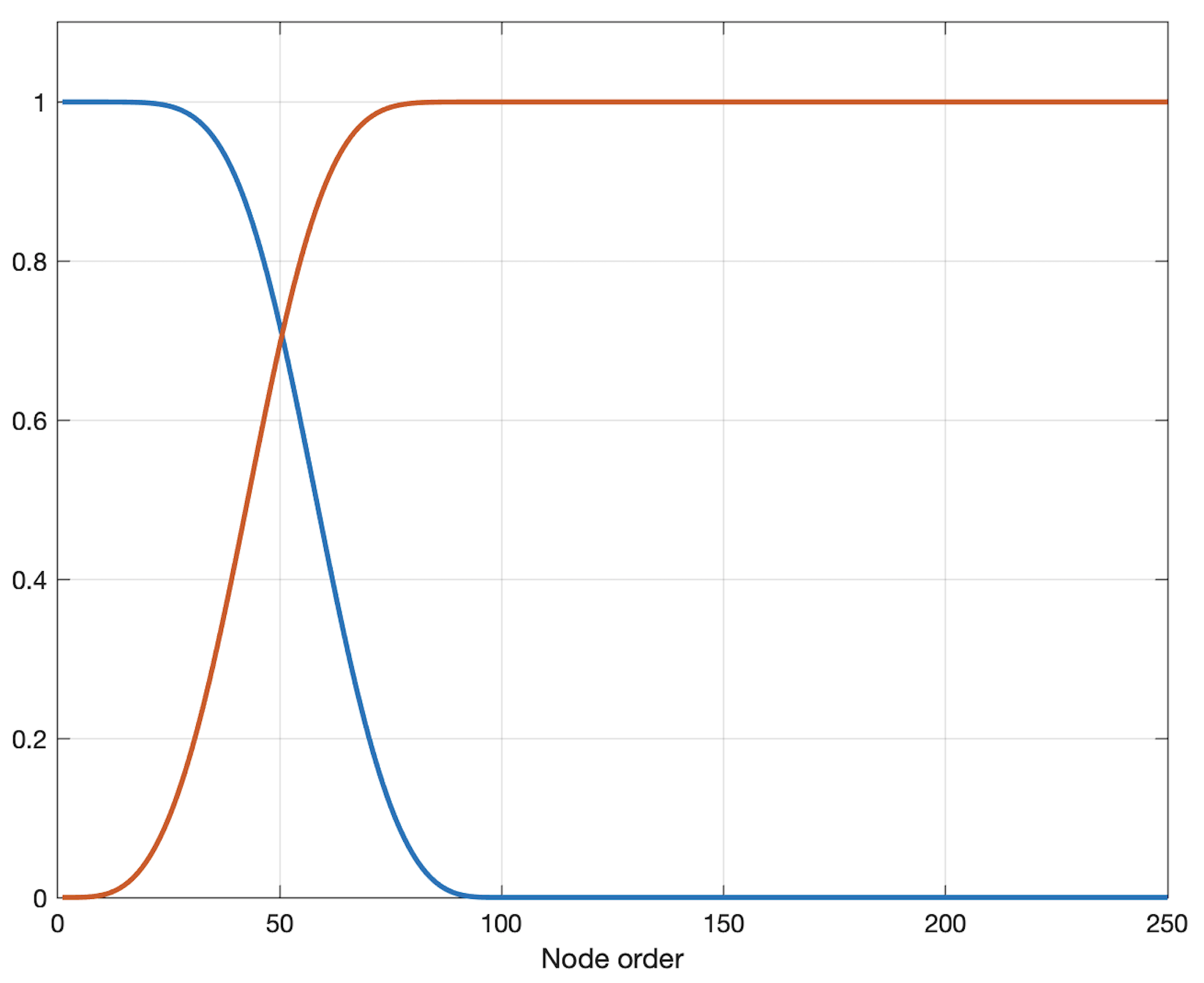}
  &\includegraphics[width=.19\textwidth]{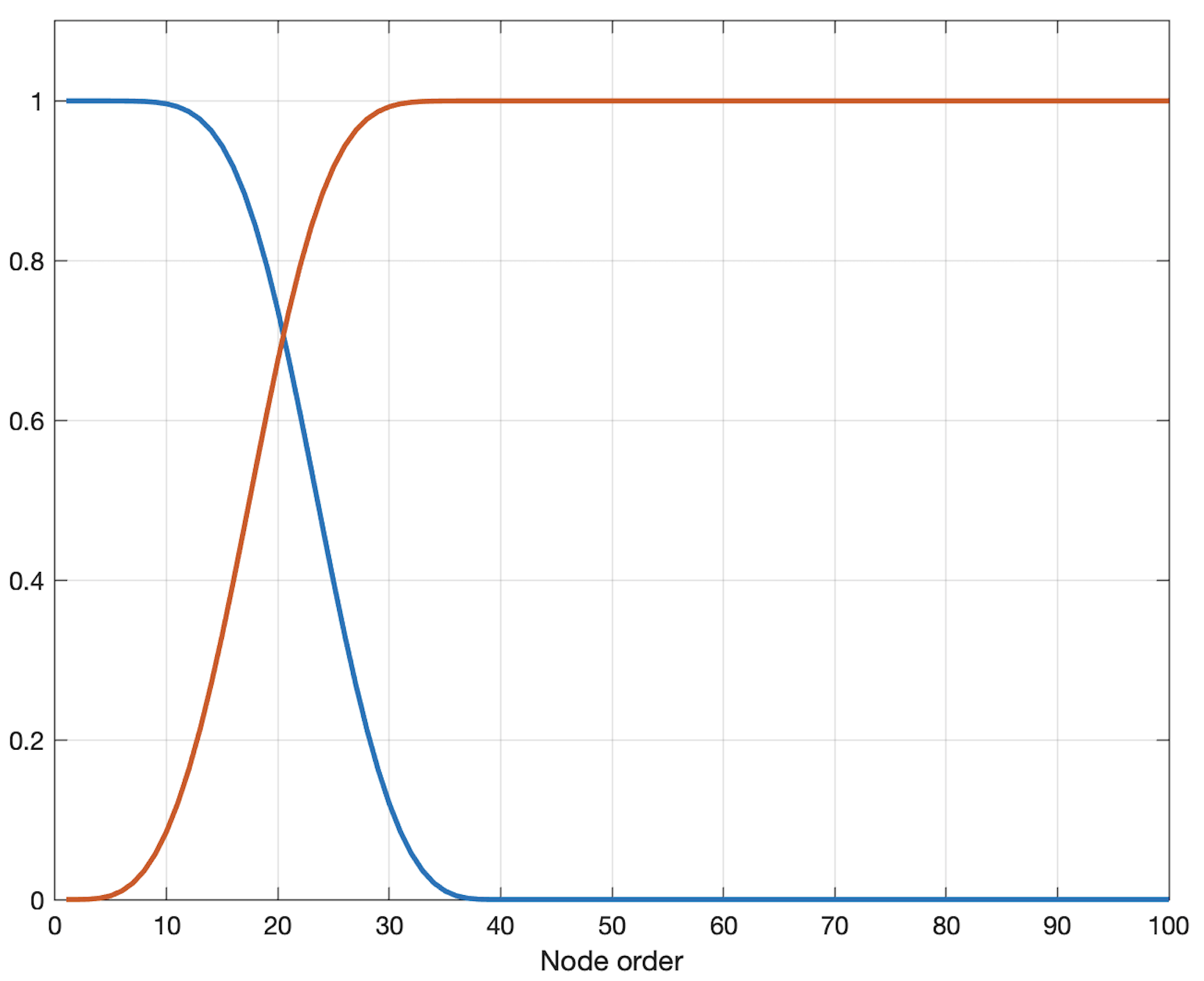}
  &\includegraphics[width=.19\textwidth]{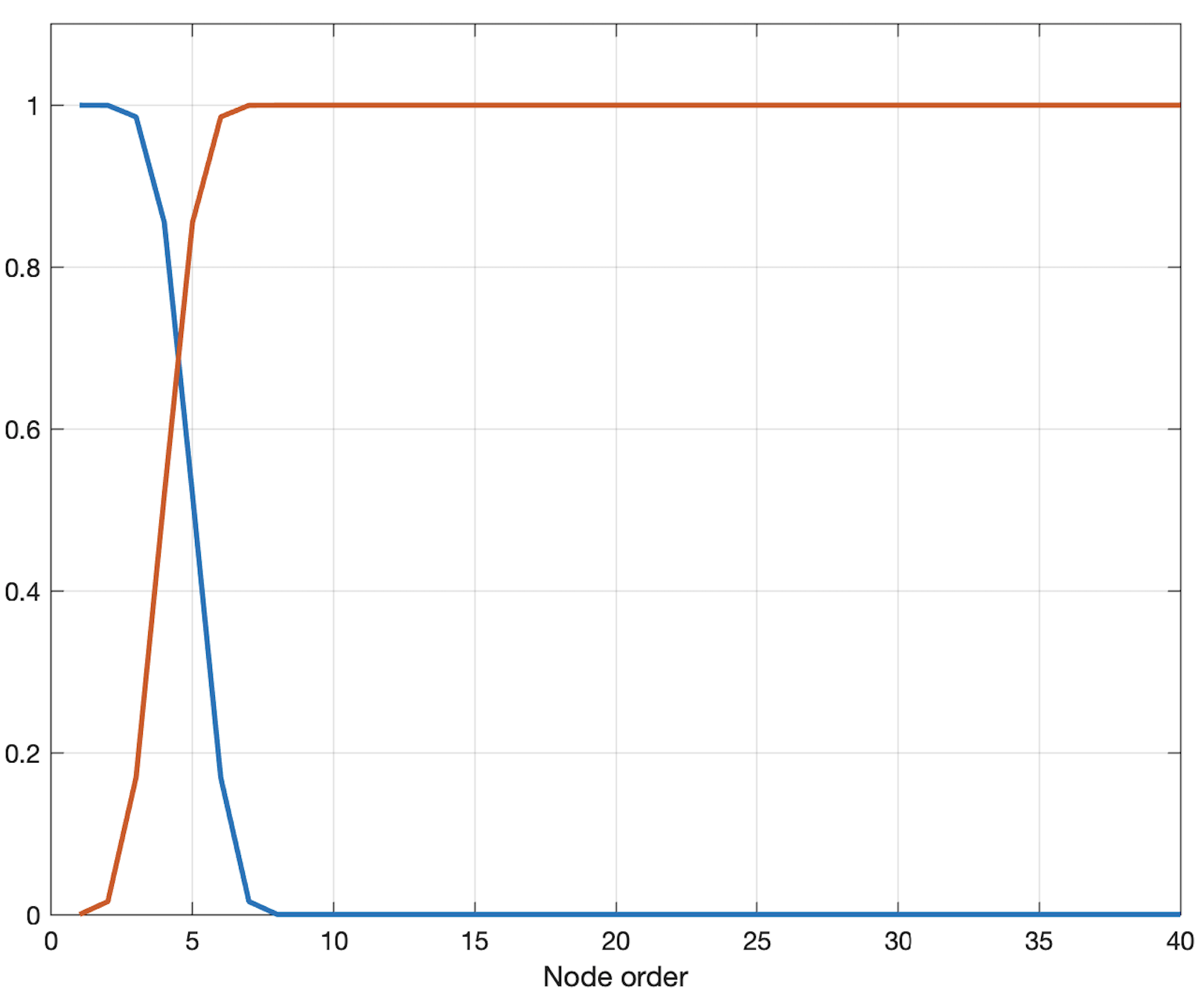}\\
  \includegraphics[width=.19\textwidth]{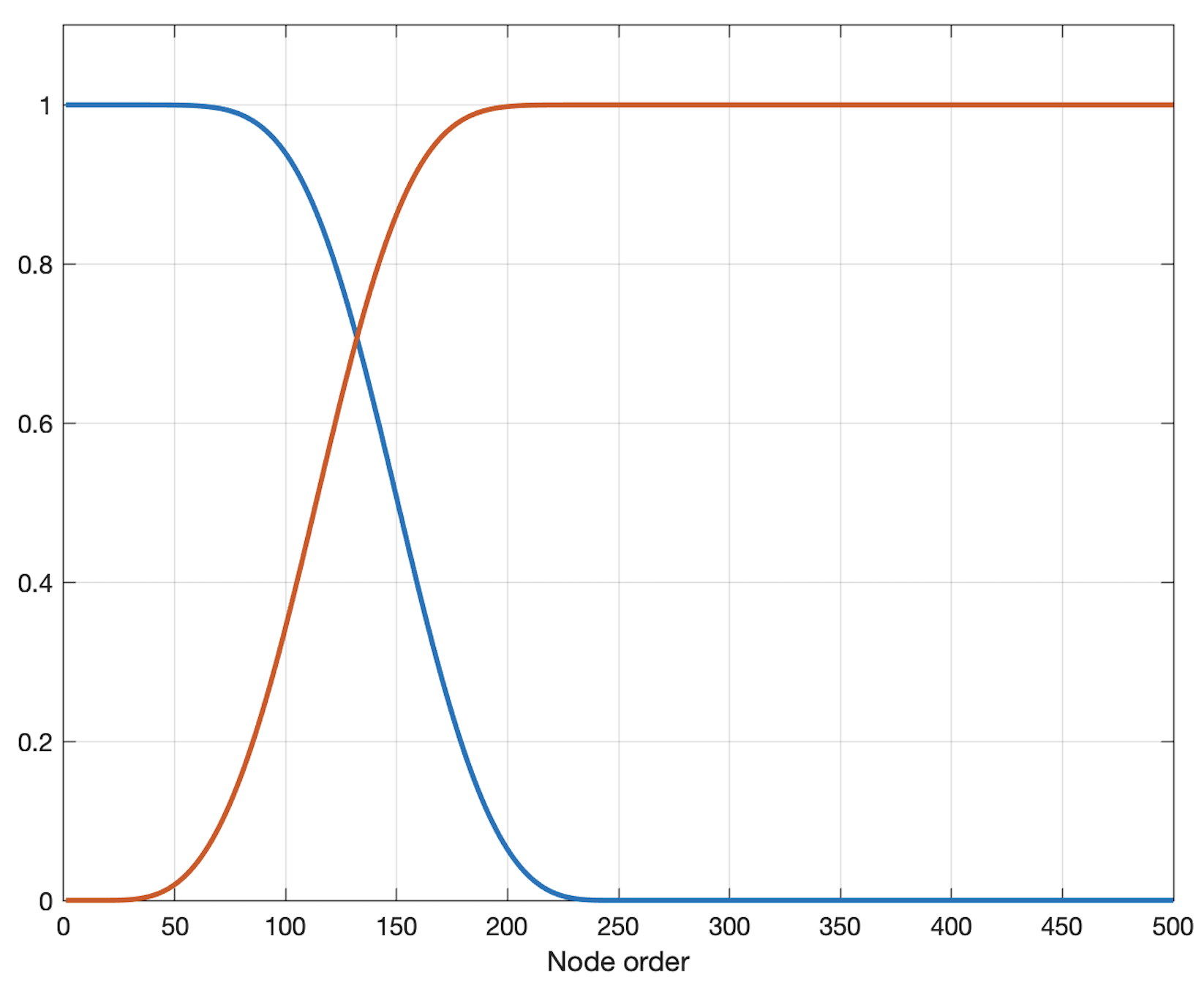}
  &\includegraphics[width=.19\textwidth]{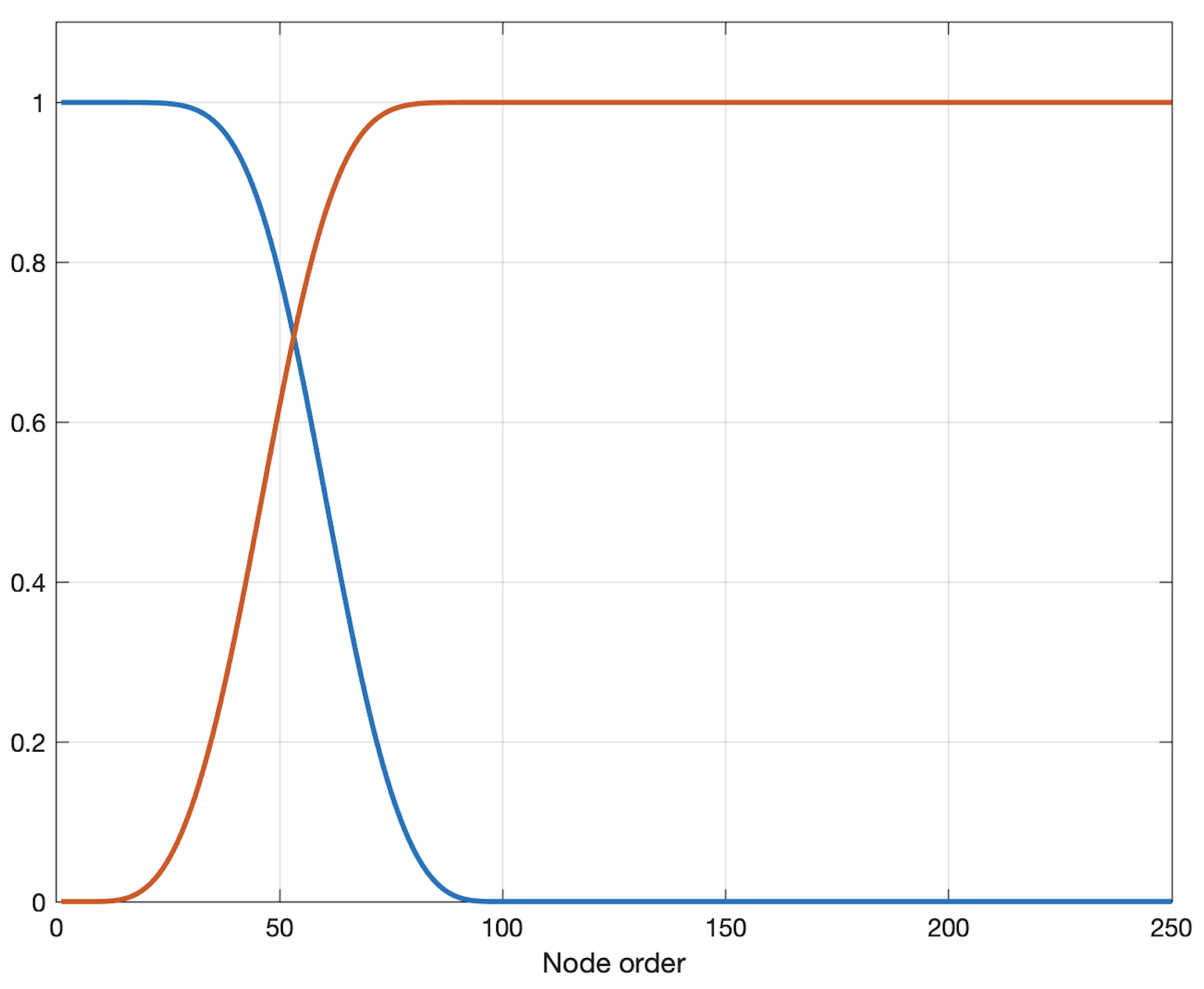}
  &\includegraphics[width=.19\textwidth]{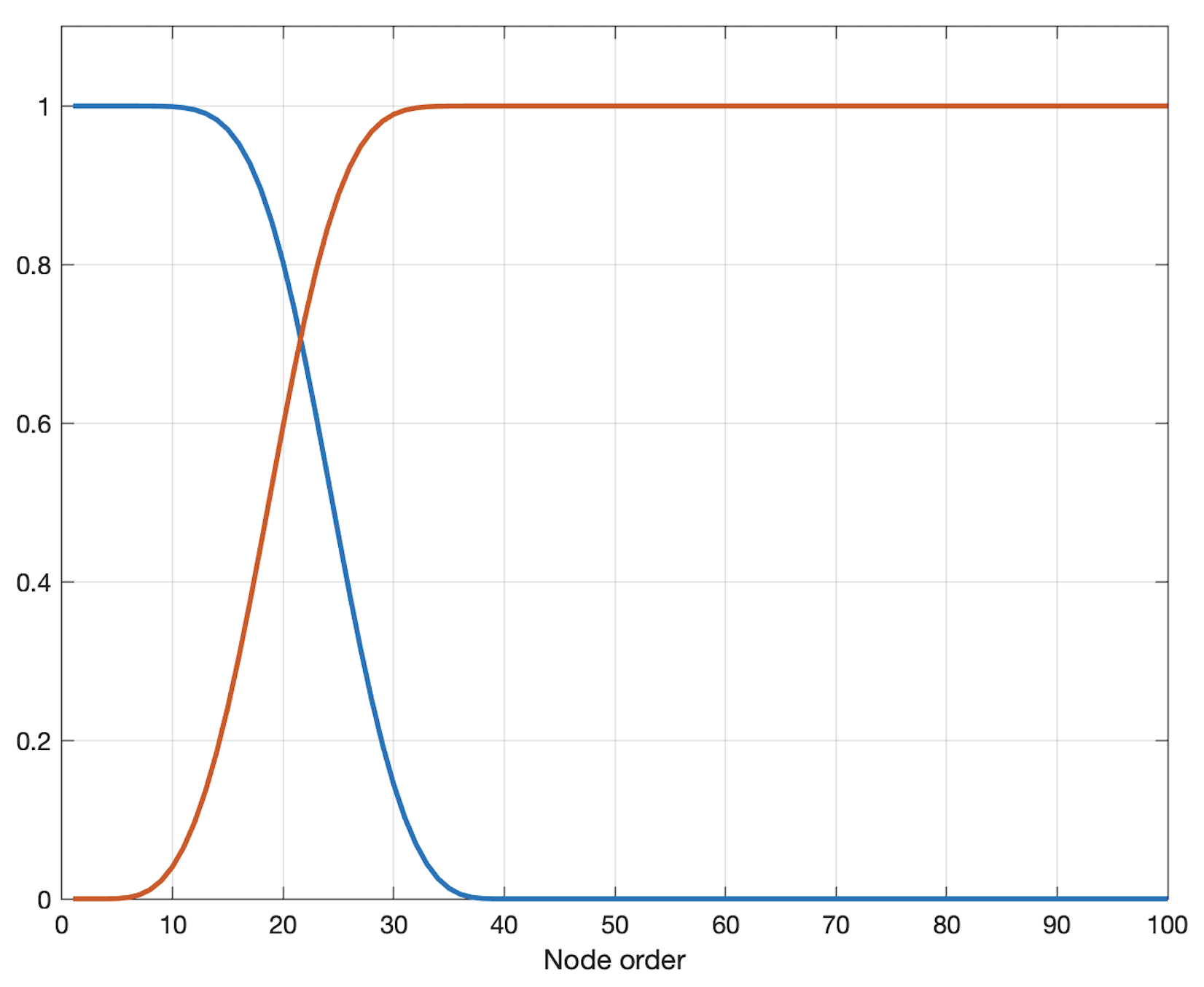}
  &\includegraphics[width=.19\textwidth]{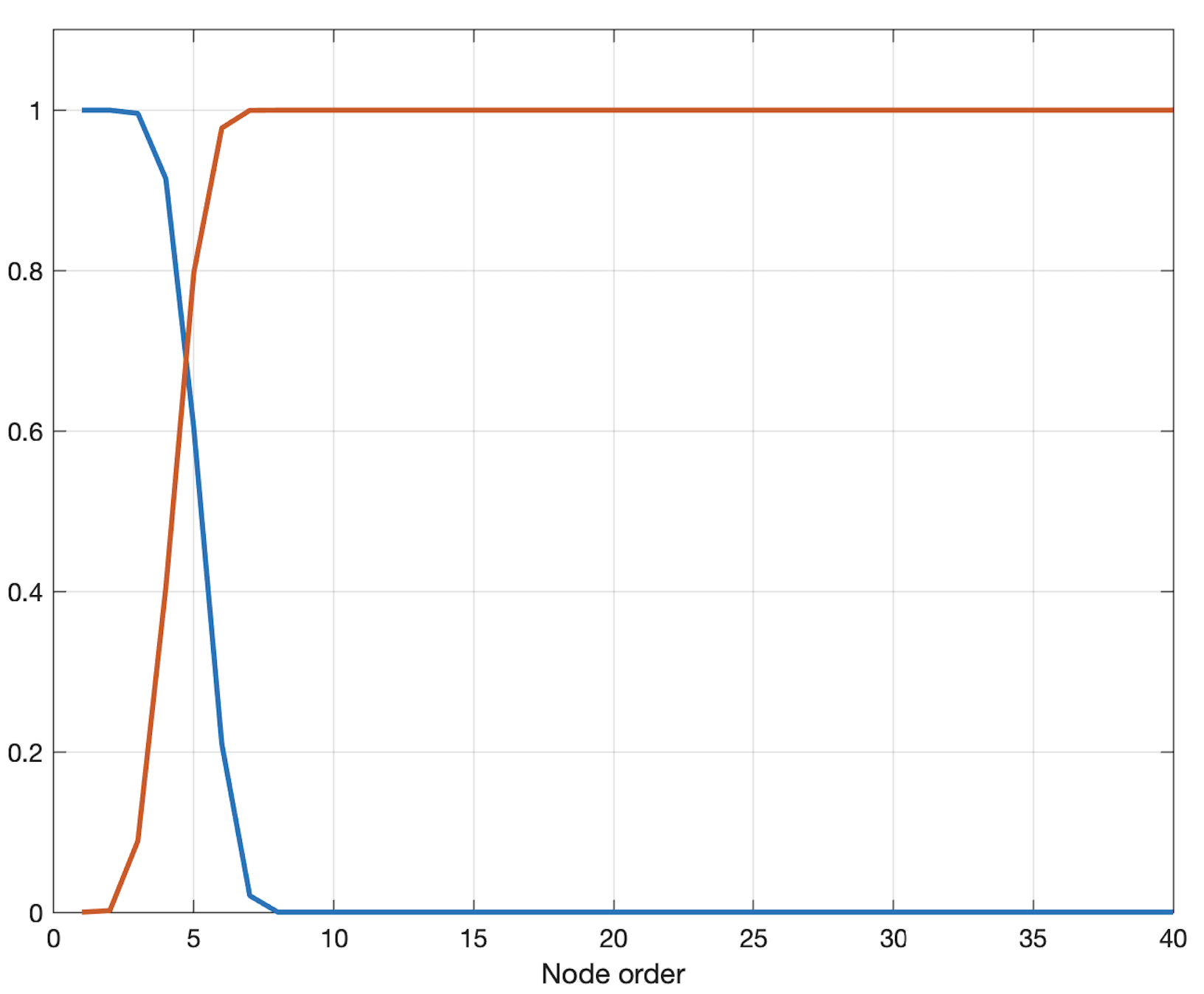}
\end{tabular}
\begin{minipage}{\textwidth}
    \caption{Filter banks with different number of high passes and different intersection points.
    For each filter bank case, from left to right are the pictures for the filter banks for the  chain level from finest to the second last coarsest. For 3-high pass case, the top and bottom use the intersection point parameters $\zeta_a=0.25, \zeta_c^{b^{(1)}}=0.25, \zeta_c^{b^{(2)}}=0.25$ and $\zeta_a=0.3, \zeta_c^{b^{(1)}}=0.2, \zeta_c^{b^{(2)}}=0.3$ respectively. For the 2-high pass case, the top and bottom use the parameters $\zeta_c^a=0.25, \zeta_c^{b}=0.25$ and $\zeta_a=0.3, \zeta_c^{b}=0.2$. For 1-high pass case,
    the top and bottom use the parameters $\zeta_c^a=0.25$ and $\zeta_c^a=0.3$.
    The graphs of the chain have 500, 250, 100, 40 and 8 from finest to coarsest. Note that we do not need a filter bank for the coarsest level.}
    \label{fig:filter_banks}
    \end{minipage}
\end{figure*}

\subsection{Computational Complexity of {\fgt} on Random Graphs}
\label{exp1}
In the first experiment, we show the computational cost of the Fast Fourier transforms (DFT [Algorithm~\ref{alg:DFT:G}] \& ADFT [Algorithm~\ref{alg:ADFT:G}]) and the fast framelet transforms (Decomposition [Algorithm~\ref{alg:decomp.multi.level}] \& Reconstruction [Algorithm~\ref{alg:reconstr.multi.level}] for {\fgt}). We demonstrate the computational complexities by timing the elapsed time for each of these algorithms on a set of randomly generated graphs with different sizes ranging from $500$ nodes to $30,000$ nodes. We simulate each adjacency matrix $\adjG\sim\mathsf{U}_{(0, 1)}$ with $\adjG_{ij}=\adjG_{ji}$ if $i\neq j$ and $\adjG_{ij}=0$ if $i=j$ for $i, j=1,2,\ldots,N$, where $N$ denotes the number of nodes of the graph. 
This implies that the simulated graph does not contain any isolated node. We use the uninformative feature (random scalar) for each node of all the randomly generated graphs. We build the coarse-grained chains with $6$ levels for the graphs larger than $2,500$ nodes; $5$ levels for the graphs with sizes between $1,000$ and $2,500$ nodes; $4$ levels for the graphs smaller than $1,000$ nodes. The choice of the clustering method and the resulting chain will have some impact on the performance of the GNNs with framelet convolution. Here, as an empirical study, we drop roughly $60\%$ of the nodes in every coarsening. All the coarse-grained chains used in Sections~\ref{exp1} \&~\ref{exp2} are generated by METIS \citep{karypis1998fast}, which is a computationally efficient graph partitioning method. Algorithm~\ref{algo:haar_basis} is applied to construct the Haar-like orthonormal basis for each coarse-grained chain for DFTs and {\fgt}s.

The computing environment of this experiment is MATLAB\textsuperscript{\textregistered} R2019b installed on a macOS Catalina machine with 2.3GHz 8-Core Intel Core i9 processor and 16GB RAM. Figure~\ref{fig:exp1} shows for both DFTs and {\fgt}s, the computational time is approximately proportional to $N$. This observation verifies our theoretical analysis in Section~\ref{sec:fmtG} that {\fgt}s have a computational cost $\bigo{}{N}$ for the graph with size $N$.

\subsection{Decimated Framelets on Road Network}
\label{exp2}
In this section, we present an application of multiscale analysis by {\fgt}s for a real-world traffic network of Minnesota \citep{nr}. The dataset represents the roads of Minnesota by edges of a graph and the intersections and towns by $2,642$ graph vertices. In this experiment, the graph $\gph$ is unweighted, which means all the edge weights are equal to unity regardless of the length of the road segment. The spatial coordinates of each node are only used for visualization purpose but do not affect the input graph data of {\fgt}s. We use the uninformative feature (constant scalar) for the nodes of the graph. The road network is visualized in Figure~\ref{fig:Minnesota.graph}(a).

We first employ Algorithm~\ref{alg:decomp.multi.level} for {\fgt} to decompose the graph signal into a set of framelet detail and approximation coefficients and visualize the significance and distribution of these coefficients (here the values are in the spectral domain) on the road network. The framelets use the filter banks, each of which has one low pass and two high passes. At the $j$th level, the lengths of the coefficient vectors for the low-pass and high-pass are equal to the number of nodes for the $j$th and $(j+1)$th levels of the chain respectively.

We construct a four-level coarse-grained chain $\gph_{3\to 0}:=(\gph_3, \gph_{2}, \gph_1, \gph_0)$ with $\gph_3=\gph$, $|V_{2}|=1,000$, $|V_1|=300$ and $|V_0|=100$ by using the same graph partitioning method as above. Algorithm~\ref{algo:haar_basis} is applied to calculate the Haar-like orthonormal basis and the computing environment is same as the previous experiment. Figure~\ref{fig:Minnesota.graph} visualizes the experimental results.

\begin{remark}
In practice, the numerical approximation error of eigendecomposition for the graph Laplacian is dependent on the structure and sparsity of the graph, the clustering for the chain and the size of the graph.
\end{remark}

\begin{figure}[th!]
\begin{minipage}{\textwidth}
\centering
\includegraphics[width=.95\linewidth]{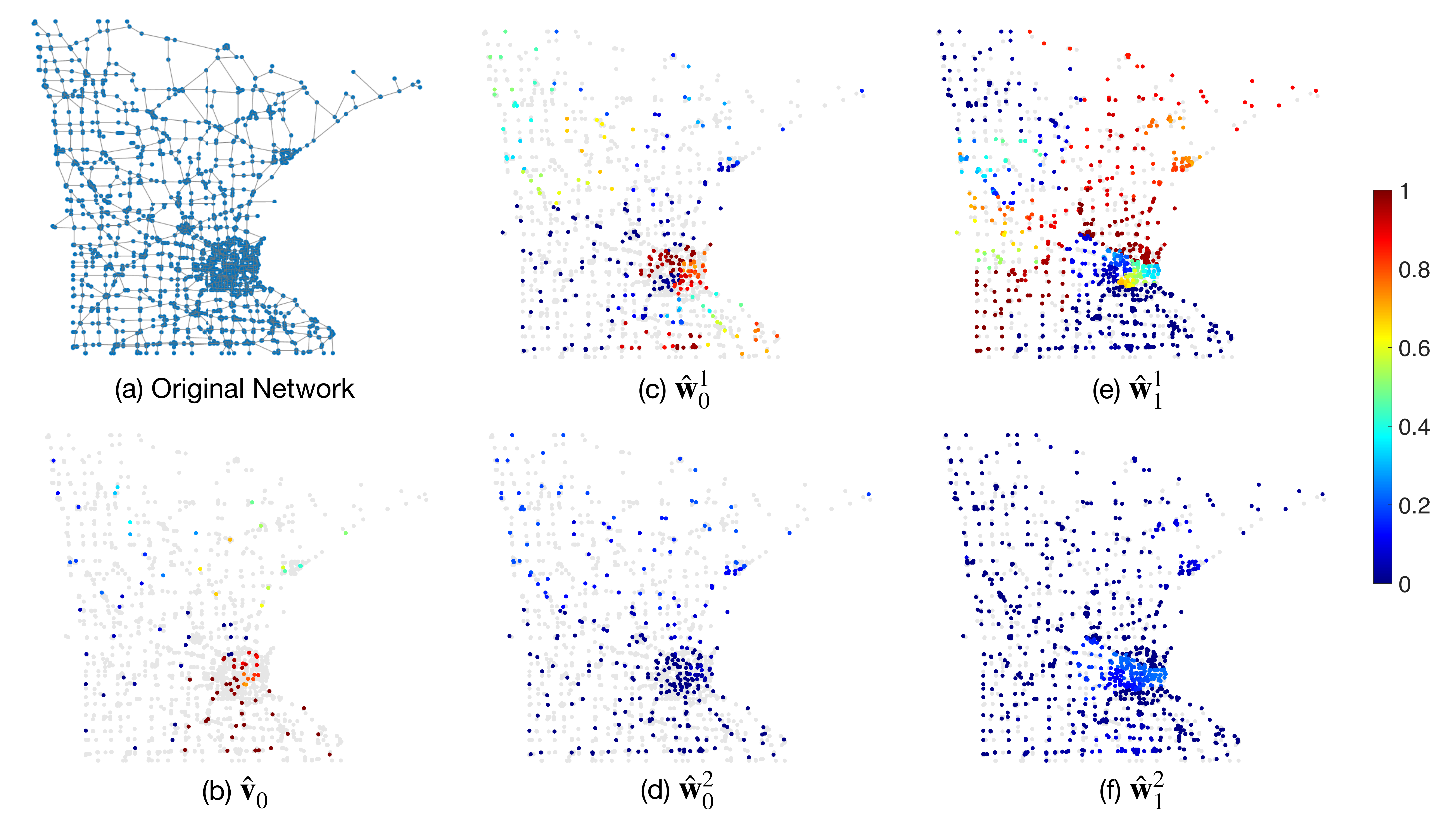}
\caption{
(a) Graph of Minnesota road network $\gph_{3}$. (b) Framelet approximation coefficients $\dfcav[0]$ on $\gph_0$ with $100$ clusters. (c) Framelet detail coefficients $\dfcbv[0]{1}$ on $\gph_{1}$ with $300$ clusters. (d) Framelet detail coefficients $\dfcbv[0]{2}$ on $\gph_{1}$ with $300$ clusters. (e) Framelet detail coefficients $\dfcbv[1]{1}$ on $\gph_{2}$ with $1,000$ clusters. (f) Framelet detail coefficients $\dfcbv[1]{2}$ on $\gph_{2}$ with $1,000$ clusters. For subfigures (b)--(f), the dots in pale grey represent the background (which is the original graph) which outlines the geographical structure of the road network and exhibits the spatial and distributional information of the framelet coefficients (colored dots). Each colored dot corresponds to the first node of the cluster $[p]_{\gph_i}$, for $p\in V_i$ and $i=0, 1, 2$, with different colors depicting the values of the framelet coefficients. Note that all the displayed framelet coefficients are in the spectral domain.}
\label{fig:Minnesota.graph}
\end{minipage}
\end{figure}

The framelet approximation coefficients $\dfcav[0]$ correspond to the low passes of the filter banks, which capture the global information of the input graph. In Figure~\ref{fig:Minnesota.graph}(b), we observe that the elements of $\dfcav[0]$ with more significant values are located over both high-density area (lower part of the road network) and low-density area (upper part of the road network). Thus, the framelet approximation coefficients $\dfcav[0]$ can represent the global structure of the input graph data. On the other hand, the framelet detail coefficients $\{\dfcbv[j]{n}:n=1,\ldots,r_j, j=J_0,\ldots, J-1\}$ are related to the high passes of the filter banks, which capture the information from the local regions of the input graph. From Figures~\ref{fig:Minnesota.graph}(c)--(f), we can see that the framelet detail coefficients with more significant values are mostly spread over the specific subregions. For example, the coefficients are concentrated in the lower left and upper right parts of the road network in Figure~\ref{fig:Minnesota.graph}(e), and the lower right and upper left parts of the network in Figure~\ref{fig:Minnesota.graph}(c). Note that the subregions with more significant framelet detail coefficients in Figures~\ref{fig:Minnesota.graph}(c) and (e) are complementary, which reflects different parts of details the two high-passes captured from the graph signal. The second sets of the framelet detail coefficients $\dfcbv[0]{2}$ and $\dfcbv[1]{2}$ visualized in Figures~\ref{fig:Minnesota.graph}(d) and (f) represent even more detailed information of their first high-passes $\dfcbv[0]{1}$ and $\dfcbv[1]{1}$. Hence, these framelet detail coefficients enable to capture the detail information from all the subregions of the input graph $\gph$.

\subsection{Graph Classification with {\fgt}-based Graph Convolution}
\label{sec:fgconv}
Graph-level classification task relies on the graph representation learning which has a broad range of real-world applications, such as social network analysis \citep{hamilton2017inductive, veli2018graph} and molecule classification \citep{duvenaud2015convolutional, gilmer2017neural}. Graph convolution is one of the key components in a graph convolutional neural network, which has been proved effective and powerful for learning a graph representation \citep{ying2018hierarchical, ma2019graph, wang2019haar}. In this section, we explore the feasibility of using {\fgt} with Haar-like orthonormal basis (Algorithm~\ref{algo:haar_basis}) to define a spectral graph convolution layer based on the classic work \citep{bruna2013spectral} which is the first attempt at implementing CNNs on graphs in the spectral domain with graph Fourier transforms.

In $\Rd[d]$, convolution induced by Fourier transforms or wavelet transforms are well known, see for example, \cite{stein2011fourier,Mallat2009}. For $g,f\in l_2(\gph)$, we define \emph{\textbf{F}ramelet \textbf{G}raph \textbf{Conv}olution} (\textsc{FGConv}) as
\begin{equation}
\label{eq:fgtconv}
    g\star f = \synOp\bigl((\analOp g)\odot (\analOp f)\bigr),
\end{equation}
where $g$ is a trainable filter; $f$ denotes the graph data with one feature on each node; the symbol $\odot$ is the Hadamard product; and the operators $\synOp$ and $\analOp$ are the framelet decomposition (Algorithm~\ref{alg:decomp.multi.level}) and reconstruction (Algorithm~\ref{alg:reconstr.multi.level}) of {\fgt}, respectively. By presuming that the trainable filter $g$ lies in the spectral domain and has an identical shape as $\analOp f$, the convolution~\eqref{eq:fgtconv} can be simplified as
\[
g\star f = \synOp\bigl(g\odot (\analOp f)\bigr).
\]
Here we use $\dfcav[j]$ and $\dfcbv[j]{n}$ in Algorithms~\ref{alg:decomp.multi.level} and \ref{alg:reconstr.multi.level} to simplify computation. Similar graph wavelet convolution was developed by \cite{xu2019graph}, where the wavelets \cite{HaVaGr2011} are undecimated framelets which do not include downsampling and upsampling processes.

Computationally, the \textsc{FGConv} is performed by first decomposing the graph signal $f$ into several sets of framelet detail and approximation coefficients, then taking the Hadamard product between the trainable filter $g$ and the coefficients $\analOp f$ in the spectral domain, and finally reconstructing the output graph data from the processed framelet coefficients $g\odot (\analOp f)$. By applying the weight detaching trick \citep{xu2019graph, LI2020188, zheng2020mathnet}, the \textsc{FGConv} with multiple input features $d$ and reduced parameter complexity reads
\begin{equation}
\label{FGCov}
    \mathcal{F}^{\textnormal{out}} = \sigma\bigl(\synOp(G(\analOp (\mathcal{F}^{\textnormal{in}}\mathcal{W})))\bigr),
\end{equation}
where $\mathcal{F}^{\textnormal{in}}\in\R^{|V|\times d^{\textnormal{in}}}$ and $\mathcal{F}^{\textnormal{out}}\in\R^{|V|\times d^{\textnormal{out}}}$ denote the input and embedded feature matrices of the graph with the indicated shapes; $\mathcal{W}\in\R^{d^{\textnormal{in}}\times d^{\textnormal{out}}}$ is a trainable weight matrix for affine transformation; and $\sigma$ is the activation function (e.g., ReLU). For an example graph $\gph$ with feature matrix $\mathcal{F}^{\textnormal{in}}$ and coarse-grained chain $\gph_{3\to 0}:=(\gph_3, \gph_{2}, \gph_1, \gph_0)$ where $\gph_3=\gph$, we concatenate all the framelet detail and approximation coefficients from $\analOp(\mathcal{F}^{\textnormal{in}}\mathcal{W})$ into a matrix with shape $K\times d^{\textnormal{out}}$ for the ease of computation, where $K=|V_0|+r\sum_{i=1}^3 |V_i|$ with $r$ representing the number of high-pass filters. Hence, the matrix $G$ in \eqref{FGCov} denotes a trainable diagonal matrix with shape $K\times K$ and the parameter complexity of \textsc{FGConv} in \eqref{FGCov} is $\bigO(K+d^{\textnormal{in}}\times d^{\textnormal{out}})$.

\begin{table*}[th!]
\caption{Statistical information of the datasets used for graph classification}
\label{partial_statistics}
\begin{center}
\begin{tabular}{l c c c}
\toprule
\multirow{1}{*}{\textbf{Datasets}} 
& \textbf{PROTEINS} & \textbf{MUTAG} & \textbf{D\&D}\\
\midrule
Max. \#Nodes & 620 & 28 & 5,748\\
Min. \#Nodes & 4 & 10 & 30\\
Avg. \#Nodes & 39.06 & 17.93 & 284.32\\
Avg. \#Edges & 72.82 & 19.79 & 715.66\\
\#Graphs & 1,113 & 188 & 1,178\\
\#Classes & 2 & 2 & 2\\
\bottomrule
\end{tabular}
\end{center}
\end{table*}

\paragraph{Datasets} We evaluate the performance of the proposed \textsc{FGConv} on three graph classification benchmarks. The selected datasets are described as follows. \textbf{D\&D} \citep{dobson2003distinguishing, shervashidze2011weisfeiler} is a graph dataset consists of 1,113 protein structures, each of which is represented by a graph whose nodes are amino acids and there is an edge if two nodes are less than six angstroms apart. The node features of each graph are formed by the binary encoding of some chemical properties. The task of using this dataset is a binary classification problem, and we aim to classify each protein structure into either enzymes or non-enzymes. \textbf{PROTEINS} \citep{dobson2003distinguishing, borgwardt2005protein} is another protein structure dataset with the same task, which is treated as a simplified version of \textbf{D\&D} in terms of graph size. \textbf{MUTAG} \citep{debnath1991structure, kriege2012subgraph} is a mutagen dataset which contains 188 chemical compounds. We use graphs to depict the compounds, where the nodes and edges of each graph correspond to the atoms and covalent bonds of each compound, respectively. The task is to predict whether the compounds in the dataset are mutagenic or not. Important statistical information of these datasets is provided in Table~\ref{partial_statistics}.

\begin{table}[th!]
\caption{Mean test accuracy (in percentage) and standard deviation of \textsc{FGConv-Sum} as compared with existing methods on the benchmark graph classification datasets, over 10 repetitions.}
\label{task1_results}
\begin{center}
\setlength{\tabcolsep}{20pt}
\begin{tabular}{l c c c}
\toprule
\textbf{Methods} & \textbf{PROTEINS} & \textbf{MUTAG} & \textbf{D\&D}\\
\midrule
\textsc{SP} & 75.07$^*$ & 85.79$^*$ & --\\
\textsc{Graphlet} & 71.67$^*$ & 81.58$^*$ & 78.45$^*$\\
\textsc{RW} & 74.22$^*$ & 83.68$^*$ & --\\
\textsc{WL} & 72.92$^*$ & 80.72$^*$ & 77.95$^*$\\
\midrule
\textsc{GIN} & 76.2 & \textbf{89.4} & --\\
\textsc{PatchySan} & 75.00 & \textbf{\textcolor{red}{91.58}} & 76.27\\
\textsc{DGCNN} & 75.54 & 85.83 & 79.37\\
\textsc{DiffPool} & 76.25 & -- & \textbf{80.64}\\
\textsc{SAGPool} & 72.17 & -- & 77.07\\
\textsc{EigenPool} & \textbf{76.6} & -- & 78.6\\
\textsc{g-U-Nets} & \textbf{\textcolor{violet}{77.68}} & -- & \textbf{\textcolor{violet}{82.43}}\\
\midrule
\textsc{FGConv-Sum} & \textbf{\textcolor{red}{ 78.3$\pm$2.26}} & \textbf{\textcolor{violet}{90.8$\pm$2.50}} & \textbf{\textcolor{red}{82.9$\pm$2.55}}\\
\bottomrule
\multicolumn{4}{l}{`$*$' denotes the record retrieved from \cite{niepert2016learning}.}\\
\multicolumn{4}{l}{ `--' means that there is no public record for the method on the dataset.}\\
\multicolumn{4}{l}{! The records without superscription are retrieved from their corresponding}\\
\multicolumn{4}{l}{\hspace{0.3cm}original papers.}\\
\multicolumn{4}{l}{! The decimal place is not modified when transferring the results.}\\
\multicolumn{4}{l}{! The top three scores are highlighted as: \textbf{\textcolor{red}{First}}, \textbf{\textcolor{violet}{Second}}, and \textbf{Third}.}
\end{tabular}
\end{center}
\end{table}

\paragraph{Model architecture} In this experiment, we propose to employ the following network architecture for composing our graph neural network (GNN) \textsc{FGConv-Sum}
\[
\textsc{FGConv}-\textsc{FGConv}-\textsc{SumPool}-\textsc{MLP}.
\]
Specifically, we use two \textsc{FGConv} layers followed by a sum pooling to generate a unified vectorial graph representation which is then sent to \textsc{MLP} for classification. Batch normalization \citep{ioffe2015batch} is employed after each layer of \textsc{MLP}, except for the output layer. We implement a three-layer \textsc{MLP} for \textbf{PROTEINS} and \textbf{MUTAG}, and a two-layer \textsc{MLP} for \textbf{D\&D} in order to achieve a better performance on each dataset. Figure~\ref{fig:FGConv} depicts the architecture of our composed model \textsc{FGConv-Sum}.

\begin{figure}[th!]
\begin{minipage}{\textwidth}
\centering
\includegraphics[width=\linewidth]{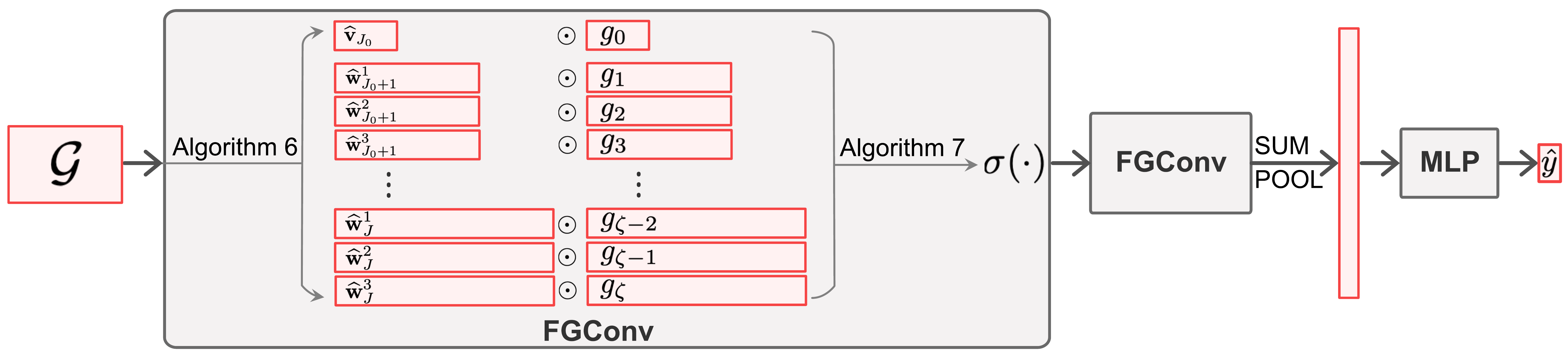}
\caption{Network architecture of \textsc{FGConv-Sum}. Blocks $g_1,\ldots,g_{\zeta}$ are concatenated to form the diagonal of the filter matrix $G$, and $\hat{y}$ denotes the predicted label for the input graph $\mathcal{G}$. All the framelet coefficients are in the spectral domain in order to perform the Hadamard product with the trainable filters $g=[g_0, g_1,\ldots,g_{\zeta}]$.}
\label{fig:FGConv}
\end{minipage}
\end{figure}

\paragraph{Baselines} To evaluate the effectiveness of \textsc{FGConv} on the three benchmark graph classification datasets, we include the following representative GNN methods as our baselines. \textsc{GIN} \citep{xu2018how}, \textsc{PatchySan} \citep{niepert2016learning}, \textsc{DGCNN} \citep{zhang2018end}, \textsc{DiffPool} \citep{ying2018hierarchical}, \textsc{SAGPool} \citep{lee2019self}, \textsc{EigenPool} \citep{ma2019graph}, \textsc{g-U-Nets} \citep{gao2019graph}. Additionally, we also consider several graph kernel methods for comparison, including Shortest-Path kernel (\textsc{SP}) \citep{borgwardt2005shortest}, \textsc{Graphlet} Count kernel \citep{shervashidze2009efficient}, Random Walk kernel (\textsc{RW}) \citep{gartner2003graph} and Weisfeiler-Lehman subtree kernel (\textsc{WL}) \citep{shervashidze2011weisfeiler}.

\begin{table}[th!]
\caption{Grid search space for the hyperparameters.}\vspace{-3mm}
\begin{center}
\begin{tabular}{l c}
\toprule
\textbf{Hyperparameters} & \textbf{Choices}\\
\midrule
Learning Rate & 1$\mathrm{e}$-4, 5$\mathrm{e}$-4, 1$\mathrm{e}$-3, 5$\mathrm{e}$-3, 1$\mathrm{e}$-2\\
Hidden Size & 16, 32, 64, 128\\
Weight Decay (L2) & 1$\mathrm{e}$-4, 5$\mathrm{e}$-4, 1$\mathrm{e}$-3, 5$\mathrm{e}$-3   \\
Batch Size & 32, 64, 128, 256\\
\bottomrule
\end{tabular}
\label{grid_search}
\end{center}
\end{table}

\paragraph{Training scheme} We employ spectral clustering \citep{shi2000normalized, stella2003multiclass} to construct a two-level coarse-grained chain for each graph in the dataset. Spectral clustering has been proved capable of clustering various data patterns and can handle the graph with isolated nodes. The number of parents/clusters in the coarsened level is set to the half of that in its finer level. We split the dataset into training, validation and test sets with portions 80\%, 10\% and 10\% respectively. Since different data splits might have a great impact on the performance of a GNN model  \cite{shchur2018pitfalls}, we repeat the experiment on each dataset $10$ times with random shuffling for the dataset before splitting. We report the mean test accuracy along with the standard deviation of our model \textsc{FGConv-Sum} for each dataset. We use the Adam optimizer \citep{adam_optimizer} with an early stopping strategy suggested in \cite{shchur2018pitfalls} to train our model. Specifically, we stop the training if the validation loss does not improve for continuous 10 epochs with a maximum of 50 epochs. We use a simple grid search to tune the hyperparameters. We show a list of the hyperparameters in the model along with their search spaces in Table~\ref{grid_search}. All the programs used in this section are written in PyTorch and the library PyTorch Geometric \citep{pyg}, and we run the experiments on NVIDIA\textsuperscript{\textregistered} Tesla V100 GPU with 5,120 CUDA cores and 16GB HBM2 mounted on a high performance computing cluster.

\begin{figure}[th!]
\begin{subfigure}{.32\textwidth}
  \centering
  \includegraphics[width=.95\linewidth]{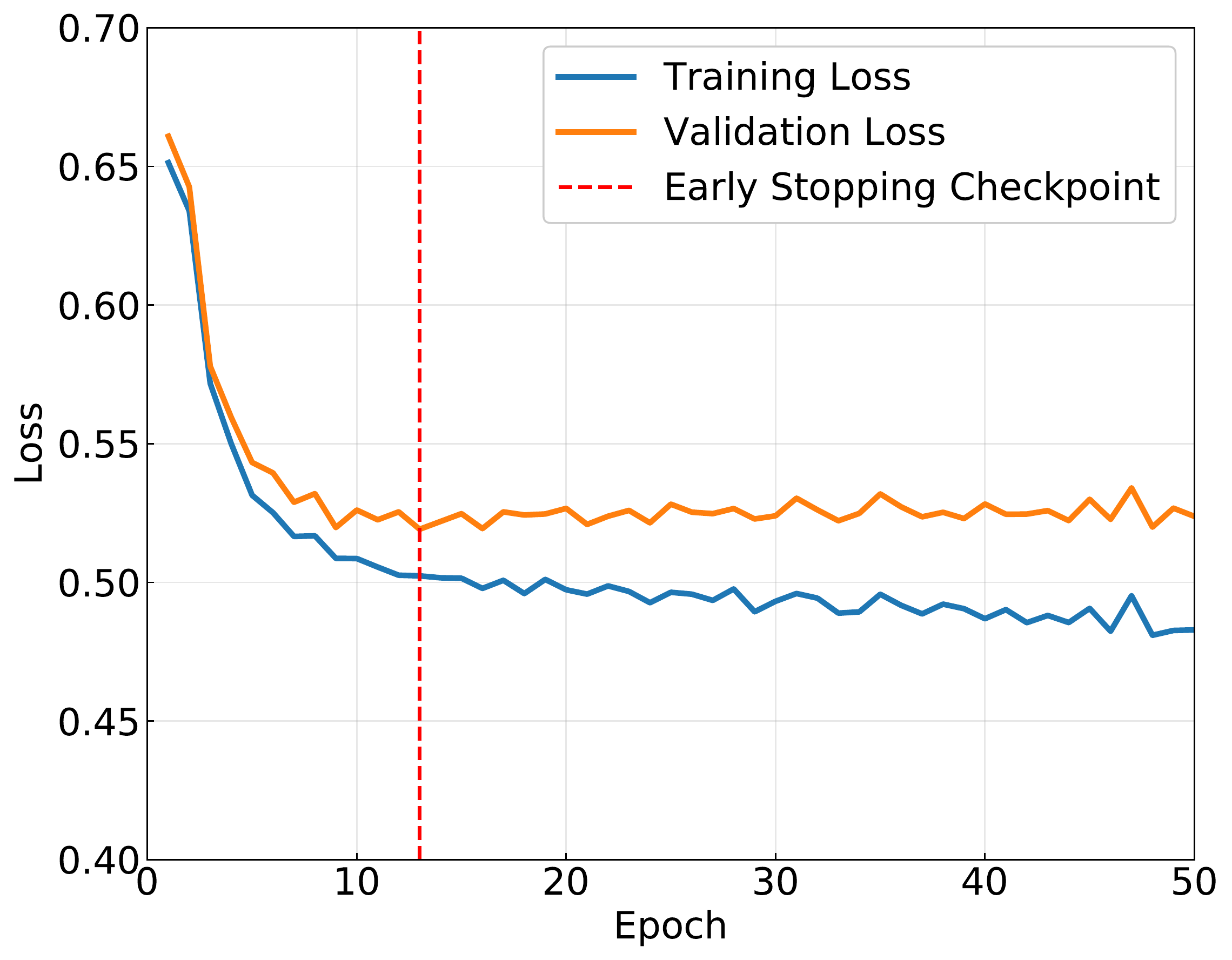}
  \caption{PROTEINS}
\end{subfigure}
\begin{subfigure}{.32\textwidth}
  \centering
  \includegraphics[width=.95\linewidth]{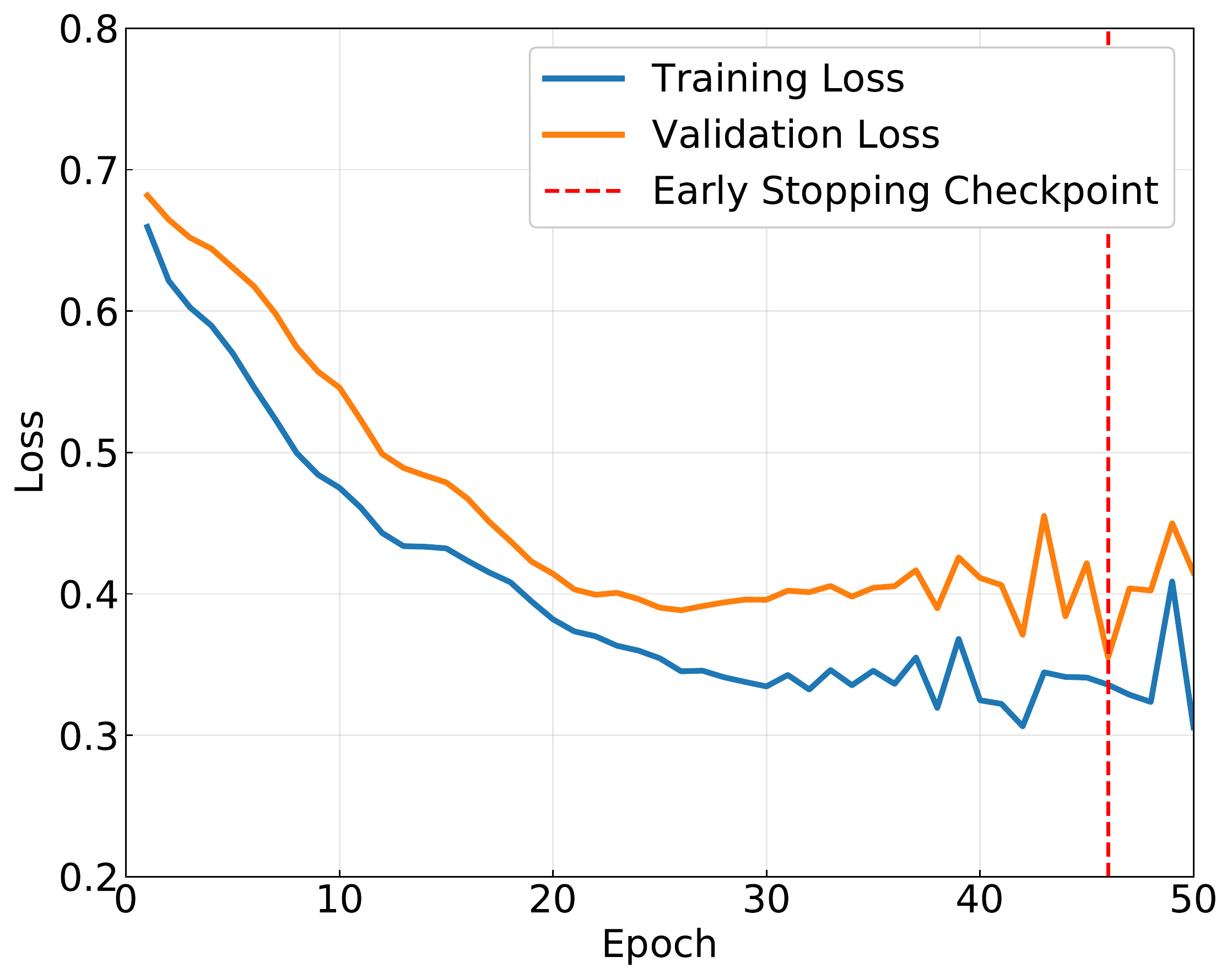}
  \caption{MUTAG}
\end{subfigure}
\begin{subfigure}{.32\textwidth}
  \centering
  \includegraphics[width=.95\linewidth]{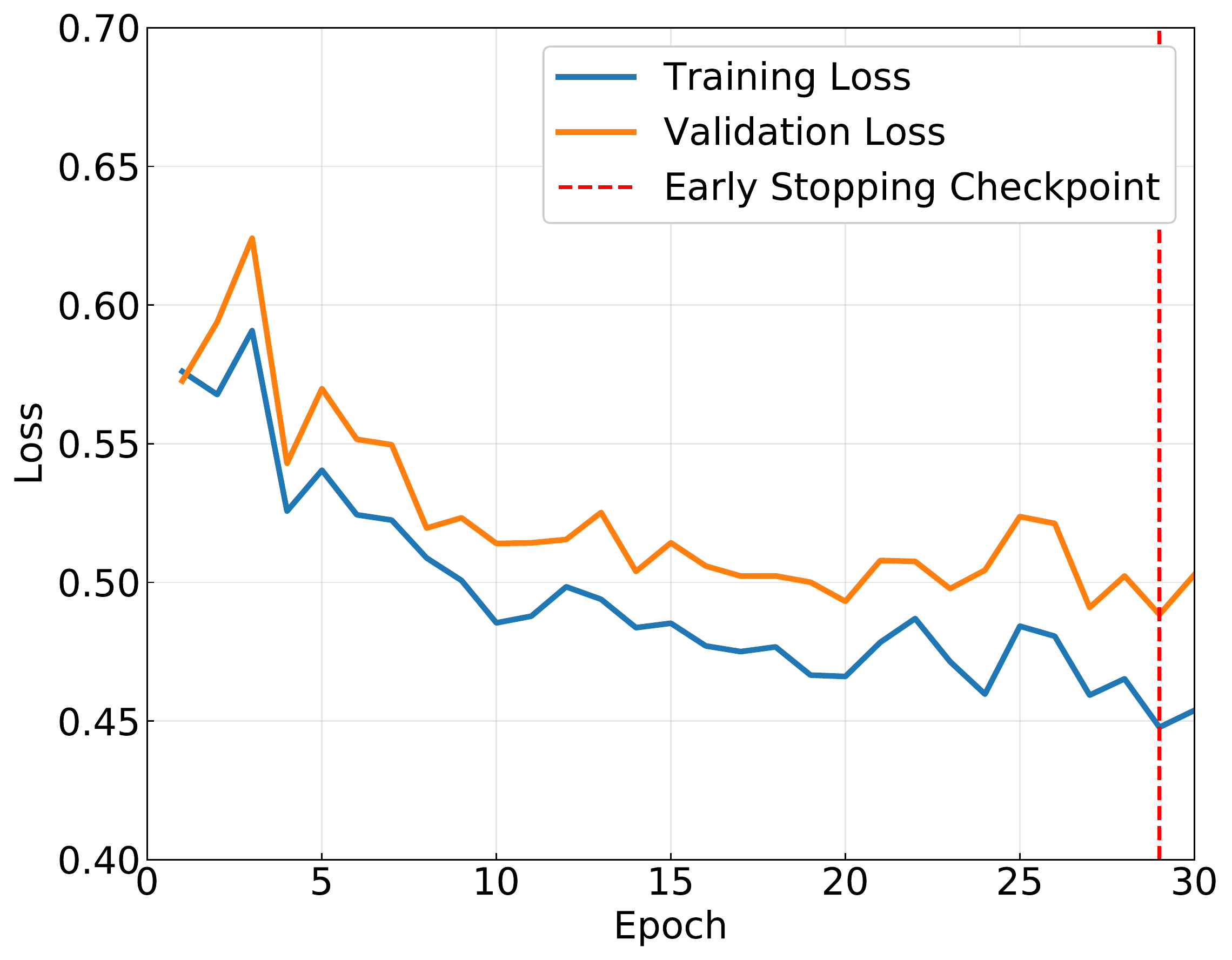}
  \caption{D\&D}
\end{subfigure}
\caption{Plots of losses against epoch from one repetition on the three benchmark graph classification datasets. The early stopping chechpoint corresponds to the epoch where the model is saved for evaluation on the test set, i.e. when the smallest validation loss was achieved.}
\label{exp3:plots}
\end{figure}

\begin{figure}[th!]
\begin{subfigure}{.485\textwidth}
  \centering
  \includegraphics[width=.95\linewidth]{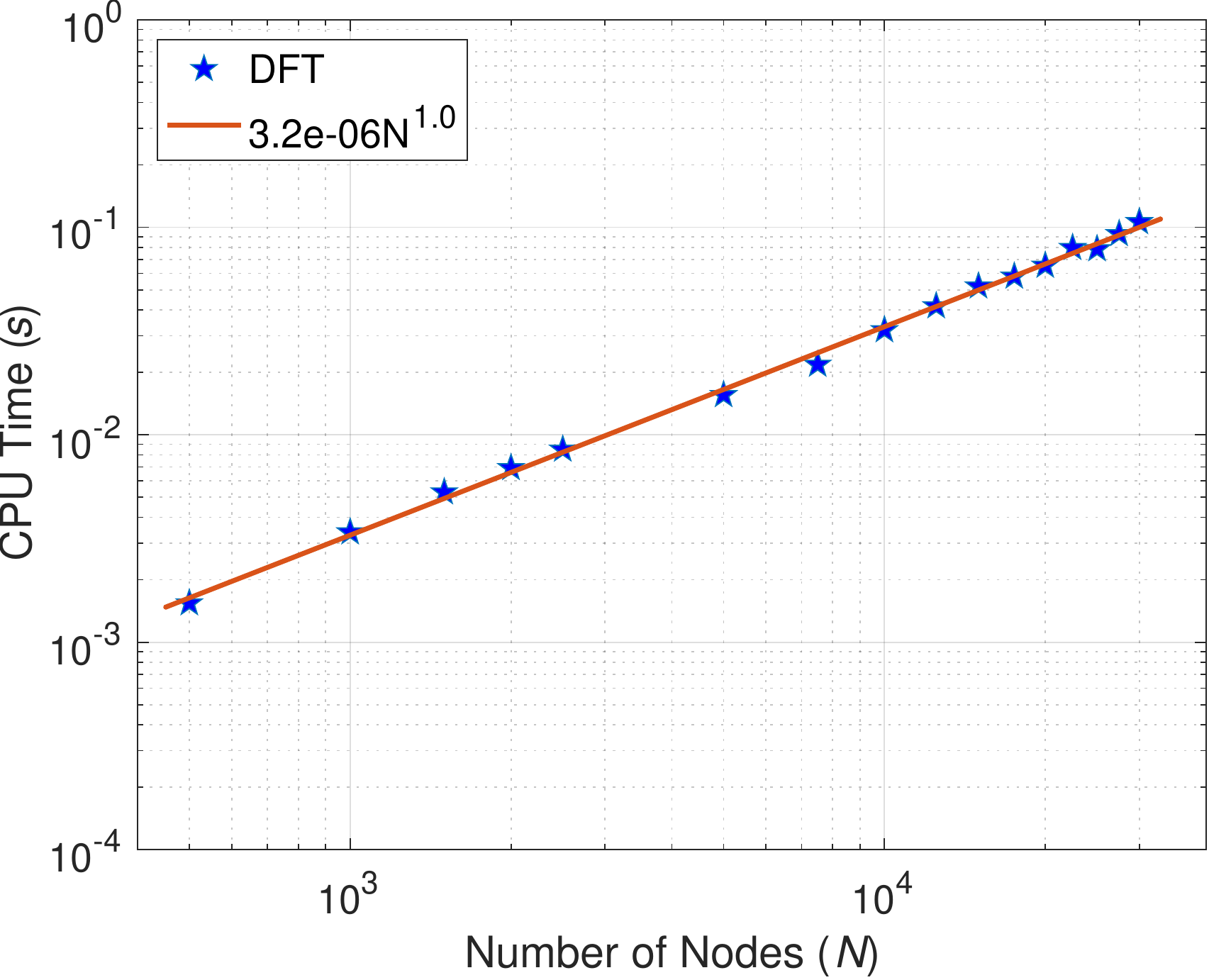} 
  \par\bigskip\bigskip
  \includegraphics[width=.95\linewidth]{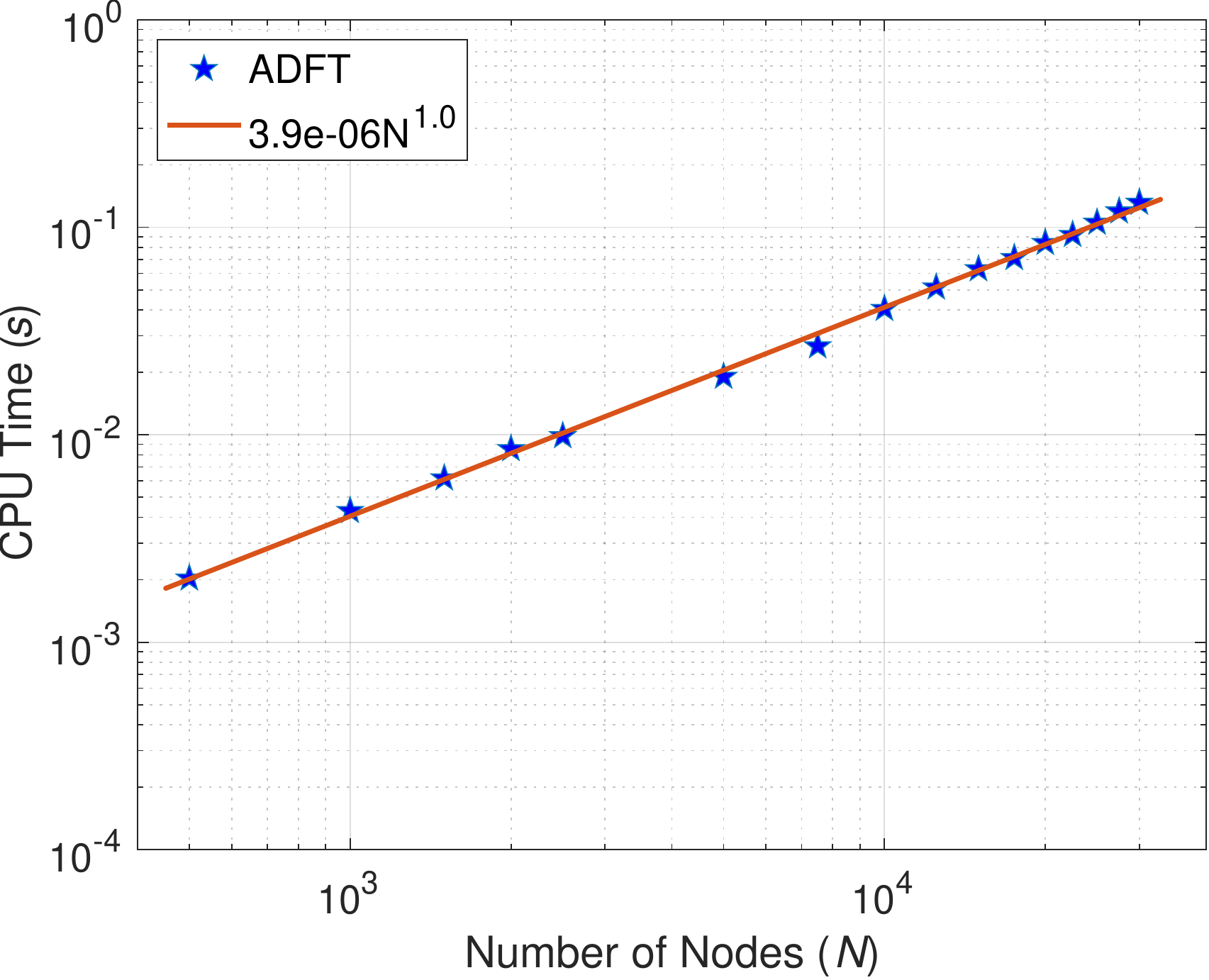}
  \par\bigskip\bigskip
  \includegraphics[width=.95\linewidth]{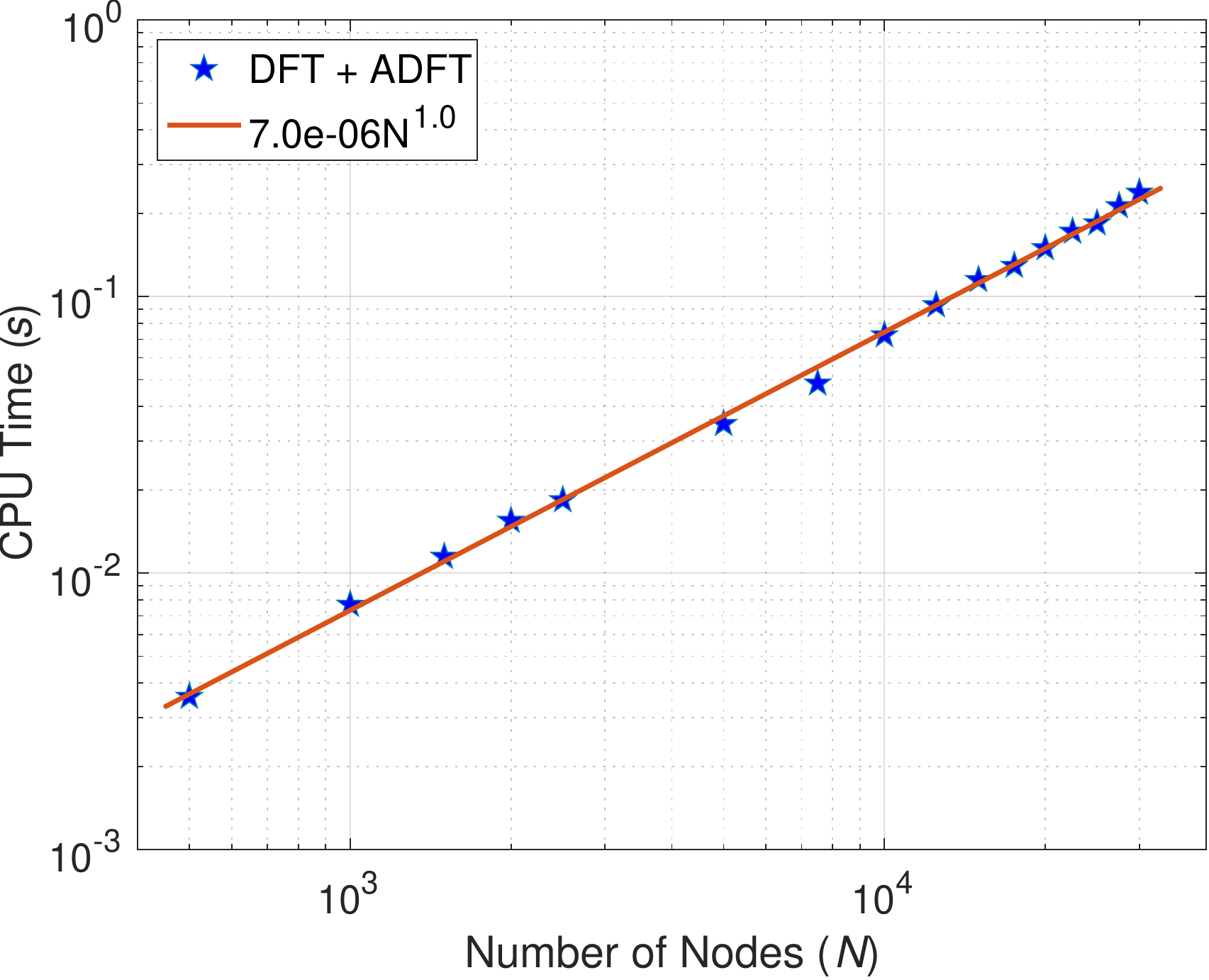}
\end{subfigure}
\hspace*{\fill}
\begin{subfigure}{.485\textwidth}
  \centering
  \includegraphics[width=.95\linewidth]{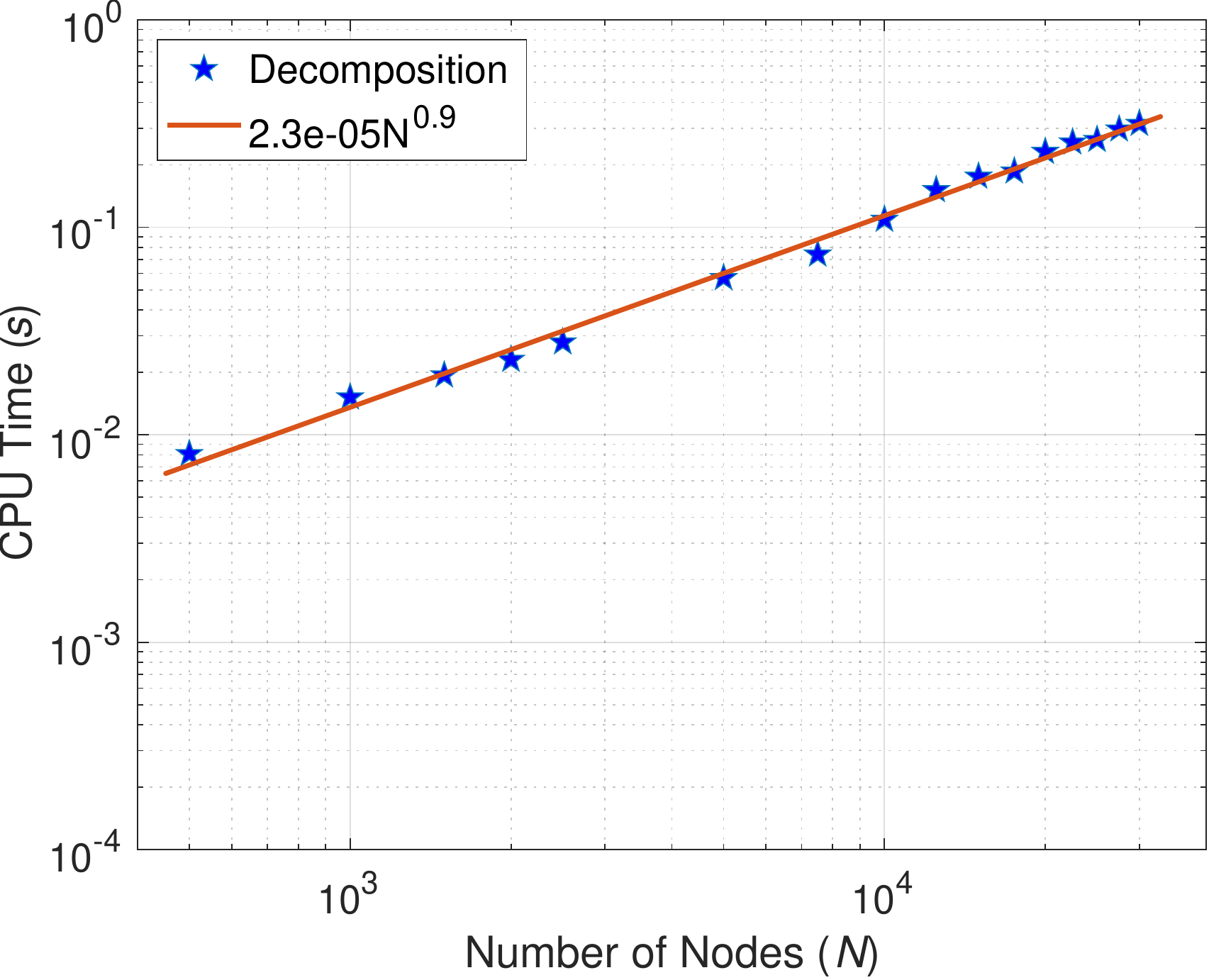}
  \par\bigskip\bigskip
  \includegraphics[width=.95\linewidth]{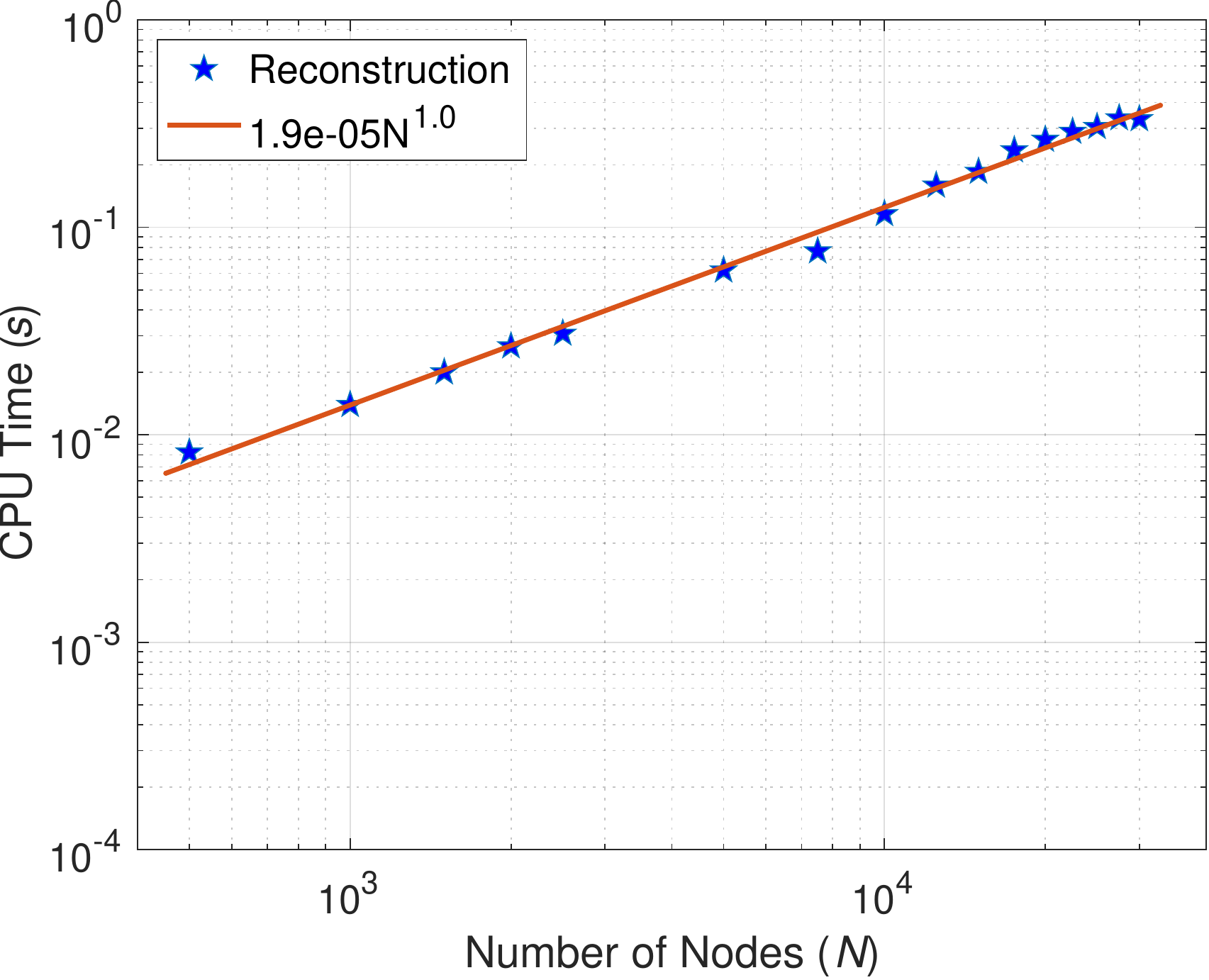} 
  \par\bigskip\bigskip
  \includegraphics[width=.95\linewidth]{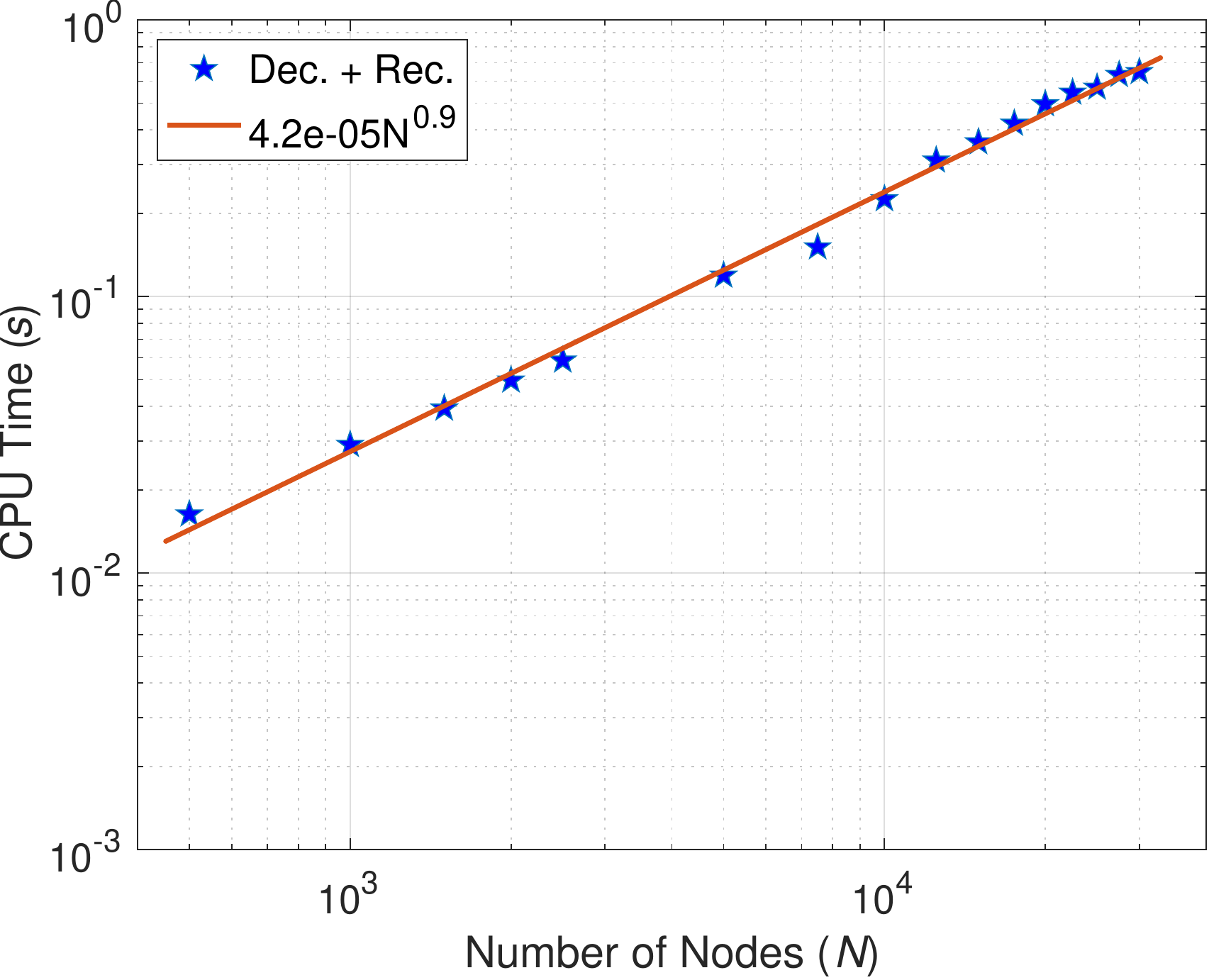} 
\end{subfigure}
\caption{CPU times (in seconds) of DFTs and {\fgt}s elapsed on 16 random graphs with different sizes ranging from $500$ nodes to $30,000$ nodes. The least-squares method is used to estimate the order of computational time (in a power function of the form $y=k_1 N^{k_2}$). We visualize the fitted curves with red lines and report the estimated order in the legends.}
\label{fig:exp1}
\end{figure}

\paragraph{Results} The experimental results are reported in Table~\ref{task1_results}. We can observe that our \textsc{FGConv-Sum} achieves comparable performances with the state-of-the-art baseline methods on all the datasets. Particularly, our model obtains the top mean test accuracy on \textbf{PROTEINS} and \textbf{D\&D}, while the performance of our \textsc{FGConv-Sum} is still ranked top-three among all the baseline methods on \textbf{MUTAG}. We visualize the convergence of training and validation losses of \textsc{FGConv-Sum} for one running on each dataset in Figure~\ref{exp3:plots}. Note that the fluctuations of the loss values near the end of the training are due to the use of moderate learning rate and the absence of learning rate scheduler. However, an obvious convergence demonstrated by each sub-figure further shows evidence of the feasibility and effectiveness of using {\fgt} in a spectral graph convolution.

\section{Discussion}\label{sec:discussion}
We construct a decimated tight framelet system on a graph by filtered spectral expansion with well-designed scaling functions. The scaling function gives a set of the filter bank. A coarse-grained chain for the graph is generated to achieve decimation of the framelet system. The associated filter bank allows fast implementation of the decimated $\gph$-framelet transforms. The linear computational complexity depends on the speed of the discrete Fourier transforms for the orthogonal basis. To this end, the orthogonal basis is generated based on the chain. When the chain has an appropriate structure, the resulting orthogonal basis has fast evaluation. It distinguishes our method and construction from existing framelet or wavelet methods that are undecimated \cite{HaVaGr2011,Dong2017}.

\paragraph{Choice of chain} By Theorem~\ref{thm:DFS}, the choice of chain and its clustering algorithm does not affect the tightness of the decimated framelet system. Nevertheless, the chain structure is highly related to the computation for the decimated $\gph$-transforms. For instance, Proposition~\ref{prop:haar:sparse} provides a sufficient condition of the chain for the sparsity of the Haar global orthonormal basis matrix. Theorem~\ref{thm:flop} and Proposition~\ref{prop:fgt:fft} then guarantee the linear computational complexity of the discrete Fourier transforms and $\gph$-framelet transforms under the Haar basis. On the other hand, the subgraphs of the chain carry the clustering information of the original graph. The chain-based orthonormal basis and the subsequent decimated framelet system then have a clustering feature of the graph. Thus, besides the neighbour information, the chain-based graph Laplacian basis also has the clustering property of the graph embedded in the system. As the Haar basis is solely constructed based on the chain, it mainly reflects the clustering property rather than the neighbour linking. For graph convolution based on our decimated $\gph$-framelet transforms, the Haar version is more suitable for graph-level property prediction tasks (graph classification and graph regression) while the graph Laplacian version could also work for node property prediction tasks.

\paragraph{Number of high passes} The number of high passes does not affect the tightness of the system. However, the representation by two framelet systems with different numbers of high passes contains a distinct extent of details in the framelet domain. Empirically, framelets with more high passes have a better signal-to-noise ratio \citep{WaZh2018}.

\paragraph{Application} There are many potential application of the {\fgt}. For example, it can be used to accelerate the diffusion-based Gaussian process when the inversion of the Gaussian covariance matrix based on graph Laplacian needs efficient evaluation \citep{dunson2020diffusion}. {\fgt} can also be used in graph pooling \citep{wang2019haar} where we can filter out the high pass signal for pooling operation in GNN models. When filtering out only the framelet details while keeping the framelet approximation in the pooling output, the main information of the graph signal is preserved while the graph size is compressed.

%

\paragraph{Acknowledgements}
The authors are grateful to Professors Junbin Gao, Ming Li and Pietro Li\`{o} for their helpful comments.
The last author acknowledges support in part from Research Grants Council of Hong Kong (Project No. CityU 11301419) and City University of Hong Kong (Project No. 7005497). The third author acknowledges the partial support of funding from the European Research Council (ERC) under the European Union's Horizon 2020 research and innovation programme (grant agreement n\textsuperscript{o} 757983), and the support of SJTU Explore X Fund (SD6040004/034) and SJTU Start-up Fund (WH220441902).
This material is based upon work supported by the National Science Foundation under Grant No.~DMS-1439786 while the third author was in residence at the Institute for Computational and Experimental Research in Mathematics in Providence, RI, during the Collaborate@ICERM on ``Geometry of Data and Networks''.

\FloatBarrier
\newpage
\appendix
\section{Table of Notations} 
\label{app:notations}

\begin{longtable}{p{.20\textwidth}  p{.80\textwidth}} 
\multicolumn{1}{l}{\bf Symbol} & \multicolumn{1}{c}{\bf Meaning}\\
\toprule
$\Rd[d]$ & $d$-dimensional real coordinate space\\
$\gph$ & An undirected connected graph with a non-empty finite vertex set $\VG$ and edge set $\EG$\\
$\Lambda_j$ & The index sets of the translation points and centers on the domain $\dom$\\
$|\VG|$ or $N$ & The number of vertices\\
$\uG$ or $\vG$ & An arbitrary vertex of $\VG$\\
$\wG$ & A non-negative weight function that project $\EG\rightarrow\R$\\
$\vol(\gph)$ & The volume of the graph\\
$\mG(\uG,\vG)$ & The transition kernel or a Markov random walk on $\gph$\\
$\distG(\uG,\vG)$ & The distance between two vertices $\uG$ and $\vG$\\
$\gph_c$ & A coarse-grained graph of $\gph$, where $\VG_c$ is a partition of $V$\\
$[\vG]_{\gph_c}$ & An equivalent class (cluster) where a vertex in $\gph_c$ is associated with a vertex $\vG\in\VG$\\
$|[\vG]|$ or $\#[\vG]$ & The number of vertices in the cluster $[\vG]$\\
$\gph_{J\rightarrow J_0}$ & A coarse-grained chain of $\gph$ constituted by a sequence of graphs $\gph_J, \gph_{J-1},\ldots,\gph_{J_0}$\\
$\gph_j$ & The level-$j$ graph of the chain $\gph_{J\to J_0}$. The $\gph_{j-1}$ is the coarse-grained graph of $\gph_j$\\
$\tri$ & A tree of the coarse-grained chain $(\gph_J,\ldots,\gph_{J_0})$ of $\gph$. The $\{V\}$ is the root of $\tri$ at the top level $J_0$, and all vertices of $\gph$ are the leaves at the bottom level $J$.\\
$l_2(\gph)$ & The Hilbert space of vectors $\vf:\VG\rightarrow \C$ on $\gph$\\
$\{\eigfm\}_{\ell=1}^\NV$ & An orthonormal basis for $l_2(\gph)$, usually defined as an eigenvector set with eigenvalues $\eigvm$ of $\gL$\\
$\conj{\vg}$ & The complex conjugate to $\vg\in l_2(\gph)$ \\
$\nrmG{\cdot}_\gph$ & The induced norm, which is defined as the square root of an inner product on the vectors. \\
$\delta_{\ell,\ell'}$ & The Kronecker delta satisfying $\delta_{\ell,\ell'}=1$ if $\ell=\ell'$ and $\delta_{\ell,\ell'}=0$ if $\ell\neq\ell'$\\
$\Fcoem{\vf}$ & The generalized Fourier coefficient of degree $\ell$ for $\vf\in l_2(\gph)$ with respect to $\eigfm$\\
$\vf$ & The input vector on the time domain\\
$\gL$ & The combinatorial or unnormalized graph Laplacian operator\\
$\{\vg_\ell\}_{\ell=1}^M$ & A frame for $l_2(\gph)$ if satisfying (\ref{defn:frame})\\
$L$ & Lebesgue measure\\
$\FT{\gamma}$ & The Fourier transform of a function $\gamma\in L_1(\R)$\\
$\mask$ & A filter (or mask) of a complex-valued sequence $\{\mask_k\}_{k\in\Z}\subseteq \C$\\
$\Psi$ & A set of framelet generators (or scaling functions) $\{\scala; \scalb^{(1)},\ldots,\scalb^{(r)}\}$ in $L_1(\R)$\\ 
$\Psi_j$ & A set of framelet generators (or scaling functions) $\{\scala_j; \scalb_j^{(1)},\ldots,\scalb_j^{(r)}\}$ in $L_2(\R)$ at scale $j$, see \eqref{eq:refinement:nonstationary}\\ 
$\filtbk$ & A filter bank with a set of filters $\{\maska; \maskb[1],\ldots,\maskb[r]\}$\\
$\filtbk_j$ & A filter bank at level $j$ connecting $\Psi_{j}$ and $\Psi_{j-1}$\\
$\scala$ & Scaling function associated with low pass for undecimated framelets\\
$\scalb^{(n)}$ & Scaling function associated with $n$th high pass for undecimated framelets\\
$\scala_j$ & Scaling function associated with low pass at level $j$ for decimated framelets\\
$\scalb_j^{(n)}$ & Scaling function associated with $n$th high pass at level $j$ for decimated framelets\\
$\maska^\star$ & The low-pass filter or refinement mask at decomposition\\
$\maska$ & The low-pass filter or refinement mask at reconstruction\\
$(\maskb)^\star$ & The high-pass filters or framelet masks at decomposition\\
$\maskb$ & The high-pass filters or framelet masks at reconstruction\\
$\xi$ & A signal on the time domain\\
$\AS_J(\Psi)$ & A stationary non-homogeneous affine system of the form (\ref{defn:AS:R})\\
$\varphi_{j,p}$ & Framelets at scale $j$ and at a vertex $\uG\in\VG$ for low-pass signals\\
$\psi_{j,p}^{(n)}$ & The $n$-th framelets at scale $j$ for high-pass signals\\
$\ufrsys[]\left(\Psi,\filtbk;\gph\right)$ or $\ufrsys[J_1]^{J}(\Psi,\filtbk)$ & An undecimated tight framelet system for $l_2(\gph)$\\
$\dfrsys(\{\Psi_j\}_{j=J_1}^J$, $\{\filtbk_j\}_{j=J_1+1}^{J})$ & The decimated framelet system defined as (\ref{defn:DFS})\\
$\wV$ & The associated weight of vertex $\pV$ at level $j$\\
$\QN[j]$ & A set of weights $\{\wV: \pV\in\VG_j\}$ on $\gph_j$\\
$\QU{\ell'}(\QN[j])$ & The weighted sum of the product of $\eigfm$, $\eigfm[\ell']$ by (\ref{eq:sum:wt:u})\\
$c_{[\vG]_{\gph_j}}$ & The number of vertices in $\gph_{j+1}$ which are in the cluster $[\vG]_{\gph_j}$\\
$\spoc(\eigfm)$ & Spoc of a vector $\eigfm$\\
$\supp(\eigfm)$ & Support of a vector $\eigfm$\\
$\chi_j^c$ & The characteristic function for the $j$th vertices on $\gph_c$\\
$\fracoev[J_0]$ & The framelet approximation coefficients\\
$\frbcoev{n}$ & The framelet detail coefficients for $n=1,\ldots, r_j,\; j=J_0,\ldots,J$\\
$\vc$ or $\FT{\fracoev[]}$ & The discrete Fourier coefficients of the vector $\fracoev[]$\\
$\fft[j]^*$ & The discrete Fourier transform (DFT) operator $l_2(\VG_j)\rightarrow l_2(\Omega_j)$ on $\gph_j$\\
$\fft[j]$ & The adjoint discrete Fourier transform (ADFT) operator $l_2(\Lambda_j)\rightarrow l_2(\VG_j)$ on $\gph_j$\\
$\dconv$ & The discrete convolution operator\\
$\downsmp$ & The downsampling operator, defined in (\ref{defn:dwnsmp}), from level $j$ to $j-1$\\
$\upsmp$ & The upsampling operator, defined in (\ref{defn:upsmp}), from level $j-1$ to level $j$\\
$\analOp$ & The multi-level $\gph$-framelet analysis operator\\
$\synOp$ & The multi-level $\gph$-framelet synthesis operator\\
\end{longtable}

\vskip 0.2in
\newpage
\bibliography{FGT.bib}

\end{document}